\documentclass{article}
\pdfoutput=1

\usepackage[utf8x]{inputenc}
\usepackage[T1]{fontenc}
\usepackage{microtype}
\usepackage{subfigure}
\usepackage{booktabs} % for professional tables
\usepackage{amsfonts, amsmath, amssymb, amsthm}
\usepackage{xfrac}
\usepackage{graphicx, xcolor, colortbl}
\usepackage{footnote}
\usepackage{tablefootnote}
\usepackage{natbib}
\usepackage{dsfont}
\usepackage{bbm}
\usepackage{algorithm}
\usepackage{algorithmic}
\usepackage{import}
\usepackage{mathtools}
\usepackage{bm}
\usepackage{xspace}
\usepackage{thmtools,thm-restate} %to restate lemma
\usepackage{accents}
\usepackage{framed}
\usepackage{enumitem}
\usepackage{booktabs}
\usepackage{color, colortbl}

\usepackage[accepted]{sty/icml2022}

\usepackage{sty/notations}
\theoremstyle{plain}
\newtheorem{theorem}{Theorem}[section]
\newtheorem{proposition}[theorem]{Proposition}
\newtheorem{lemma}[theorem]{Lemma}
\newtheorem{corollary}[theorem]{Corollary}
\theoremstyle{definition}

\theoremstyle{remark}
\newtheorem{remark}[theorem]{Remark}

\usepackage{scrextend}

\usepackage{multirow}

\usepackage{minitoc}

\icmltitlerunning{Bayes UCBVI}

\begin{document}

\twocolumn[
\icmltitle{From Dirichlet to Rubin: Optimistic Exploration in RL without Bonuses}
% It is OKAY to include author information, even for blind
% submissions: the style file will automatically remove it for you
% unless you've provided the [accepted] option to the icml2022
% package.

% List of affiliations: The first argument should be a (short)
% identifier you will use later to specify author affiliations
% Academic affiliations should list Department, University, City, Region, Country
% Industry affiliations should list Company, City, Region, Country

% You can specify symbols, otherwise they are numbered in order.
% Ideally, you should not use this facility. Affiliations will be numbered
% in order of appearance and this is the preferred way.
%\icmlsetsymbol{equal}{*}

\begin{icmlauthorlist}
\icmlauthor{Daniil Tiapkin}{hse,airi}
\icmlauthor{Denis Belomestny}{essen,hse}
\icmlauthor{\' Eric Moulines}{polytechnique,hse}
\icmlauthor{Alexey Naumov}{hse}
\icmlauthor{Sergey Samsonov}{hse}
\icmlauthor{Yunhao Tang}{deepmind}
\icmlauthor{Michal Valko}{deepmind}
\icmlauthor{Pierre M\' enard}{ovgu}
\end{icmlauthorlist}

\icmlaffiliation{hse}{HSE University}
\icmlaffiliation{airi}{Artificial Intelligence Research Institute}
\icmlaffiliation{essen}{Duisburg-Essen University}
\icmlaffiliation{polytechnique}{\' Ecole Polytechnique}
\icmlaffiliation{deepmind}{DeepMind}
\icmlaffiliation{ovgu}{Otto von Guericke University}

\icmlcorrespondingauthor{Daniil Tiapkin}{dtyapkin@hse.ru}
%\icmlcorrespondingauthor{Firstname2 Lastname2}{first2.last2@www.uk}

% You may provide any keywords that you
% find helpful for describing your paper; these are used to populate
% the "keywords" metadata in the PDF but will not be shown in the document
\icmlkeywords{Machine Learning, ICML}

\vskip 0.3in
]
% this must go after the closing bracket ] following \twocolumn[ ...

% This command actually creates the footnote in the first column
% listing the affiliations and the copyright notice.
% The command takes one argument, which is text to display at the start of the footnote.
% The \icmlEqualContribution command is standard text for equal contribution.
% Remove it (just {}) if you do not need this facility.

\printAffiliationsAndNotice{}  % leave blank if no need to mention equal contribution
%\printAffiliationsAndNotice{\icmlEqualContribution} % otherwise use the standard text.

% For TOC in appendix (https://tex.stackexchange.com/a/419290)
\doparttoc % Tell to minitoc to generate a toc for the parts
\faketableofcontents % Run a fake tableofcontents command for the partocs

\begin{abstract}
We propose the \BayesUCBVI algorithm for reinforcement learning in tabular, stage-dependent, episodic Markov decision process: a natural extension of the \BayesUCB algorithm by \citet{kaufmann12} for multi-armed bandits. Our method uses the quantile of a Q-value function posterior as upper confidence bound on the optimal Q-value function. For \BayesUCBVI, we prove a regret bound of order $\tcO(\sqrt{H^3SAT})$ where $H$ is the length of one episode, $S$ is the number of states, $A$ the number of actions, $T$ the number of episodes, that matches the lower-bound of $\Omega(\sqrt{H^3SAT})$ up to poly-$\log$ terms in $H,S,A,T$ for a large enough $T$. To the best of our knowledge, this is the first algorithm that obtains an optimal dependence on the horizon $H$ (and $S$) \textit{without the need of an involved Bernstein-like bonus or noise.} Crucial to our analysis is a new fine-grained anti-concentration bound for a weighted Dirichlet sum that can be of independent interest. We then explain how \BayesUCBVI can be easily extended beyond the tabular setting, exhibiting a strong link between our algorithm and Bayesian bootstrap \citep{rubin1981bayesian}.

\end{abstract}

% !TEX root = ../KL-UCRL_episodic.tex
\section{Introduction}
\label{sec:intro}

In reinforcement learning (RL), an agent interacts with an environment with the objective of maximizing the sum of collected rewards \citep{puterman1994markov}. In order to fulfill this objective, the agent should balance between \emph{exploring the environment} and \emph{exploiting the current knowledge} to accumulate rewards. In this paper aim at providing \emph{generic solution} to this exploration-exploitation dilemma.

We model the environment as an unknown episodic tabular Markov decision process (MDP) with $S$ states, $A$ actions, and episodes of length $H$. After $T$ episodes, we measure the performance of the agent by its cumulative regret which is the difference between the total reward collected by an optimal policy and the total reward collected by the agent during the learning. In particular, we study the \emph{non-stationary} setting where rewards and transitions can change within an episode.

An effective and widely used way to solve the exploration-exploitation dilemma is the application of the principle of optimism in face of uncertainty. One line of work \citep{azar2017minimax,dann2017unifying,Zanette19Euler} for episodic MDPs and for non-episodic MDPs \citep{jaksch2010near,fruit2018efficient,talebi2018variance} implements this principle by injecting optimism through \textit{bonuses added to the rewards}. By adding these bonuses we can build upper confidence bounds (UCBs) on the optimal Q-value functions and act greedily with respect to them. Typically, these bonuses are decreasing functions of counts on the number of visits of state-action pairs. Notably, for such approach, \citet{azar2017minimax} proved a regret bound of order\footnote{We translate all the bounds to the \emph{stage-dependent} setting by multiplying by $\sqrt{H}$ the regret bounds in the stage-independent setting.}\footnote{In the $\tcO(\cdot)$ notation we ignore terms poly-$\log$  in $H,S,A,T$.} $\tcO(\sqrt{H^3SAT})$. Note that this upper bound matches, in the first order and up to poly-logarithmic terms, the known lower bound \citep{domingues2021episodic,jin2018is} of order $\Omega(\sqrt{H^3SAT})$ for the considered setting. The exploration based on building UCBs and adding bonuses is besides model-based also used for model-free algorithms \citep{jin2018is,zhang2020advantage}. For example, \citet{zhang2020advantage} proved a regret bound of order $\tcO(\sqrt{H^3SAT})$ for an optimistic version of the Q-learning algorithm \citep{watkins1992q}. One shortcoming of this exploration method is that algorithms with bonuses designed to obtain optimal problem-independent regret bound often perform poorly in practice, even for simple MDPs \citep{obsband2013more,osband2017why}. Furthermore, the notion of count used in the bonuses does not easily generalize beyond the tabular setting\footnote{Or simple linearly parameterized settings.} even if some solutions exist \citep{bellemare2016unifying,tang2017exploration,burda2019exploration};  See Section~\ref{sec:related-work-deepRL} for a thorough review of these methods.

A second line of work introduces optimism by \textit{injecting noise}.  \citet{obsband2013more,osband16generalization,osband2017why,agrawal2020posterior} proposed the posterior sampling for RL (\PSRL), an adaptation of the well-known Thompson sampling \citep{thompson1933on} for multi-armed bandits. Using a Bayesian view, \PSRL maintains a posterior on the MDP parameters and at each episode samples a new parameter from this posterior to act greedily with respect to it. Despite its good empirical performance in comparison to bonus-based algorithms \citep{obsband2013more,osband2017why}, it is not known if \PSRL algorithm can attain the problem-independent lower bound. Indeed, the best regret bound proved by \citet{agrawal2020posterior,qian2020concentration} is of order $\tcO(H^2S\sqrt{AT})$ for \PSRL. \PSRL is close to the randomized least-square value iteration (\RLSVI, \citealp{osband16generalization,russo2019worst}) which injects noise directly in the value iteration through noisy Bellman updates. Specifically, a Gaussian with a variance that shrinks with the number of visits is added at each state-action pair during the value iteration.
Interestingly, \RLSVI also demonstrates good empirical performance in practice but most importantly can easily be extended outside the tabular setting as explained by \citet{russo2019worst,osband2019deep}, in particular, deep RL environment \citet{osband2016deep,osband18randomized,osband2019deep} %successfully extended \RLSVI to the Deep RL setting.  
Specifically, they combine \RLSVI with \DQN \citep{mnih2015humanlevel} by replacing the Gaussian noise in \RLSVI with a bootstrap sample \citep{efron1979bootstrap} of the next targets. As a first step to analyze such noise, recently, \citet{pacchiano2021towards} analyzed a version of \RLSVI where the Gaussian noise is replaced by a bootstrap sample of \emph{the past rewards} and adding pseudo rewards in the same fashion as \citet{kveton2019garbage}. Their algorithm, \BootNARL, comes with a regret bound of order $\tcO(H^2S\sqrt{AT})$. Note that \citet{russo2019worst} proved a regret bound of order $\tcO(H^2S^{3/2}\sqrt{AT})$ for the original version of \RLSVI. Later \citet{xiong2021nearoptimal} improved this bound to $\tcO(\sqrt{H^3SAT})$ but at the price of scaling the Gaussian noise by a term similar to the Bernstein bonuses used in \UCBVI. In particular, it is not clear if such variant could also be extended beyond the tabular setting. 
\par
Thus among the above Bayesian-inspired algorithms which both demonstrate good empirical results and are also readily extendable to large-scale environments none of them enjoys such as strong guarantee as problem-independent optimality. Therefore, in this paper, we propose to fill this gap with the \BayesUCBVI algorithm. It is an optimistic algorithm that does not rely on bonuses but uses the \textit{quantile of a of Q-value functions posterior} as UCBs on the optimal Q-value functions. We can think of \BayesUCBVI as a deterministic version of \PSRL, which, in particular, shares with \PSRL the same good empirical performance, see Section~\ref{sec:experiments}. We adopt a \emph{surrogate Bayesian model} for the transitions starting from an (improper) Dirichlet prior. \textit{No assumption is made on the environment} and that the Bayesian model is purely instrumental for \BayesUCBVI. The posterior on the Q-value function is then obtained by the Bellman equations. The prior can be interpreted as prior observations of pseudo-transitions toward an absorbing pseudo-state with maximal reward. As a result, \BayesUCBVI has the advantage of requiring no information on the state space. We note that similar optimistic prior observations were already explored by \citet{Brafman02RMAX,szita2008}. For \BayesUCBVI, we prove a regret bound of order $\tcO(\sqrt{H^3SAT})$ matching the lower bound at first order and up to poly-log terms, see Table~\ref{tab:upper_bounds}. In particular we get a tight dependence on the horizon $H$ \textit{without the need of an involved Bernstein-like bonus} \citep{azar2013minimax,zanette2019tighter,zhang2020advantage} \textit{or Bernstein-type noise} \citep{xiong2021nearoptimal}. Indeed, in \BayesUCBVI the UCBs on the optimal Q-value functions induced by the Dirichlet posteriors over the transitions adapt to the variance automatically. Our proof relies on fine control of the deviations of the posterior. This tight control of the posterior is central in the analysis of the Bayesian inspired algorithm; see \citet{agrawal2020posterior,osband2017gaussian}. In particular, we provide a new anti-concentration inequality for a random Dirichlet-weighted sum that could be of independent interest, see Theorem~\ref{th:two_sided_dbc}. We believe that this anti-concentration inequality could also be used to tighten the bound of the \PSRL algorithm.
\par
As \RLSVI, \BayesUCBVI can be extended in a smooth way beyond the tabular setting. Indeed, we can reinterpret the posterior over the Q-value function of a given state-action pair as the distribution of a Bayesian bootstrap sample of the targets for this pair. Recall that in \textit{Bayesian bootstrap} \citep{rubin1981bayesian} the observations are re-weighted by a Dirichlet sample \textit{instead} of being sampled with replacement as done by the standard bootstrap \citep{efron1979bootstrap}. Consequently, the quantile serving as UCB can be straightforwardly approximated by Monte-Carlo method with Bayesian bootstrap samples. Thus, the exploration procedure of \BayesUCBVI, can also be combined with the \DQN algorithm to tackle large-scale RL: We achieve that by simply re-weighting the regression loss of the Q-value functions by a different Dirichlet sample.  In particular, we explain how to combine the exploration procedure of \BayesUCBVI with the \DQN algorithm for Deep RL. The resulting algorithm is in essence an optimistic version of the one of \citet{osband2019deep}. We show experimentally that the resulting algorithm is competitive with the one introduced by \citet{osband2019deep}.

% An appealing feature of \BayesUCBVI is that it can be extended in a smooth way beyond the tabular setting as described next.
% \BayesUCBVI does not need bonus or noise injection is is instead built on \textit{Bayesian bootstrap} \citep{rubin1981bayesian} where the observations are re-weighted by a Dirichlet sample instead of being sampled with replacement as in the standard bootstrap \citep{efron1979bootstrap}. Thus, in \BayesUCBVI the posterior over the Q-value function of a given state-action pair is basically the distribution of a Bayesian bootstrap sample of the targets for this pair. Consequently the quantile serving as UCBs can be approximated by Monte-Carlo method with Bayesian bootstrap samples. It is then as straightforward as \BootDQN to combine \BayesUCBVI with \DQN by learning an ensemble of Q-value functions to approximate Bayesian bootstrap samples. Indeed we only need to weight the regression loss of each Q-value function by a different Dirichlet sample. The resulting algorithm, \BayesUCBDQN, is in essence an optimistic version of \BootDQN. 
% However, contrary to \citet{bai21principled} \BootDQN 
% \todom{\BayesUCBVI? ie contrary to \BootDQN, \BayesUCBVI}
% does not need to specify any bonuses since the optimism is provided by taking the quantile. Thus on the one hand \BayesUCBVI leads naturally to a principled algorithm to tackle the exploration-exploitation dilemma in large scale environments. 
% On the other hand our analysis of \BayesUCBVI gives a solid theoretical justification for the use of (Bayesian) bootstrap in deep RL.

We highlight our main contributions:
  \begin{itemize}
    \item We propose the \BayesUCBVI algorithm for tabular, stage-dependent, episodic RL. Interestingly \BayesUCBVI is an optimistic algorithm that does not rely on adding bonuses but rather builds UCBs on the optimal Q-value functions by taking the quantile of a well-chosen posterior.  For \BayesUCBVI, we provide a regret bound of order $\tcO(\sqrt{H^3SAT})$ matching the problem independent lower bound up to poly-$\log$ terms.
    \item Central to the analysis of \BayesUCBVI is a new anti-concentration inequality for a Dirichlet weighted sum (Theorem~\ref{th:two_sided_dbc}). We believe this inequality could be of independent interest, e.g., to sharpen the regret bound of other Bayesian inspired algorithms like \PSRL.
    \item We extend \BayesUCBVI beyond the tabular setting, exhibiting a strong link between our algorithm and Bayesian bootstrap \citep{rubin1981bayesian}. In particular, we explain how to combine the exploration procedure of \BayesUCBVI with the \DQN algorithm for Deep RL. We show experimentally that the resulting algorithm is competitive with the one introduced by \citet{osband2019deep}.
        %\todom{maybe add some benetifs over existing deep solutions, Pierre will add something here}
 \end{itemize}
 
 \begin{table}[t!]
    \centering
	\label{tab:upper_bounds}
	\resizebox{\columnwidth}{!}{
	\begin{tabular}{@{}ll@{}}
		\toprule
		\bfseries{Algorithm} & {\bf Upper bound}\tiny{ (non-stationary)} \\
		\midrule
		\midrule
    \UCBVI~\citep{azar2017minimax} & \multirow{3}{*}{$\tcO(\sqrt{H^3 SA T})$} \\
% 		 \OptQL~\citep{jin2018is} & $\tcO(\sqrt{H^4 SAT})$ \\
		 \UCBAdventage~\citep{zhang2020advantage} &
		 \\%$\tcO(\sqrt{H^3 SAT})$ \\
		 \RLSVI~\citep{xiong2021nearoptimal} & \\%$\tcO(\sqrt{H^3 SA T})$ \\
		 \midrule
		 \PSRL~\citep{agrawal2020posterior} & \multirow{2}{*}{$\tcO(H^2 S \sqrt{A T})$}\\%$\tcO(H^2S\sqrt{A T})$ \\
		 \BootNARL~\citep{pacchiano2021towards} & \\
		 \midrule
		 \rowcolor{LightGray}
		 \BayesUCBVI~(this paper) & $\tcO(\sqrt{H^3SAT})$ \\
		 \midrule
		 \midrule
		 Lower bound~\citep{jin2018is,domingues2021episodic} & $\Omega(\sqrt{H^3SAT})$\\
		\bottomrule
	\end{tabular}}
\caption{Regret upper bound for episodic, non-stationary, tabular MDPs.}
\end{table}

\section{Setting}
\label{sec:setting}

 We consider a finite episodic MDP $\left(\cS, \cA, H, \{p_h\}_{h\in[H]},\{r_h\}_{h\in[H]}\right)$,  where $\cS$ is the set of states, $\cA$ is the set of actions, $H$ is the number of steps in one episode, $p_h(s'|s,a)$ is the probability transition from state~$s$ to state~$s'$ by taking the action $a$ at step $h,$ and $r_h(s,a)\in[0,1]$ is the bounded deterministic\footnote{We study deterministic rewards to simplify the proofs but our result extend to random rewards as well.} reward received after taking the action $a$ in state $s$ at step $h$. Note that we consider the general case of rewards and transition functions that are possibly non-stationary, i.e., that are allowed to depend on the decision step $h$ in the episode. We denote by $S$ and $A$ the number of states and actions, respectively.

\paragraph{Policy \& value functions} A \emph{deterministic} policy $\pi$ is a collection of functions $\pi_h : \cS \to \cA$ for all $h\in [H]$, where every $\pi_h$  maps each state to a \emph{single} action. The value functions of $\pi$, denoted by $V_h^\pi$, as well as the optimal value functions, denoted by $\Vstar_h$ are given by the Bellman respectively  optimal Bellman equations
\begin{small}
\begin{align*}
	Q_h^{\pi}(s,a) &= r_h(s,a) + p_h V_{h+1}^\pi(s,a) & V_h^\pi(s) &= \pi_h Q_h^\pi (s)\\
  Q_h^\star(s,a) &=  r_h(s,a) + p_h V_{h+1}^\star(s,a) & V_h^\star(s) &= \max_a Q_h^\star (s, a)
\end{align*}
\end{small}%
\!where by definition, $V_{H+1}^\star \triangleq V_{H+1}^\pi \triangleq 0$. Furthermore, $p_{h} f(s, a) \triangleq \E_{s' \sim p_h(\cdot | s, a)} \left[f(s')\right]$   denotes the expectation operator with respect to the transition probabilities $p_h$ and
$\pi_h g(s) \triangleq  g(s,\pi_h(s))$ denotes the composition with the policy~$\pi$ at step $h$.

\paragraph{Learning problem} The agent, to which the transitions are \emph{unknown} (the rewards are assumed to be known for simplicity), interacts with the environment during $T$ episodes of length $H$, with a \emph{fixed} initial state $s_1$.\footnote{As explained by \citet{fiechter1994efficient} and \citet{kaufmann2020adaptive}, if the first state is sampled randomly as $s_1\sim p,$ we can simply add an artificial first state $s_{1'}$ such that for  any action $a$, the transition probability is defined as the distribution $p_{1'}(s_{1'},a) \triangleq p.$} Before each episode $t$ the agent select a policy $\pi^t$ based only on the past observed transitions up to episode $t-1$. At each step $h\in[H]$ in episode $t$, the agent observes a state $s_h^t\in\cS$, takes an action $\pi_h^t(s_h^t) = a_h^t\in\cA$ and  makes a transition to a new state $s_{h+1}^t$ according to the probability distribution $p_h(s_h^t,a_h^t)$ and receives a deterministic reward $r_h(s_h^t,a_h^t)$.

\paragraph{Regret} The quality of an agent is measured through its regret, that is the difference between what it could obtain (in expectation) by acting optimally and what it really gets,
\[
\regret^T \triangleq  \sum_{t=1}^T \Vstar_1(s_1)- V_1^{\pi^t}(s_1)\,.
\]

\paragraph{Counts} $n_h^{t}(s,a) \triangleq  \sum_{i=1}^{t} \ind{\left\{(s_h^i,a_h^i) = (s,a)\right\}}$ are the number of times the state action-pair $(s,a)$ was visited in step $h$ in the first $t$ episodes. Next, we define $n_h^{t}(s'|s,a) \triangleq \sum_{i=1}^{t} \ind{\big\{(s_h^i,a_h^i,s_{h+1}^i) = (s,a,s')\big\}}$ the number of transitions from $s$ to $s'$ at step $h$.

\paragraph{Pseudo counts and improper Dirichlet distribution} We define the pseudo counts as the counts shifted by initial pseudo counts $\upn_h^t(s,a) \triangleq n_h^t(s,a)+n_0$. For $m\in\N^*$ the simplex of dimension $m-1$ is denoted by $\simplex_{m-1}$. For $\alpha \in (\R_{++})^{m}$ we denote by $\Dir(\alpha)$ the Dirichlet distribution on $\simplex_{m-1}$ with parameter $\alpha$. We also extend this distribution to improper parameter $\alpha \in (\R_{+})^{m}$ such that $\sum_{i=1}^{m}  \alpha_i > 0$
 by injecting $\Dir((\alpha_i)_{i:\alpha_i>0})$ into $\simplex_{m-1}$. Precisely we say that $p \sim \Dir(\alpha)$ if $(p_i)_{i:\alpha_i>0}\!\!\sim\!\!\Dir((\alpha_i)_{i:\alpha_i>0})$ and all other coordinates are zero.

\paragraph{Additional notation} For $N\in\N_{++}$ we define the set $[N]\triangleq \{1,\ldots,N\}$. We denote the uniform distribution over this set by $\Unif[N]$. The vector of dimension $N$ with all entries one is $\bOne^N \triangleq  (1,\ldots,1)$. $\hp^{\,t}_h(s,a)$ is an empirical distribution defined as $\hp^{\,t}_h(s'|s,a) = n^{\,t}_h(s'|s,a) / n^{\,t}_h(s,a)$ if $n^{\,t}_h(s,a) >0$ else $\hp^{\,t}_h(s'|s,a)=1/S$, and $\up^{\,t}_h(s,a)$ is an pseudo-empirical measure defined as $\up^{\,t}_h(s,a) = \upn^{\,t}_h(s'|s,a) / \upn^{\,t}_h(s,a)$. Appendix~\ref{app:notations} gives a reference of the notation used.

\section{Algorithm }
\label{sec:algorithm}

In this section we describe the \BayesUCBVI algorithm. Similarly to \UCBVI,  we build upper confidence bounds (UCBs) on the Q-value functions and act greedily with respect to them. However, to construct the UCBs we instead use a  quantile  of certain posterior distribution. The name \BayesUCBVI highlights the link between our algorithm and the one of \citet{kaufmann12} for multi-arm bandits.

First, we extend the state space $\cS$ by an absorbing pseudo-state $s_0$ with reward $r_h(s_0,a) \triangleq  \ur > 1$ for all $h,a$ and transition probability distribution $p_h(s'| s_0,a) \triangleq  \ind{\{s'=s_0\}}$. A similar pseudo-state called ``garden of even'' was used  by \citet{Brafman02RMAX,szita2008}. We denote the extended state space by $\cS'\triangleq \cS\cup\{s_0\}$. The optimal value at $s_0$ is $\Vstar_h(s_0) = \ur(H-h+1)$ by definition. Next, we adopt a Bayesian model on the transition distributions. Note that it is   only a surrogate model used by the algorithm but not the one from which the transition are sampled. Precisely, the improper prior on the probability transition is a Dirichlet\footnote{See Section~\ref{sec:setting} for the extension of a Dirichlet distribution to parameter with coordinates that could be equal to zero.} distribution $\rho^0_h(s,a)= \Dir(\upn_h^0(s'|s,a)_{s\in\cS'})$ parameterized by the initial pseudo-count $\upn_h^0(s'|s,a) = n_0 \ind\{s'=s_0\}$. We recall that the pseudo-counts $\upn_h(s,a)$ are the counts plus a prior observation of a transition to the artificial state~$s_0$.  In fact, the prior is just a Dirac distribution at a deterministic transition $p_0(s') = \ind_{\{s'=s_0\}}$  leading to the artificial state $s_0$. Then the posterior is a Dirichlet distribution $\rho_h^t(s,a) = \Dir(\upn_h^t(s'|s,a)_{s\in\cS'})$. Given an upper bound on the value function at the next step $\uV_{h+1}^t,$ we set the upper confidence bound on the Q-value at step $h$ to the quantile of order $\kappa_h^t(s,a)$ of the distribution of Q-value where the transition probability distribution is sampled according to the posterior,
\[
\uQ_h^t(s,a) \triangleq r_h(s,a)+ \Q_{p\sim \rho_h^t(s,a)}\big(p\uV_{h+1}^t(s,a), \kappa_h^t(s,a)\big),
\]
where  for $\rho \in \simplex_{S'-1}$, $\kappa\in[0,1]$, $V: \cS' \to\R$, the quantile $\Q_{p\sim \rho}(pV,\kappa)$ of order $\kappa$ is defined as  $\P_{p\sim \rho}\big( pV \leq \Q_{p\sim \rho}(pV,\kappa) \big) = \kappa$. 

\vspace{0.25cm}
To compute UCBs on the value and Q-value functions for all $(s,a)\in\cS\times\cA$, we use an optimistic value iteration,
\label{eq: optQval iteration}
\begin{align}
  \uQ_h^{t}(s,a) & = r_h(s,a) + \Q_{p\sim \rho_h^t(s,a)}\big(p\uV_{h+1}^t(s,a), \kappa_h^t(s,a)\big)\nonumber \\
  \uV_h^{t}(s) & \triangleq \max_{a\in\cA} \uQ_h^{t}(s,a)\label{eq:optimistic_planning}\qquad \uV_h^{t}(s_0) \triangleq \Vstar_h(s_0)\\
  \uV_{H+1}^{t}(s) & \triangleq 0\nonumber\,.
\end{align}
\noindent
The complete procedure of \BayesUCBVI is described in Algorithm~\ref{alg:BayesUCBVItabular}.

\begin{algorithm}[h!]
\centering
\caption{\BayesUCBVI}
\label{alg:BayesUCBVItabular}
\begin{algorithmic}[1]
  \STATE {\bfseries Input:} quantile functions $(\kappa^t)_{t\in[T]}$, prior dist.\,$\rho^0$
      \FOR{$t \in[T]$}
      \STATE Optimistic planning, see \eqref{eq:optimistic_planning} 
    %   \begin{align*}
    %     \uQ_h^{t-1}(s,a) &= r_h(s,a)+ \Q_{p\sim \rho_h^{t-1}(s,a)}\big(p\uV_{h+1}^{t-1}(s,a), \kappa_h^{t-1}(s,a)\big) \\
    %     \uV_h^{t-1}(s) & = \max_{a} \uQ_h^{t-1}(s,a) \\
    %     \uV_{H+1}^{t-1}(s) & = 0\\
    %     \pi_h^t(s,a) & \in \argmax_{a\in\cA} \uQ_{h}^{t-1} (s,a)\,,
    %   \end{align*}
      \FOR{$h \in [H]$}
        \STATE Play $a_h^t \in\argmax_{a\in\cA} \uQ_h^{t-1}(s_h^t,a)$
        \STATE Observe $s_{h+1}^t\sim p_h(s_h^t,a_h^t)$
        \STATE{ Update the posterior distributions $\rho_h^t(s_h^t,a_h^t)$ with $(s_h^t,a_h^t,s_{h+1}^t)$}
      \ENDFOR
   \ENDFOR
\end{algorithmic}
\end{algorithm}

\subsection{Analysis}
%\todoPi{and if $n_h^t(s,a)=0?$}
We fix $\delta\in(0,1)$ and the quantile function
\begin{equation}
\kappa_h^t(s,a) \triangleq 1 - \frac{C_\kappa \delta}{SAH [2n^t_h(s,a)+1]^{3}[\upn^t_h(s,a)]^{3/2}}\,
\label{eq:def_quantile_function}
\end{equation}
up to an absolute constant $C_\kappa \triangleq 1/(5(\rme \pi)^{3})$.
We now state the main result of the paper, proved in Appendix~\ref{app:proof-BayesUCBVI} and with a proof sketch in Section~\ref{sec:sketch_proof_BayesUCBVI}.
\begin{theorem}
\label{th:regret_bound_bayesUCBVI} 
Consider a parameter $\delta>0$. Let $n_0 \triangleq \lceil c_{n_0} + \log_{17/16}(T) \rceil$, $\ur \triangleq 2$, where  $c_{n_0}$ is an absolute constant defined in \eqref{eq:constant_c_n0}; see Appendix~\ref{app:optimism}. Then for \BayesUCBVI, with probability at least $1-\delta$, 
\[
    \regret^T = \cO\left( \sqrt{H^3 SAT} L + H^3 S^2 A L^{2} \right),
\]
 where $L \triangleq \cO(\log(HSAT/\delta))$.
%  $\tcO(\cdot)$ stands for inequality up to terms which are  poly-$\log$   in $H,S,A,T,1/\delta.$
\end{theorem}

Notice that \BayesUCBVI matches the problem-independent lower bound $\Omega(\sqrt{H^3SAT})$ by \citet{jin2018is,domingues2021episodic} for $T\geq H^3S^3A$ and up to poly-$\log$ terms  in $H,S,A,T,1/\delta$.

\paragraph{Complexity} \BayesUCBVI is a model-based algorithm, and thus gets the $\cO(HS^2A)$ space complexity of \UCBVI. Note that there is no closed-form solution to compute the quantile used in the UCB and thus we approximate it e.g., by Monte-Carlo sampling; see Section~\ref{sec:non_tabular_extension}. In particular, if we use $B$ Monte-Carlo samples to approximate one quantile the time complexity of \BayesUCBVI is $\cO(BHS^2AT)$ for~$T$ episodes.

\paragraph{Comparison with \PSRL and \RLSVI} \BayesUCBVI is close to \PSRL  \citep{obsband2013more,agrawal2020posterior}. Instead of computing quantiles, \PSRL directly samples a transition probability distribution from the posterior to compute Q-values. Note that these Q-values may not necessary be UCBs as for \BayesUCBVI. \citet{agrawal2020posterior} proved a regret bound of order $\tcO(H^2S\sqrt{AT})$ for \PSRL. We believe that our analysis for \BayesUCBVI and in particular Theorem~\ref{th:two_sided_dbc}, could be used to improve the regret bound for \PSRL to $\tcO(\sqrt{H^3SAT}),$ thus matching (in terms of its dependence on the number of states $S$ and  the horizon $H$) the \textit{Bayesian} regret bound provided by \citet{osband2017why}. Another Bayesian-inspired algorithm is \RLSVI by \citet{obsband2013more}. It works by injecting Gaussian noise into the Bellman equations. Adding this noise can be seen as sampling accordingly to a certain posterior on the Q-value functions \citep{russo2014learning}. Recently, \citet{xiong2021nearoptimal} improved the dependence on the horizon~$H$ in \RLSVI's regret bound to $\tcO(\sqrt{H^3SAT})$ thanks to a Gaussian noise with a ``Bernstein'' shaped variance. Yet, it is not clear if this variant of \RLSVI %inherits the Bayesian interpretation and if there is%
has a clean extension beyond the tabular setting. The interesting property of  \BayesUCBVI is that its Dirichlet posterior on the transitions \textit{adjusts automatically to the variance without the need of Bernstein bonuses/noises}.
Moreover, \citet{pacchiano2021towards} proposed to replace the Gaussian noise in \RLSVI by bootstrap sampling of past rewards and adding pseudo-rewards as \citet{kveton2019garbage}. They proved\footnote{We hypothesise that the noise should be scaled by $H$ as in \RLSVI for their result to be valid.} a regret bound of order $\tcO(H^2S\sqrt{AT})$ for this type of noise. 
Note that in \BayesUCBVI it is  targets rather than the rewards that are used in the (Bayesian) bootstrap, see Section~\ref{sec:non_tabular_extension}.

\subsection{Proof sketch}
%\todom{somehow higlight that Honda etc could not get lower bound tight}
\label{sec:sketch_proof_BayesUCBVI}
We now sketch the proof of Theorem~\ref{th:regret_bound_bayesUCBVI}.
The proof  relies heavily on \textit{boundary-crossing probabilities for weighted sums of the Dirichlet distribution with integer parameter}.
The result below gives tight bounds for these probabilities. The lower bound in particular, 
is one of our main technical contributions.
\par
\paragraph{Step 1. Dirichlet boundary crossing}
%\todoPi{What is the definition of $\gtrsim$? Can we rather put a numerical constant?}
%\todoDa{Fixed!}
%\todom{Kinf is not defined yet}
\begin{theorem}[see Lemma~\ref{lem:upper_bound_dbc} and Theorem~\ref{thm:lower_bound_dbc}]\label{th:two_sided_dbc}
    For any $\alpha = (\alpha_0, \alpha_1, \ldots, \alpha_m) \in \N^{m+1}$ define $\up \in \simplex_m$ with $\up(\ell) = \alpha_l/\ualpha, \ell=0,\ldots,m$, where $\ualpha = \sum_{j=0}^m \alpha_j$. Assume that $\alpha_0 \geq \log_{17/16}(\ualpha) + c_{n_0}$, where $c_{n_0}$ is defined in \eqref{eq:constant_c_n0}; see Appendix~\ref{app:optimism}, and $\ualpha \geq 2\alpha_0$. Then for any $f \colon \{0,\ldots,m\} \to [0, \ub]$ such that $f(0) = \ub,\, f(\ell) \leq b < \ub/2,\, \ell \in \{1,\ldots,m\}$ and any $\mu \in (\up f, \ub)$ we have
    \begin{small}
    \[
        \frac{\rme^{-\ualpha \Kinf(\up, \mu, f)}}{\ualpha^{3/2}} \leq \P_{w \sim \Dir(\alpha)}[wf \geq \mu] \leq \rme^{-\ualpha \Kinf(\up, \mu, f)}, 
    \]
    \end{small}
    %\!\!where $\gtrsim$ stands for inequality up to an absolute constant and 
    where $\Kinf(p,u,f)$ is given by
    \begin{small}
    \[
     \Kinf(p,u,f) \triangleq \max_{\lambda \in[0,1]} \E_{X\sim p}\left[ \log\left( 1-\lambda \frac{f(X)-u}{b_0-u}\right)\right].
    \]
    \end{small}
\end{theorem}
While the upper bound follows directly from the work of \citet{riou20a},   the lower bound is new. The proof of the lower bound is presented in Theorem~\ref{thm:lower_bound_dbc} and consists of two main steps:

\vspace{-10pt}
\begin{enumerate}
\item Geometrical reduction of the density of $wf$ to 1D complex integral (see \citealp{dirksen2015sections,lasserre2020simple});
\item Sharp non-asymptotic analysis of the integral using the saddle-point method (see \citealp{olver1997asymptotics,fedoryuk1977metod}).
\end{enumerate}
\vspace{-10pt}

Using Theorem~\ref{th:two_sided_dbc}, we show that $\uQ^t$ is an upper confidence bound on the optimal action-value function.

\paragraph{Step 2. Optimism}
Using the lower bound of Theorem~\ref{th:two_sided_dbc}, we show that for our choice of $\kappa^t_h(s,a),$ given in  \eqref{eq:def_quantile_function}, the following result holds.
\begin{lemma}[see Lemma~\ref{lem:first_quantile_bound}]\label{lem: quantile lower bound}
    Let $n_0 \geq \log_{17/16}(\ualpha) + c_{n_0}$ and $\ur \geq 2$, where $c_{n_0}$ is defined in \eqref{eq:constant_c_n0}; see Appendix~\ref{app:optimism}. Then on event $\cE^\star(\delta)$; see Appendix~\ref{app:concentration}, for any $t \in \N, h \in [H], (s,a) \in \cS \times \cA$,
    \begin{small}
    \begin{equation*}
        \Q_{p\sim \rho_h^t(s,a)}\big(p\Vstar_{h+1}(s,a), \kappa_h^t(s,a)\big) \geq p_h \Vstar_{h+1}(s,a).
    \end{equation*}
    \end{small}
\end{lemma}
%\todom{is below reference to decomposition 3 OK?}
% Fixed!
By the decomposition \eqref{eq:optimistic_planning} and the Bellman equation, we see that
\begin{small}
\begin{align*}
    \uQ^t_h&(s,a) - \Qstar_h(s,a) \\
    &\geq \Q_{p\sim \rho_h^t(s,a)}\big(p\uV^t_{h+1}(s,a), \kappa_h^t(s,a)\big) - p_h \Vstar_{h+1}(s,a).
\end{align*}
\end{small}
\!\!Induction over $h$ and Lemma~\ref{lem: quantile lower bound} yield 
that  on event $\cE^\star(\delta),$ $\uQ^t_h(s,a) \geq \Qstar_h(s,a)$ for any $t \leq T, h \in [H], (s,a) \in \cS \times \cA$.

\paragraph{Step 3. Reduction to \UCBVI with Bernstein bonuses}

By optimism we have
\begin{small}
\[
    \regret^T \triangleq \sum_{t=1}^T \Vstar_1(s^t_1) - V^{\pi_t}_1(s^t_1)  \leq \sum_{t=1}^T \delta^t_1,
\]
\end{small}
\!\!where $\delta^t_h \triangleq \uV^{t-1}_h(s^t_1) - V^{\pi_t}_h(s^t_1)$. The quantity $\delta^t_h$ can be decomposed as follows
\begin{scriptsize}
    \begin{align*}
        \delta^t_h &=\underbrace{\Q_{p \sim \rho^{t-1}_h(s^t_h, a^t_h)}( p \uV^{t-1}_{h+1}(s^t_h,a^t_h), \hat\kappa^{t}_h) - \up^{t-1}_h \uV^{t-1}_{h+1}(s^t_h,a^t_h)}_{\termA} \\
        &+ \underbrace{[\up^{t-1}_h - \hp^{t-1}_h] \uV^{t-1}_{h+1}(s^t_h, a^t_h)}_{\termB} \\
        &+  \underbrace{(\hp^{t-1}_h - p_h)  [\uV^{t-1}_{h+1} - \Vstar_{h+1}](s^t_h, a^t_h)}_{\termC} + \underbrace{(\hp^{t-1}_h - p_h) \Vstar_{h+1}(s^t_h, a^t_h)}_{\termD} \\
        &+ \underbrace{p_h[\uV^{t-1}_{h+1} - V^{\pi_t}_{h+1}](s^t_h, a^t_h) - [\uV^{t-1}_{h+1} - V^{\pi_t}_{h+1}](s^t_{h+1})}_{\xi^t_h}+ \delta^t_{h+1},
    \end{align*}
\end{scriptsize}
\!\!where $\hat\kappa^t_h = \kappa^{t-1}_h(s^t_h,a^t_h)$. The terms $\termC, \termD$ and $\xi^t_h$ coincide with similar terms in the analysis of \UCBVI with Bernstein bonuses. The term $\termB$ could be upper-bounded by $\frac{\ur H}{\max\{n^{t-1}_h, 1\}}$ and turns out to be a second-order term. The analysis of the term $\termA$ is novel. Using the upper bound from Theorem 3.2 we may obtain the Bernstein type inequality for the Dirichlet distribution (see Lemma~\ref{lem:bernstein_dirichlet} in Appendix~\ref{app:concentration}) which yields the following key inequality for the quantile $\Q_{p \sim \rho^t_h(s,a)}(p \uV^t_{h+1}(s,a), \kappa^t_h(s,a))$. 
\begin{lemma}[see Corollary~\ref{cor:quantile_bounds}]\label{lem: quantile upper bound}
     Assume conditions of Theorem~\ref{th:regret_bound_bayesUCBVI}. On event $\cE^\star(\delta)$,  
     for any $t \in \N, h \in [H], (s,a) \in \cS \times \cA,$
     \begin{small}
    \begin{align*}
        \Q&_{p \sim \rho^t_h(s,a)}(p \uV^t_{h+1}(s,a), \kappa^t_h(s,a)) \\
        &\leq \up^t_h \uV^t_{h+1}(s,a) +  2\sqrt{ \Var_{\up^t_h}[\uV^t_{h+1}](s,a) \frac{  \log\left( \frac{1}{1 - \kappa^t_h(s,a)}\right)}{\upn^t_h(s,a)}} \\
        &\qquad+ \frac{2 \sqrt{2} \cdot \ur H \log\left( \frac{1}{1 - \kappa^t_h(s,a)}\right)}{\upn^t_h(s,a)}.
    \end{align*}
    \end{small}
\end{lemma}

Since $1-\kappa^t_h(s,a)$ depends on $n^t_h(s,a)$ only polynomially, we see that the term $\termA$ can be upper bounded by a quantity which looks very similar to the Bernstein bonuses in \UCBVI and, moreover, it has the same role in the regret analysis.  After using these upper bounds, the rest of the proof follows from the analysis of \UCBVI with the  Bernstein bonuses; see \citealp{azar2017minimax}.

% Indeed, by optimism we have
% \begin{small}
% \[
%     \regret^T = \sum_{t=1}^T \Vstar_1(s^t_1) - V^{\pi_t}_1(s^t_1)  \leq \sum_{t=1}^T \delta^t_1,
% \]
% \end{small}
% \!\!where $\delta^t_h = \uV^t_h(s^t_1) - V^{\pi_t}_h(s^t_1)$. The quantity $\delta^t_h$ can be decomposed as follows
% \begin{scriptsize}
%     \begin{align*}
%         \delta^t_h &= \underbrace{\Q_{p \sim \rho^k_h(s^t_h, a^t_h)}( p \uV^t_{h+1}(s^t_h,a^t_h), \kappa^t_h(s^t_h,a^t_h)) - \up^t_h \uV^t_{h+1}(s^t_h,a^t_h)}_{\termA} \\
%         &+ \underbrace{[\up^t_h - \hp^t_h] \uV^t_{h+1}(s^t_h, a^t_h)}_{\termB} \\
%         &+  \underbrace{(\hp^t_h - p_h)  [\uV^{t}_{h+1} - \Vstar_{h+1}](s^t_h, a^t_h)}_{\termC} + \underbrace{(\hp^t_h - p_h) \Vstar_{h+1}(s^t_h, a^t_h)}_{\termD} \\
%         &+ \underbrace{p_h[\uV^{t}_{h+1} - V^{\pi_t}_{h+1}](s^t_h, a^t_h) - [\uV^{t}_{h+1} - V^{\pi_t}_{h+1}](s^t_{h+1})}_{\xi^t_h}+ \delta^t_{h+1}.
%     \end{align*}
% \end{scriptsize}
% \!\!The terms $\termC, \termD$ and $\xi^t_h$ coincide with similar terms in the analysis of \UCBVI with Bernstein bonuses. The analysis of the terms $\termA$ and $\termB$ is more novel. In particular,
% \begin{enumerate}
% \item the term $\termA$ could be upper-bounded with Bernstein-type inequality using Lemma~\ref{lem: quantile upper bound}.
% \item the term $\termB$ could be upper-bounded by $\frac{\ur H}{n^t_h}$ and turns out to be a second-order term.
% \end{enumerate}

% \todom{maybe a ref where such proof can be found?}
\section{\texorpdfstring{\BayesUCBVI}{Bayes-UCBVI} for Deep RL}
\label{sec:non_tabular_extension}
We now extend \BayesUCBVI beyond the tabular setting. Fix a state-action pair $(s,a)$. At episode $t$, the targets to estimate the Q-value function at state-action pair $(s,a)$ at step~$h$ are $y_h^n(s,a) \triangleq r_h(s,a) + \uV_{h+1}^t(s_{h+1}^n)$ for $n\in[n_h^t(s,a)]$ where $s_{h+1}^n$ is
the next state observed after taking the action~$a$ in state~$s$ for the $n^{\text{th}}$ time.\footnote{In particular, $n$ is a number of visits of a state-action pair $(s,a)$ and not the global time (the number of episodes).} We also need prior targets\footnote{In the case of unknown rewards we use the sample returns instead. For the pseudo-target, we always set the rewards to $1$ which gives $y_h^0(s,a) \triangleq H-h+1$.} $y_h^n(s,a) \triangleq r_h(s,a) + \uV_h^t(s_0)$ for $(-n+1)\in[n_0]$ corresponding to the pseudo-transition to $s_0$. Using the \textit{aggregation property of the Dirichlet distribution} we can compute the UCB by taking the quantile of randomly re-weighted sum of targets. Precisely, we have that
\begin{small}
\begin{align*}
\uQ_h^t(s,a) & \triangleq  r_h(s,a)+ \Q_{p\sim \rho_h^t(s,a)}\big(p\uV_{h+1}^t(s,a), \kappa_h^t(s,a)\big)\\
&= \Q_{w\sim \Dir(\bOne^{\upn_h^t(s,a)})}\left(\sum_{n=-n_0+1}^{n_h^t(s,a)} w_n y_h^n(s,a), \kappa_h^t(s,a)\right)\,.
\end{align*}
\end{small}
\!\!\!We can approximate this quantile by the quantile of the empirical distribution of $B$ Bayesian bootstrap samples \citep{rubin1981bayesian}. Precisely, if we fix $(w_h^{b}(s,a))_{b\in[B]}$ i.i.d.\,samples from a Dirichlet distribution $\Dir(\bOne^{\upn_h^t(s,a)})$ we have
\begin{align*}
\uQ_h^t(s,a) \approx\ &\Q_{b\sim\Unif([B])}\left(\uQ_h^{t,b}(s,a),\kappa_h^t(s,a)\right) \\
\text{ where }&\uQ_h^{t,b}(s,a) \triangleq \!\!\sum_{n=-n_0+1}^{n_h^t(s,a)} \!w_h^{n,b}(s,a) y_h^n(s,a)\,.
\end{align*}
In particular, using that a uniform Dirichlet distribution can be obtained by normalizing independent samples from the exponential probability distribution $\Exponential(1)$, we can obtain the Bayesian samples by solving a weighted linear regression
\begin{align}
  \uQ_h^{t,b}(s,a) =& \argmin_{x} \sum_{n=-n_0+1}^{n_h^t(s,a)} z_h^{n,b}(s,a) \left(x - y_h^n(s,a)\right)^2 \label{eq:weighted_regression}\\
  &\text{ where }z_h^{n,b}(s,a) \sim \Exponential(1)\text{ i.i.d.}\,.
\nonumber\end{align}
We name this approximation of \BayesUCBVI, \textit{incremental Bayes-UCBVI} (\IncrBayesUCBVI) and provide its detailed pseudo-code as Algorithm~\ref{alg:incrBayesUCBVI} in Appendix~\ref{app:non-tabular-extension-detailed}. Note that this way to generate bootstrap sample is similar to the incremental Bayesian bootstrap by \citealp{osband2015bootstrap} (their Algorithm~5).

A great advantage of this formulation of \BayesUCBVI is that it can be easily extended beyond the tabular setting. Indeed, we can simply replace the weighted linear regression loss in~\eqref{eq:weighted_regression} by the weighted regression loss of any function approximation. Remarkably, except for the initial pseudo-transitions, \IncrBayesUCBVI does not rely on counts but on a easy-to-implement Bayesian bootstrap. As an example,  in Appendix~\ref{app:non-tabular-extension-detailed}, we combine the \IncrBayesUCBVI exploration procedure with  \DQN \citep{mnih2015humanlevel} and call it \BayesUCBDQN, detailed as Algorithm~\ref{alg:BayesUCBDQN} of Appendix~\ref{app:non-tabular-extension-detailed}.

 \subsection{Related work}
 \label{sec:related-work-deepRL}
Generalizing principled solutions of the exploration-exploitation dilemma from the theoretical tabular RL setting to large-scale deep RL is quite challenging \citep{yang2021exploration}. For instance, \citet{bellemare2016unifying,ostrovski2017count} approach the count-based UCBs used in tabular RL by approximating the visits counts using a density estimation. Later, \citet{tang2017exploration} directly map states to hash codes and then count their occurrences in a hash table. Another line of work sets bonuses to the approximation error of some quantities related to the MDP dynamics: the forward dynamics \citep{schmidhuber1991possibility,pathak2017curiosity}, the inverse dynamics \citep{haber2018learning} or simply a constant function \citep{choshen2018dora}. Similarly, \citet{burda2019exploration} builds bonus from he prediction error of a randomly initialized network. 
This can be further combined with the pseudo-counts \citep{badia2020never} leading to impressive results.
As in the tabular setting, a second line of work deals with the exploration-exploitation trade-off by injecting noise. \citet{fortunato2018noisy} add parametric noise to the weights of the agent's network that is learned with the weights. \citet{azizzadenesheli2018efficient} approximate \PSRL by replacing the typical last linear layer of agent's Q-value network with a Bayesian linear regression. Alternatively, bootstrap \DQN (\BootDQN, \citealp{osband2016deep,osband18randomized,osband2019deep}) extends \RLSVI by bootstrap sampling of the transitions to inject noise into \DQN. Specifically, in \BootDQN an ensemble of Q-value functions is learned each on a \textit{different} bootstrap sample of the transitions collected so far. Building on this work, \citet{nikolov2019information} use bootstrap as well but instead combines the bandit algorithm, information direct sampling \citep{russo2014learning}, with \DQN. Recently, \citet{bai21principled} also proposed an optimistic algorithm based on bootstrap, using a bonus that scales with the variance of the ensemble of Q-value functions learned as did \citet{osband2019deep}.

\vspace{-5pt}
% \todom{how about discussing the need/no need for "independence" between the heads?, suggestion from Pierre: 
% Bootdqn they train each head separately (each head receive target from the corresponding heads)  and they claim it is an important point to perform deep exploration. Whereas in BayesUCBDQN all the heads share the same target and we get the same result which is surprising. We want to explore deeper this intriguing point in the future...}
\paragraph{Further comparison of \BayesUCBDQN with \BootDQN} \hspace{-10pt}\BayesUCBDQN is close to \BootDQN of \citet{osband2016deep}. In \BootDQN, an ensemble of $B$ bootstrap Q-value functions (or in practice only the ``heads'' of a unique Q-value function) are learned with different sub-sets of transitions. Each transition is randomly assigned to the training of one bootstrap Q-value function with a fixed probability $p\in[0,1]$. In particular they consider $p=0.5$ for the double-or-nothing bootstrap and $p=1$ for no bootstrap. Each bootstrap Q-value function is then trained with targets computed from the \emph{corresponding} bootstrap Q-value function at the next state; see Appendix~\ref{app:experiment_details} for a detailed description. As explained by \citet{osband2016deep}, \BootDQN can be seen as approximation of \RLSVI.  That is why \BayesUCBDQN can be seen as an optimistic version of \BootDQN (as \BayesUCBVI is an optimistic version of \PSRL \& \RLSVI). The main differences between \BootDQN and \BayesUCBDQN are therefore: (i) \BayesUCBDQN acts greedily with respect to the quantile of the bootstrap Q-value functions instead of one bootstrap Q-value function sampled uniformly at random. (ii) \BayesUCBDQN uses Bayesian bootstrap instead of the classical bootstrap \citep{efron1979bootstrap}. (iii) In \BayesUCBDQN, the bootstrap Q-value functions are trained with the same target computed with the quantile of the bootstrap Q-functions at the next step, as in~\eqref{eq:weighted_regression}. We discuss the impacts of these modifications in Section~\ref{sec:experiments}.
\section{Experimental Results}
\label{sec:experiments}

In this section we provide experiments on \BayesUCBVI and its variants. We illustrate two points: First, that \IncrBayesUCBVI performs similarly as other algorithms relying on noise-injection for exploration such that \PSRL and \RLSVI. Second, that \BayesUCBDQN, the deep RL extension of \BayesUCBVI is competitive with \BootDQN. 

\subsection{Tabular environment}
We first evaluate \BayesUCBVI and \IncrBayesUCBVI on a simple tabular environment.

\paragraph{Environment} For the tabular experiments we consider a simple grid-world with $5$ connected rooms of size $5\times5,$ totalling $S=129$ states. The agent starts in the middle room. There is one small deterministic reward in the leftmost room, one large deterministic reward in the rightmost room and zero reward elsewhere. The agent can take $A=4$ actions: moving up, down, left, right. When taking an action, the agent moves in the corresponding direction with probability $0.9$ and moves to a neighboring state at random with probability $0.1$. The horizon is fixed to $H=30$; see Appendix~\ref{app:experiment_details} for details. In this environment the agent must explore efficiently all the room avoiding being lured by the small reward in the  leftmost room.

\paragraph{Baselines} We compare \BayesUCBVI and \IncrBayesUCBVI with the following baselines: \UCBVI \citep{azar2017minimax}, \RLSVI \citep{osband16generalization}, and \PSRL \citep{obsband2013more}; see Appendix~\ref{app:experiment_details} for a full description of the parameters of the algorithms used in the experiments.

\paragraph{Results} In Figure~\ref{fig:regret_baselines}, we plot the regret of the various baselines, \BayesUCBVI and \IncrBayesUCBVI in the aforementioned environment. In this experiment, we observe that both \BayesUCBVI and \IncrBayesUCBVI achieve competitive results with respect to baselines relying on noise-injection for exploration (\PSRL, \RLSVI). This is remarkable, since the \textit{common belief is that optimistic algorithm perform poorly in practice} \citep{osband2017why}. Indeed, \IncrBayesUCBVI exhibits a  regret similar to \PSRL. It is not completely surprising since they share the same model on the transitions (up to the prior). Notice that \BayesUCBVI performs slightly worse than \IncrBayesUCBVI but better than \RLSVI. One possible reason to explain this gap between \BayesUCBVI and \IncrBayesUCBVI is that the incremental implementation of Bayesian bootstrap forgets  the prior  faster than the non-incremental version, resulting in a more aggressive algorithm. We also note that \RLSVI performs slightly worse than \PSRL, \IncrBayesUCBVI but much better than \UCBVI. A possible explanation for this ranking is that \RLSVI is much more aggressive than \UCBVI when they have comparable noise, bonuses; whereas \PSRL, \IncrBayesUCBVI, \BayesUCBVI take better advantage of the small variance of this particular environment than the two last baselines.

\begin{figure}[h!]
      \vspace{-0.2cm}
    \centering
    \includegraphics[keepaspectratio,width=.48\textwidth]{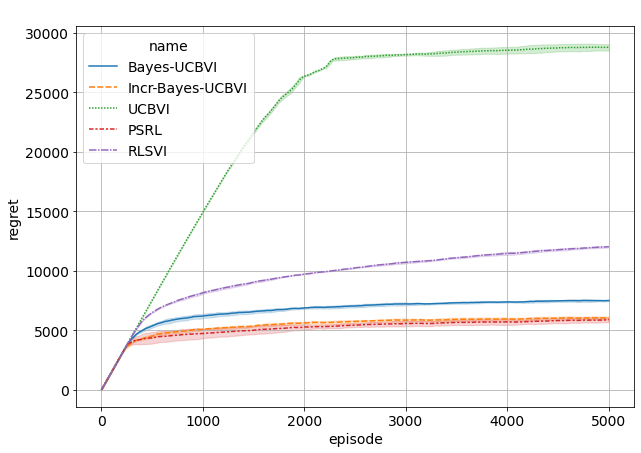}
    \caption{Regret of \BayesUCBVI and \IncrBayesUCBVI compared to baselines for $H=30$ an transitions noise $0.1$. We show $\text{average}$ over $4$ seeds.}
    \label{fig:regret_baselines}
\end{figure}

\subsection{Deep RL experiments}
In this section we evaluate the performance of \BayesUCBDQN in large-scale environments. 

\paragraph{Setup}
All algorithms are based on the architecture of \DQN \citep{mnih2013playing}. In order to implement the bootstrapped ensemble, we follow \BootDQN \citep{osband2015bootstrap} and maintain an ensemble of $B=10$ head networks over a shared torso network. 
For fairness of comparison, all algorithmic variants share hyper-parameters wherever possible; see Appendix~\ref{app:experiment_details} for further details on the detailed architecture and implementation details. 

\paragraph{Environment and evaluation} To evaluate the scalability of \BayesUCBDQN, we train DQN variants over a suite of $57$ Atari games \citep{bellemare2013arcade}. For each algorithm and each game, we train for $200$M frames and record the human normalized scores per game. The overall performance curve in Figure~\ref{fig:deeprl} is calculated as the median score over all games.

\paragraph{Results} We compare \texttt{DoubleDQN\xspace} \citep{van2016deep}, \BootDQN and \BayesUCBDQN. In Figure~\ref{fig:deeprl}, we show the evaluation performance of different algorithms over training, measured in median human normalized scores. We make a few observations: (1) Both \BayesUCBDQN and \BootDQN outperform \DoubleDQN, potentially due to better training stability thanks to more consistent exploration; (2) The performance of \BootDQN converges to about $0.7$, which is consistent with results of \citet{osband2015bootstrap}; (3) Overall, \BayesUCBDQN and \BootDQN perform similarly. We see that \BayesUCBDQN achieves very marginal advantage over \BootDQN towards the end of training, however, more significant gains might require further engineering efforts. Nevertheless, we have established that \BayesUCBDQN, as a theoretically grounded algorithm, is competitive with \BootDQN. This paves the way for future research in this space; see Appendix~\ref{app:experiment_details} for further discussion on the effect of various hyper-parameters. 

\begin{figure}[t]
    \centering
    \includegraphics[keepaspectratio,width=.48\textwidth]{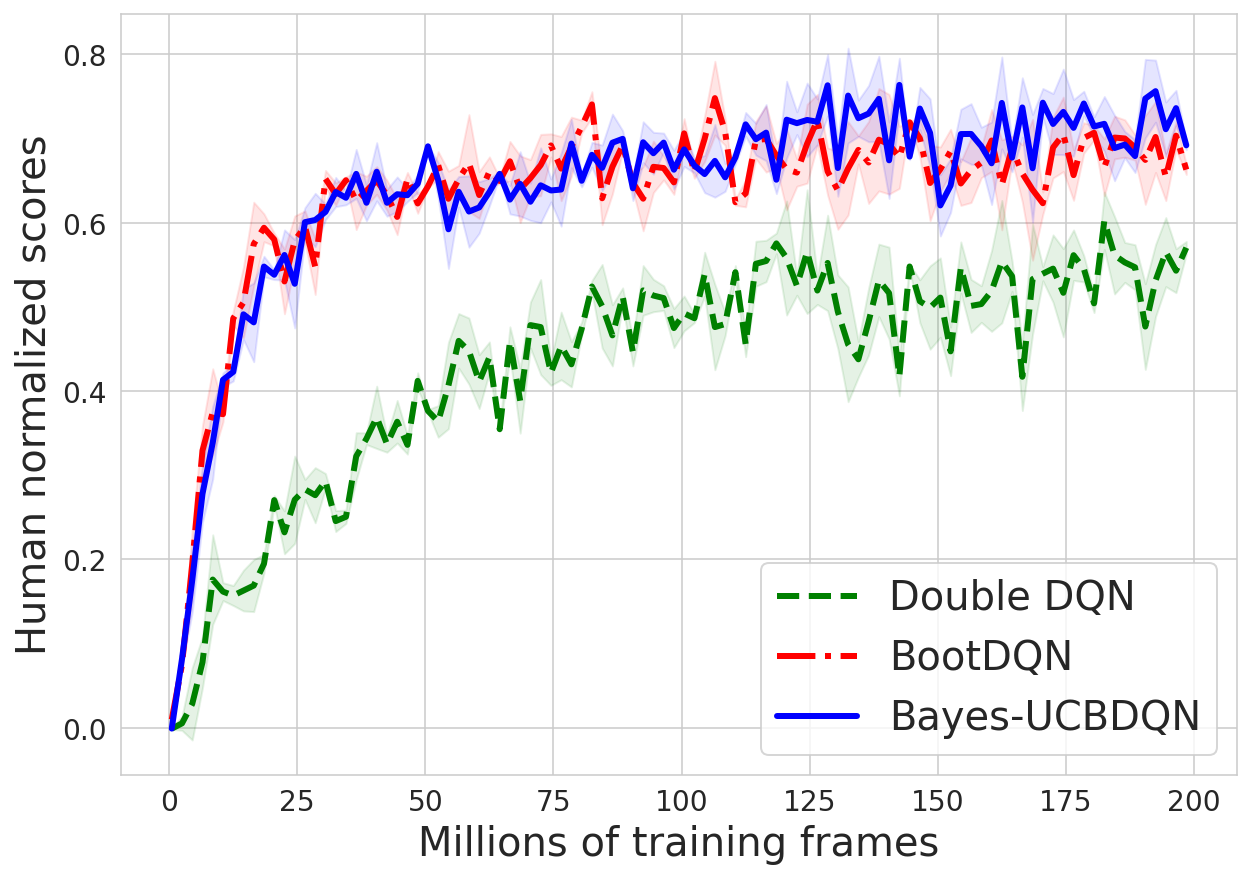}
    \caption{Evaluating deep RL algorithms with median human normalized scores across Atari-57 games. We compare \DoubleDQN, \BootDQN and \BayesUCBDQN. The training curves show the $\text{average}\pm\text{std}$ over $3$ seeds.}
    \label{fig:deeprl}
\end{figure}

%\todom{independence and that we "explain" why using all the data in all the heads makes sense}
\section{Conclusion}
\label{sec:conclusion}

We presented a new algorithm, \BayesUCBVI. It is an optimistic algorithm that does not rely on bonuses but rather uses the quantile of a well-chosen posterior to inject optimism. We proved that this algorithm is problem-independent optimal up to term poly-$\log$ in the horizon~$H$, the number of action $A$, states $S$ and episodes $T$. \BayesUCBVI also exhibits similar empirical performance than other existing Bayesian-inspired algorithms thus bridging the optimal problem-independent theoretical guarantees of optimistic algorithms and the good empirical results of algorithms relying on noise-injection for exploration. Importantly we also demonstrated that \BayesUCBVI could easily be extended beyond the tabular setting. In particular, we provided a new principled algorithm \BayesUCBDQN based on \BayesUCBVI that is competitive with \BootDQN of \citet{osband2019deep} on large-scale environments.

This work also raises the following open question that we think bring interesting future directions.

\paragraph{Problem-independent optimality of \PSRL} Central to the proof of the regret bound of \PSRL is the control of the deviation of a Dirichlet re-weighted sum \citep{agrawal2020posterior}. Thus, we believe that the anti-concentration inequality of Theorem~\ref{th:two_sided_dbc}, or a close variant, could allow to improve the regret bound to $\tcO(\sqrt{H^3SAT})$ (for $T$ large enough). In particular, this would imply that \PSRL is also problem-independent optimal. 

\paragraph{Integration with deep RL agents} Existing deep RL architectures, such as the implementations of base agent's loss functions and training pipeline, might not interact with our proposed exploration techniques in the optimal way (see Appendix~\ref{app:experiment_details_discussion} for details). Thus, an open question is whether we could make more fundamental changes to certain deep RL agents so that the exploration methods can be integrated in a better way.

\section*{Acknowledgements}
D.\,Belomestny acknowledges the financial support from Deutsche Forschungsgemeinschaft (DFG), Grant Nr.\,497300407. P.\,Ménard  is supported by the SFI Sachsen-Anhalt for the project RE-BCI ZS/2019/10/102024 by the Investitionsbank Sachsen-Anhalt. The work of D.\,Tiapkin, A.\,Naumov, and S.\,Samsonov were supported by the grant for research centers in the field of AI provided by the Analytical Center for the Government of the Russian Federation (ACRF) in accordance with the agreement on the provision of subsidies (identifier of the agreement 000000D730321P5Q0002) and the agreement with HSE University No.\,70-2021-00139.

\bibliographystyle{plainnat}
\bibliography{BayesUCBVI-bib.bib}

\appendix
\onecolumn

\part{Appendix}
\parttoc
\newpage
\section{Notation}
\label{app:notations}

\begin{table}[h]
	\centering
	\caption{Table of notation use throughout the paper}
	\begin{tabular}{@{}l|l@{}}
		\toprule
		\thead{Notation} & \thead{Meaning} \\ \midrule
	$\cS$ & state space of size $S$\\
	$\cA$ & action space of size $A$\\
	$H$ & length of one episode\\
	$T$ & number of episodes\\
	$B$ & number of bootstrap samples\\
	\hline
	$r_h(s,a)$ & reward \\
	$p_h(s'|s,a)$ & probability transition \\
	$Q^{\pi}_h(s,a)$ & Q-function of a given policy $\pi$ at step $h$\\
	$V^{\pi}_h(s)$ & V-function of a given policy $\pi$ at step $h$\\
	$\Qstar_h(s,a)$ & optimal Q-function at step $h$\\
	$\Vstar_h(s)$ & optimal V-function at step $h$ \\
	$\regret^T $ & regret \\
	\hline
	$n_0$ & number of fake samples \\
	$s_0$ & fake state \\
	$\ur$ & reward of fake transition \\
	\hline
	$s^{\,t}_h$ & state that was visited at $h$ step during $t$ episode \\
	$a^{\,t}_h$ & action that was picked at $h$ step during $t$ episode \\
	$n_h^t(s,a)$ & number of visits of state-action $n_h^t(s,a) = \sum_{k = 1}^t  \ind{\left\{(s_h^k,a_h^k) = (s,a)\right\}}$\\
	$n_h^t(s'|s,a)$ & number of transition to $s'$ from state-action $n_h^t(s'|s,a) = \sum_{k = 1}^t  \ind{\left\{(s_h^k,a_h^k, s_{h+1}^k) = (s,a,s')\right\}}$. \\
	$\upn_h^t(s,a)$ & pseudo number of visits of state-action $\upn_h^t(s,a)=n_h^t(s,a)+n_0$\\
	$\upn_h^t(s'|s,a)$ & pseudo number of  transition to $s'$ from state-action $\upn_h^t(s'|s,a)=n_h^t(s'|s,a) + \ind{\{s' = s_0\}} \cdot n_0$\\
	$\hp_h^{\,t}(s'|s,a)$ & empirical probability transition $\hp_h^{\,t}(s'|s,a) = n_h^t(s'|s,a) / n_h^t(s,a)$ \\
	$\up_h^t(s'|s,a)$ & pseudo-empirical probability transition $\up_h^t(s'|s,a) = \upn_h^t(s'|s,a) / \upn_h^t(s,a)$ \\
%	$\dirac_{x}(s'|s,a)$ & Dirac transition $\dirac_{x}(s'|s,a) = \ind{\{s' = x\}}$ \\
	\hline
    $\uQ_h^t(s,a)$ & upper bound on the optimal Q-value\\
    $\uV_h^t(s,a)$ & upper bound on the optimal V-value\\
    \hline 
    $\R_+$ & non-negative real numbers \\
    $\R_{++}$ & positive real numbers \\
    $\N_{++}$ & positive natural numbers \\
    $[n]$ & set $\{1,2,\ldots, n\}$ \\
    $\simplex_d$ & $d$-dimensional probability simplex: $\simplex_d = \{x \in \R_{+}^{d+1}: \sum_{j=0}^{d} x_j = 1 \}$ \\ 
    $\bOne^N$ & vector of dimension $N$ with all entries one is $\bOne^N \triangleq  (1,\ldots,1)$ \\
    $\norm{x}_1$ & $\ell_1$-norm of vector $\norm{x}_1 = \sum_{j=1}^m \vert x_j \vert$\\
    $\norm{x}_2$ & $\ell_2$-norm of vector $\norm{x}_2 = \sqrt{\sum_{j=1}^m x_j^2}$\\
    $\norm{f}_2 $ & for $f \colon \Xset \to \R$, where $\vert \Xset \vert <\infty$ define $\norm{f}_2 = \sqrt{\sum_{x \in \Xset} f^2(x)}$ \\
    \bottomrule
	\end{tabular}
\end{table}

Let $(\Xset,\Xsigma)$ be a measurable space and $\Pens(\Xset)$ be the set of all probability measures on this space. For $p \in \Pens(\Xset)$ we denote by $\E_p$ the expectation w.r.t. $p$. For random variable $\xi: \Xset \to \R$ notation $\xi \sim p$ means $\operatorname{Law}(\xi) = p$. We also write $\E_{\xi \sim p}$ instead of $\E_{p}$. For independent (resp. i.i.d.) random variables $\xi_\ell \mysim p_\ell$ (resp. $\xi_\ell \mysimiid p$), $\ell = 1, \ldots, d$, we will write $\E_{\xi_\ell \mysim p_\ell}$ (resp. $\E_{\xi_\ell \mysimiid p}$), to denote expectation w.r.t. product measure on $(\Xset^d, \Xsigma^{\otimes d})$. For any $p, q \in \Pens(\Xset)$ the Kullback-Leibler divergence $\KL(p, q)$ is given by
$$
\KL(p, q) = \begin{cases}
\E_{p}[\log \frac{\rmd p}{\rmd q}], & p \ll q \\
+ \infty, & \text{otherwise}
\end{cases} 
$$
For any $p \in \Pens(\Xset)$ and $f: \Xset \to \R$, $p f = \E_p[f]$. In particular, for any $p \in \simplex_d$ and $f: \{0, \ldots, d\}   \to  \R$, $pf =  \sum_{\ell = 0}^d f(\ell) p(\ell)$. Define $\Var_{p}(f) = \E_{s' \sim p} \big[(f(s')-p f)^2\big] = p[f^2] - (pf)^2$. For any $(s,a) \in \cS$, transition kernel $p(s,a) \in \Pens(\cS)$ and $f \colon \cS \to \R$ define $pf(s,a) = \E_{p(s,a)}[f]$ and $\Var_{p}[f](s,a) = \Var_{p(s,a)}[f]$.

We write $f(S,A,H,T) = \cO(g(S,A,H,T,\delta))$ if there exist $ S_0, A_0, H_0, T_0, \delta_0$ and constant $C_{f,g}$ such that for any $S \geq S_0, A \geq A_0, H \geq H_0, T \geq T_0, \delta < \delta_0, f(S,A,H,T,\delta) \leq C_{f,g} \cdot g(S,A,H,T,\delta)$. We write $f(S,A,H,T,\delta) = \tcO(g(S,A,H,T,\delta))$ if $C_{f,g}$ in the previous definition is poly-logarithmic in $S,A,H,T,1/\delta$.

For $\lambda > 0$ we define $\Exponential(\lambda)$ as an exponential distribution with a parameter $\lambda$. For $k, \theta > 0$ define $\Gamma(k,\theta)$ as a gamma-distribution with a shape parameter $k$ and a rate parameter $\theta$. For set $\Xset$ such that $\vert \Xset \vert < \infty$ define $\Unif(\Xset)$ as a uniform distribution over this set. In particular, $\Unif[N]$ is a uniform distribution over a set $[N]$.
\newpage
%\input{appendix/bayesian_bootstrap}
%\newpage
%!TEX root = ../BayesUCBVI.tex

\section{\texorpdfstring{\BayesUCBVI}{Bayes-UCBVI}  Proofs}
\subsection{Concentration events}

Let $\betastar,  \beta^{\KL}, \beta^{\conc}, \beta^{\Var}: (0,1) \times \N \to \R_{+}$ and $\beta:(0,1) \to \R_{+}$ be some function defined later on in Lemma \ref{lem:proba_master_event}. We define the following favorable events
% : $\cE^\star$ where we control $\Kinf$, $\cE^{\KL}$ where we control $\KL(\hp^{\,t}_h(s,a), p_h(s,a))$,  $\cE^\conc$ the event where we control the deviation of the optimal value-function, $\cE^{\Var}$ the event where we control the variance of value functions, and $\cE^{\conc}$ the event where we control various other quantities. 
\begin{align*}
  \cE^\star(\delta) &\triangleq \Bigg\{\forall t \in \N, \forall h \in [H], \forall (s,a)\in\cS\times\cA: \quad
    \Kinf(\hp_h^t(s,a),p_h \Vstar_{h+1}(s,a), \Vstar_{h+1}) \leq  \frac{\betastar(\delta,n_h^t(s,a))}{n_h^t(s,a)}\Bigg\}\,,\\
\cE^{\KL}(\delta) &\triangleq \Bigg\{ \forall t \in \N, \forall h \in [H], \forall (s,a) \in \cS\times\cA: \quad \KL(\hp^{\,t}_h(s,a), p_h(s,a)) \leq \frac{S \cdot \beta^{\KL}(\delta, n^{\,t}_h(s,a))}{n^{\,t}_h(s,a)} \Bigg\},\\
 %\cE^{\bias} &\triangleq \Bigg\{\forall t \in \N, \forall h \in [H], \forall (s,a)\in\cS\times\cA, \forall f \in \R_+^S, \norminf{f} \leq H:\\
 %&\qquad(\hp_h^t -p_h)f \leq \frac{2}{H} p_h f +  \frac{1.5H^2 + 4H}{n^k_h(s,a)}(S \log H + \beta^\bias(\delta, n^k_h(s,a))) + \left(1 + \frac{1}{H} \right) \frac{1}{n^k_h(s,a)} \Bigg\}\\
\cE^{\conc}(\delta) &\triangleq \Bigg\{\forall t \in \N, \forall h \in [H], \forall (s,a)\in\cS\times\cA: \\
&\qquad|(\hp_h^t -p_h) \Vstar_{h+1}(s,a)| \leq \sqrt{2 \Var_{p_h}(\Vstar_{h+1})(s,a)\frac{\beta(\delta,n_h^t(s,a))}{n_h^t(s,a)}} + 3 H \frac{\beta(\delta,n_h^t(s,a))}{n_h^t(s,a)}\Bigg\},\\
\cE^{\Var}(\delta) &\triangleq \Bigg\{\forall t \in \N: \quad \sum_{\ell=1}^t \sum_{h=1}^H \Var_{p_h}[V^{\pi_\ell}_{h+1}(s^\ell_h, a^\ell_h)] \leq H^2t + \sqrt{2H^5 t \beta^{\Var}(\delta, t)} + 3H^3 \beta^{\Var}(\delta, t)
\Bigg\},\\
%\cE^{\ell_1} &\triangleq \Bigg\{\forall t \in [T], \forall h \in [H], \forall (s,a) \in \cS \times \cA : \quad\norm{\hp^{\,t}_h(s,a) - p_h(s,a)}_1 \leq \sqrt{\frac{2S \cdot \beta^{\ell_1}(\delta, n^{\,t}_h(s,a))}{n^{\,t}_h(s,a)}}
%\Bigg\},\\
\cE(\delta) &\triangleq \Bigg\{ \sum_{t=1}^T \sum_{h=1}^H \left| p_h[\uV^{t-1}_{h+1} - V^{\pi_t}_{h+1}](s^t_h, a^t_h) - [\uV^{t-1}_{h+1} - V^{\pi_t}_{h+1}](s^t_{h+1}) \right| \leq 2\ur H\sqrt{2HT \beta(\delta)},\\
    &\qquad \sum_{t=1}^T \sum_{h=1}^H (1-1/H)^{H-h+1}\left| p_h[\uV^{t-1}_{h+1} - V^{\pi_t}_{h+1}](s^t_h, a^t_h) - [\uV^{t-1}_{h+1} - V^{\pi_t}_{h+1}](s^t_{h+1})\right|\leq 2\rme \ur H\sqrt{2HT \beta(\delta)},
\Bigg\}.
\end{align*}
We also introduce the intersection of these events, $\cG(\delta) \triangleq \cE^\star(\delta) \cap \cE^{\KL}(\delta) \cap \cE^{\conc}(\delta) \cap \cE^{\Var}(\delta) \cap \cE(\delta)$. We  prove that for the right choice of the functions $\betastar,  \beta^{\KL}, \beta^{\conc}, \beta, \beta^{\Var}$ the above events hold with high probability.
\begin{lemma}
\label{lem:proba_master_event}
For any $\delta \in (0,1)$ and for the following choices of functions $\beta,$
\begin{align*}
    \betastar(\delta,n) &\triangleq \log(5SAH/\delta) + 3\log\left(\rme\pi(2n+1)\right)\,,\\
    \beta^{\KL}(\delta, n) & \triangleq \log(5SAH/\delta) + \log\left(\rme(1+n) \right), \\
    %\beta^\bias(\delta,n) &\triangleq   \log(6SAH/\delta) + \log(2\log(n) n(n+1)),\\
    \beta^{\conc}(\delta, n) &\triangleq \log(5SAH/\delta) + \log(4\rme(2n+1)) ,\\
    \beta(\delta) &\triangleq \log\left(20/\delta\right),\\
    \beta^{\Var}(\delta, t) &\triangleq \log(20\rme(2t+1)/\delta),\\
    %\beta^{\ell_1}(\delta, n) &\triangleq \log(12SATH/\delta)
\end{align*}
it holds that
\begin{align*}
\P[\cE^\star(\delta)]&\geq 1-\delta/5, \qquad \P[\cE^{\KL}(\delta)]\geq 1-\delta/5,  \qquad \P[\cE^\conc(\delta)] \geq 1-\delta/5,\\
\P[\cE^\Var(\delta)]&\geq 1-\delta/5,  \qquad \P[\cE(\delta)]\geq 1-\delta/5.
\end{align*}
In particular, $\P[\cG(\delta)] \geq 1-\delta$.
\end{lemma}
\begin{proof}
It follows from Theorem \ref{th:max_ineq_kinf} that $\P[\cE^\star(\delta)]\geq 1-\delta/5$. Applying Theorem~\ref{th:max_ineq_categorical} and the union bound over $h \in [H], (s,a) \in \cS \times \cA$ we get $\P[\cE^{\KL}(\delta)]\geq 1-\delta/5$. Next, Theorem~\ref{th:bernstein} and the union bound over $h \in [H], (s,a) \in \cS \times \cA$ yield $\P[\cE^{\conc}(\delta)]\geq 1 - \delta/5$. By Lemma~\ref{lem:event_var},  $\P[\cE^{\Var}(\delta)]\geq 1 - \delta/5$. It remains to estimate $\P[\cE(\delta)]$. 

% Event $\cE$ holds with probability at least $1 - \delta/5$ by Azuma-Hoeffding inequality.
Define the following sequences
\begin{align*}
\bar Z_{t,h} &\triangleq \uV_{h+1}^{t-1}(s_{h+1}^t)-V^*_{h+1}(s_{h+1}^t)-p_h [\uV_{h+1}^{t-1}- V^*_{t+1}](s^t_h,a^t_h), & t\in [T], h \in [H],\\
\tilde Z_{t,h} &\triangleq (1-1/H)^{H-h+1}\left(\uV_{h+1}^{t-1}(s_{h+1}^t)-V^*_{h+1}(s_{h+1}^t)-p_h [\uV_{h+1}^{t-1}- V^*_{h+1}](s^t_h,a^t_h)\right),& t\in [T], h \in [H],
\end{align*}
It is easy to see that these sequences form a martingale-difference w.r.t filtration $\cF_{t,h} = \sigma\left\{ \{ (s^{\ell}_{h'}, a^{\ell}_{h'}), \ell < t, h' \in [H] \} \cup \{ (s^{t}_{h'}, a^t_{h'}), h' \leq h \} \right\}$. Moreover,  \(|\bar Z_{t,h}|\leq 2\ur H, |\tilde Z_{t,h}\vert \leq 2\rme \ur H\) for all \(t\in [T]\) and \(h\in [H].\)   Hence, the Azuma-Hoeffding inequality implies
 \begin{align*}
    \P\Bigl(\Bigl|\sum_{t=1}^T \sum_{h=1}^H \bar Z_{t,h}\Bigr|> 2\ur H\sqrt{2 t H \cdot \beta(\delta)}\Bigr)&\leq 2\exp(-\beta(\delta))=\delta/10, \\
    \P\Bigl(\Bigl|\sum_{t=1}^T \sum_{h=1}^H \bar Z_{t,h}\Bigr|> 2\rme \ur H\sqrt{2 t H \cdot \beta(\delta)}\Bigr)&\leq 2\exp(-\beta(\delta))=\delta/10,
\end{align*}
By the union bound $\P[\cE(\delta)] \geq 1 - \delta/5$.
% Further note that
% \begin{eqnarray*}
% (\hp_h^t -p_h) \Vstar_{h+1}(s,a)=\frac{1}{n_h^t(s,a)}\sum_{i=1}^{t} \ind{\big\{(s_h^i,a_h^i) = (s,a)\big\}}(V^*_{h+1}(s_{h+1}^i)-p_h V^*_{h+1}(s^i_h,a^i_h)).
% \end{eqnarray*}
% Another application of the Azuma-Hoeffding inequality gives for any fixed \(t\in [T],\) \(h\in [H],\) \( (s,a)\in\cS\times\cA, \)
% \begin{eqnarray*}
% \P\left(|(\hp_h^t -p_h) \Vstar_{h+1}(s,a)|> 2H \sqrt{\frac{2\beta_2(\delta)}{n^{\,t}_h(s,a)}}\right)\leq 2\exp(-\beta_2(\delta))=\frac{\delta}{6SAHT}.
% \end{eqnarray*}

\end{proof}

\begin{lemma}\label{lem:event_var}
    Under conditions of Lemma \ref{lem:proba_master_event}, for any $\delta \in (0,1)$, $\P[\cE^{\Var}(\delta)] \geq 1 - \delta/5$.
\end{lemma}
\begin{proof}
    For any $\ell \in \N$ define $\cF_\ell = \sigma \{(s^j_h, a_h^j), j \le \ell, h \in [H]\}$ and let
    \[
        Y_\ell  = \sum_{h=1}^H \Var_{p_h}[V^{\pi_\ell}_{h+1}](s^\ell_h, a^\ell_h) - \Vvar^{\pi_\ell}_1(s^\ell_1),
    \]
    where operator $\Vvar$ is defined in Section~\ref{app:technical}. It is straightforward to check that $(Y_\ell, \cF_\ell)_{\ell \in \N}$ is a martingale-difference sequence. Applying Bernstein inequality (Theorem~\ref{th:bernstein}) we get that with probability at least $1-\delta/5$ for any $t \in \N$ 
    \[        \left\vert \sum_{\ell=1}^t Y_\ell \right\vert \leq \sqrt{2\sum_{\ell=1}^t  \E[Y_\ell^2 | \F_{\ell-1}] \log(20e(2t+1)/\delta)} + 3H^3 \log(20e(2t+1)/\delta).
    \]
    Next we estimate $\E[Y_\ell^2 | \F_{\ell-1}]$ in the following way
    \[
        \E[Y_\ell^2 | \F_{\ell-1}] \leq\E\left[\left(\sum_{h=1}^H \Var_{p_h}[V^{\pi_\ell}_{h+1}](s^{\ell}_h, a^{\ell}_h) \right)^2 \biggl| \F_{\ell-1} \right]  \leq  H^3 \E_{\pi_\ell}\left[ \sum_{h=1}^H \Var_{p_h}[V^{\pi_\ell}_{h+1}](s_h, a_h) \right].
    \]
    By Lemma~\ref{lem:law_of_total_variance}
    \[
        \E_{\pi_\ell}\left[  \sum_{h=1}^H \Var_{p_h}[V^{\pi_\ell}_{h+1}](s^\ell_h, a^\ell_h) \right] = \E_{\pi_\ell}\left[ \left(\sum_{h=1}^H r_{h}( s_{h},a_{h}) - V^{\pi_\ell}_1(s_1)\right)^2\right] \leq \E_{\pi_\ell}\left[ \left(\sum_{h=1}^H r_{h}( s_{h},a_{h})\right)^2\right] \leq H^2,
    \]
    
    Since $\beta^{\Var}(\delta, t) = \log(20e(2t+1)/\delta)$, we have
    \[
        \sum_{\ell=1}^t Y_\ell \leq \sqrt{2H^5 t \beta^{\Var}(\delta, t)} + 3H^3 \beta^{\Var}(\delta, t).
    \] 
    Finally, by Lemma~\ref{lem:law_of_total_variance}
    \[
        \sum_{\ell=1}^t \sum_{h=1}^H \Var_{p_h}[V^{\pi_\ell}_{h+1}](s^\ell_h, a^\ell_h) = \sum_{\ell=1}^t Y_\ell + \sum_{\ell=1}^t \Vvar^{\pi_\ell}_1(s_1^\ell) \leq \sqrt{2H^5 t \beta^{\Var}(\delta, t)} + 3H^3 \beta^{\Var}(\delta, t) + H^2 t.
    \]
\end{proof}

\begin{lemma}\label{lem:f_and_l1_concentration}
      Assume conditions of Lemma \ref{lem:proba_master_event}. Then conditioned on event $\cE^{\KL}(\delta)$, for any $f \colon \cS \to [0, \ur H]$, $t \in \N, h \in [H], (s,a) \in \cS \times \cA$,
      \begin{align*}
            (\hp_h^t -p_h)f(s,a) &\leq \frac{1}{H} p_h f(s,a) + \frac{5 \ur H^2 S \cdot \beta^{\KL}(\delta, n^{\,t}_h(s,a))}{n^{\,t}_h(s,a)}, \\
            \norm{\hp^{\,t}_h(s,a) - p_h(s,a)}_1 &\leq \sqrt{\frac{2S \cdot \beta^{\KL}(\delta, n^{\,t}_h(s,a))}{n^{\,t}_h(s,a)}}.
      \end{align*}
\end{lemma}
\begin{proof}
    We apply Lemma~\ref{lem:Bernstein_via_kl} and Lemma~\ref{lem:switch_variance_bis} to obtain
    \begin{align*}
        (\hp_h^t -p_h)f(s,a) &\leq \sqrt{2\Var_{\hp_h^t}[f](s,a) \cdot \KL(\hp_h^t, p_h) } + \frac{2H\ur}{3} \KL(\hp_h^t, p_h) \\
        &\leq 2\sqrt{\Var_{p_h}[f](s,a) \cdot \KL(\hp_h^t, p_h) } + \left( 2\sqrt{2} + \frac{2}{3} \right)H\ur\KL(\hp_h^t, p_h).
    \end{align*}
    Since $0 \leq f(s) \leq \ur H$ we get
    \[
        \Var_{p_h}[f](s,a) \leq p_h[f^2](s,a) \leq \ur H \cdot p_h f(s,a).
    \]
    Finally, applying $2\sqrt{ab} \leq a+b, a, b \geq 0$, we obtain the following inequality
    \[
        (\hp_h^t -p_h)f(s,a) \leq \frac{1}{H} p_h f(s,a) + (H^2 + 2\sqrt{2} \ur H + 2\ur H/3) \KL(\hp_h^t, p_h) \leq \frac{1}{H} p_h f(s,a) + 5 \ur H^2 \KL(\hp_h^t, p_h).
    \]
    Definition of $\cE^{\KL}(\delta)$ implies the first statement. The second statement follows directly from the combination of Pinsker's inequality and definition of $\cE^{\KL}(\delta)$.
\end{proof}

%!TEX root = ../BayesUCBVI.tex
\subsection{Optimism}
\label{app:optimism}

In this section we prove that conditioned on the event $\cE^\star(\delta)$ our estimate of $Q$-function $\uQ^{\,t}_h(s,a)$ is optimistic that is $\uQ^{\,t}_h(s,a) \geq \Qstar_h(s,a)$ for any $t \le T, h \in [H], (s,a) \in \cS \times \cA$.

For any $\beta > 0, p \in \simplex_{S'-1}$ and $f: \cS' \to \R$ define 
\[
    U^{\Kinf}(\beta, p,f) = \sup \{\mu \geq pf : \Kinf(p, \mu, f) \leq \beta\}.
\]
First we are going to prove that $U^{\Kinf}(\beta^\star(\delta, n^{\,t}_h(s,a)) / \upn^{\,t}_h(s,a), \up^{\,t}_h(s,a), \Vstar_{h+1})$ defines an upper confidence bound for $p_h \Vstar_{h+1}(s,a)$.
\begin{lemma}\label{lem:u_kinf_is_ucb}
    Conditioned on the event $\cE^\star(\delta)$, for any $t \in \N, h \in [H], (s,a) \in \cS \times \cA,$
    \[
        p_h \Vstar_{h+1}(s,a) \leq U^{\Kinf}\left(\frac{\beta^\star(\delta, n^{\,t}_h(s,a))}{\upn^{\,t}_h(s,a)}, \up^{\,t}_h(s,a), \Vstar_{h+1}\right),
    \]
    where event $\cE^\star(\delta)$ and function $\beta^\star(\delta,n)$ were defined in Lemma~\ref{lem:proba_master_event}.
\end{lemma}
\begin{proof}
    By Lemma~\ref{lem:var_form_Kinf} we have for any $\up^{\,t}_h V^\star_h(s,a) < u < \ur(H-h)$
    \begin{align*}
        \upn^{\,t}_h(s,a) \Kinf(\up^{\,t}_h(s,a), u, \Vstar_{h+1}) &= \upn^t_h(s,a) \max_{\lambda \in [0,1]} \E_{s' \sim \up^{\,t}_h( s,a)}\left[ \log\left(1 - \lambda \frac{\Vstar_{h+1}(s') - u}{\ur(H - h) - u} \right)\right] \\
        &\leq \max_{\lambda \in [0,1]} n_0 \log\left(1 - \lambda \right) + (\upn^{\,t}_h(s,a) - n_0) \max_{\lambda \in [0,1] }\E_{s' \sim \hp^{\,t}_h(s,a)}\left[ \log\left(1 - \lambda \frac{\Vstar_{h+1}(s') - u}{\ur(H - h) - u} \right) \right] \\
        &\leq (\upn^{\,t}_h(s,a) - n_0) \max_{\lambda \in [0,1] }\E_{s' \sim \hp^{\,t}_h(s,a)}\left[ \log\left(1 - \lambda \frac{\Vstar_{h+1}(s') - u}{H - h - u} \right) \right] \\
        &= (\upn^{\,t}_h(s,a) - n_0) \Kinf(\hp^{\,t}_h(s,a), u, \Vstar_{h+1}) = n^{\,t}_h(s,a) \Kinf(\hp^{\,t}_h(s,a), u, \Vstar_{h+1}).
    \end{align*}

    By the definition of event $\cE^\star(\delta)$ we have for any $t \in \N, h \in [H], (s,a) \in \cS \times \cA$,
    \[
        n^{\,t}_h(s,a) \Kinf(\hp^{\,t}_h(s,a), p_h \Vstar_{h+1}(s,a), \Vstar_{h+1}) \leq \beta^\star(\delta, n^{\,t}_h(s,a)),
    \]
    hence $\upn^{\,t}_h(s,a) \Kinf(\up^{\,t}_h(s,a), p_h \Vstar_{h+1}(s,a), \Vstar_{h+1}) \leq \beta^\star(\delta, n^{\,t}_h(s,a))$. Therefore a value $p_h \Vstar_{h+1}(s,a)$ is feasible for optimization problem in the definition of $U^{\Kinf}$. 
\end{proof}

For the further results we have to guarantee that a number of observations of the fake state $s_0$ is large enough to apply anti-concentration result of Dirichlet distribution. Define constant 
\begin{equation}\label{eq:constant_c_n0}
    c_{n_0} = \frac{1}{(\sqrt{2\pi} - 1)^2} \cdot \left(\frac{2 \sqrt{2}}{\sqrt{\log(17/16)}} + \frac{98 \sqrt{6}}{9} \right)^2 +  \frac{\log(10\pi)}{\log(17/16) }.
\end{equation}

\begin{lemma}\label{lem:first_quantile_bound}
    Let \( n_0 \geq c_{n_0} +  \log_{17/16}(T)\), where   $c_{n_0}$ is defined in \eqref{eq:constant_c_n0}, and $\ur \geq 2$, and assume conditions of Lemma~\ref{lem:proba_master_event}. Then on the event $\cE^\star(\delta)$, it holds for any $t \in \N, h \in [H], (s,a) \in \cS \times \cA,$ 
    \begin{align*}
        p_h \Vstar_{h+1}(s,a) &\leq \Q_{p \sim \rho^{\,t}_h(s,a)}(p \Vstar_{h+1}(s,a), \kappa^{\,t}_h(s,a)), 
    \end{align*}
    where \(\kappa^{\,t}_h(s,a) = 1 - \frac{C_\kappa \delta}{SAH [2 n^{\,t}_h(s,a) + 1]^{3} [\upn^{\,t}_h(s,a)]^{3/2}}\) with an absolute constant $C_\kappa = 1/(5 \cdot (\rme \pi)^3)$.
\end{lemma}
\begin{proof} To simplify notations we set $\upn = \upn^{\,t}_h(s,a)$ and $n = n^{\,t}_h(s,a)$. Note that $\rho^{\,t}_h(s,a)$ is a Dirichlet distribution $\Dir(\{\upn^{\,t}_h(s'|s,a)\}_{s' \in \cS'})$. Since \( \upn^{\,t}_h(s_0|s,a) = n_0 \geq c_{n_0} +\log_{17/16}(T)\), we may apply Theorem~\ref{thm:lower_bound_dbc} if $\upn \geq 2n_0 $: for any $\up^{\,t}_h \Vstar_{h+1}(s,a) \leq u < \ur(H - h)$
\begin{equation}
\label{eq: optimizm lem1}
    \P_{p \sim \rho^{\,t}_h(s,a)}\left( p \Vstar_{h+1} \geq u \right) \geq \exp(-\upn \Kinf(\up^{\,t}_h(s,a), u, \Vstar_{h+1})- 3/2\log \upn).
\end{equation}
Notice that the same inequality also holds for $u < \up^{\,t}_h \Vstar_{h+1}(s,a)$ because $\Kinf(\up^{\,t}_h(s,a), u, \Vstar_{h+1})= 0$ and 
\[
    \P_{p \sim \rho^{\,t}_h(s,a)}\left( p \Vstar_{h+1} \geq u \right)  \geq \P_{p \sim \rho^{\,t}_h(s,a)}\left( p \Vstar_{h+1} \geq \up^{\,t}_h \Vstar_{h+1}(s,a) \right).
\]

    Let $u' = U^{\Kinf}(\beta^\star(\delta, n) / \upn, \up^{\,t}_h(s,a), \Vstar_{h+1})$. Fix arbitrary $\varepsilon >0$ and set $u = u' - \epsilon$. This choice implies that $\upn \Kinf(\up^{\,t}_h(s,a), u, \Vstar_{h+1}) \leq \beta^\star(\delta, n)$, and together with \eqref{eq: optimizm lem1} yields 
    \[
     \P_{p \sim \rho^{\,t}_h(s,a)}\left( p \Vstar_{h+1} \geq u \right) \geq \exp(-\beta^\star(\delta,n) - 3/2\log(\upn)) \geq \frac{C_\kappa \delta}{SAH [2 n^{\,t}_h(s,a) + 1]^{3} [\upn^{\,t}_h(s,a)]^{3/2}}.
    \]
    By Lemma~\ref{lm:quantile_bound_tail} and definition of $\kappa^{\,t}_h(s,a)$,  $\Q_{p \sim \rho^{\,t}_h(s,a)}(p \Vstar_{h+1}(s,a), \kappa^{\,t}_h(s,a)) \geq u' - \epsilon$. Allowing $\epsilon \to 0$ we have $\Q_{p \sim \rho^{\,t}_h(s,a)}(p \Vstar_{h+1}(s,a), \kappa^{\,t}_h(s,a)) \geq u'$. It remains to apply Lemma~\ref{lem:u_kinf_is_ucb} to conclude the statement in the case $\upn \geq 2n_0$.
    
    To handle the case $\upn < 2n_0$ we remark that
    \[
        \P_{p \sim \rho^{\,t}_h(s,a)}\left( p \Vstar_{h+1} \geq p_h V^\star_{h+1}(s,a) \right) \geq \P_{\xi \sim B(n_0, \upn - n_0)}\left( \ur(H-h) \xi \geq H-h \right) \geq \P_{\xi \sim B(n_0, \upn - n_0)}\left( \xi \geq \frac{1}{2} \right),
    \]
    where we used an upper bound $p_h V^\star_{h+1}(s,a) \leq H-h$ and a lower bound $\Vstar_{h+1}(s) \geq 0$ for $s \in \cS$ and $\Vstar_{h+1}(s_0) = \ur(H-h)$. By the result of \citet{groeneveld1977mode} we have that for $n_0 \leq \upn - n_0$ we have that the median $m$ of $\xi$ is greater than the mode $(n_0-1)/(\upn - 2)$. Since $2n_0 > \upn$, we have that $m \geq 1/2$, thus
    \begin{align*}
          \P_{p \sim \rho^{\,t}_h(s,a)}\left( p \Vstar_{h+1} \geq p_h V^\star_{h+1}(s,a) \right) &\geq \P_{\xi \sim B(n_0, \upn - n_0)}\left( \xi \geq \frac{1}{2} \right) \geq \P_{\xi \sim B(n_0, \upn - n_0)}\left( \xi \geq m \right) = \frac{1}{2} \\
          &\geq \frac{C_\kappa \delta}{SAH [2 n^{\,t}_h(s,a) + 1]^{3} [\upn^{\,t}_h(s,a)]^{3/2}}.
    \end{align*}
    Lemma~\ref{lm:quantile_bound_tail} concludes the statement.
\end{proof}

\begin{proposition}[Optimism]\label{prop:optimism}
    Let \( n_0 = \lceil c_{n_0} +  \log_{17/16}(T)\rceil\), where $c_{n_0}$ is an absolute constant defined in \eqref{eq:constant_c_n0}.  Furthermore, let $\ur = 2$ and assume that conditions of Lemma~\ref{lem:proba_master_event} are satisfied. Then $\uQ^{\,t}_h(s,a) \geq \Qstar_h(s,a)$ on the event $\cE^\star(\delta)$ for any $t \leq T, h \in [H]$ and $(s,a) \in \cS \times \cA.$ 
\end{proposition}
\begin{proof}
    We proceed using backward induction over $h$. For $h = H+1$,  $\uQ^{\,t}_h(s,a) = \Qstar_h(s,a) = 0$. Let $h \le H$. Note that
    \begin{align}
    \label{eq: optimizm lem2_1}
        \uQ^{\,t}_h(s,a) - \Qstar_h(s,a) = \Q_{p \sim \rho^{\,t}_h(s,a)}(p \uV^t_{h+1}(s,a), \kappa^{\,t}_h(s,a)) - p_h \Vstar_{h+1}(s,a).
    \end{align}
    Induction hypothesis implies that
    \begin{equation*}
         \Vstar_{h+1}(s) = \Qstar_{h+1}(s, \pistar(s)) \leq \uQ^t_{h+1}(s, \pistar(s))  \leq \uV^t_{h+1}(s),
    \end{equation*}
    and hence 
    \begin{equation}\label{eq: optimizm lem2_2}
        \Q_{p \sim \rho^{\,t}_h(s,a)}(p \uV^t_{h+1}(s,a), \kappa^{\,t}_h(s,a)) \geq \Q_{p \sim \rho^{\,t}_h(s,a)}(p \Vstar_{h+1}(s,a), \kappa^{\,t}_h(s,a)).
    \end{equation}
    % Next, we apply Lemma~\ref{lem:ucb_kinf_via_quantile}, using a choice of $\kappa^{\,t}_h(s,a) = 1 - \exp(- \beta^*(\delta, n^{\,t}_h(s,a)) - 1/2\log n^{\,t}_h(s,a) - r_{n^{\,t}_h(s,a)})$.
    Equation \eqref{eq: optimizm lem2_1}, inequality \eqref{eq: optimizm lem2_2} and Lemma~\ref{lem:first_quantile_bound} imply the statement.
%     Lemma~\ref{lem:quantile_bounds}, \eqref{eq: optimizm lem2_1}
%      \begin{equation}
%      \label{eq: optimizm lem2_2}
%         \Q_{p \sim \rho^{\,t}_h(s,a)}(p \Vstar_{h+1}(s,a), \kappa^{\,t}_h(s,a)) \geq U^{\Kinf}_{n^{\,t}_h(s,a) - n_0}\left( \beta^*(\delta, n^{\,t}_h(s,a)), \hp^{\,t}_h(\cdot | s,a), \Vstar_h\right) \geq p_h \Vstar_h(s,a).
%   \end{equation}
%     The last inequality is due to the fact that $p_h V^*_h(s,a)$ lies in the feasible set of the optimization problem in the definition of $U_{\textcolor{red}{n}}^{\Kinf}$. Equation \eqref{eq: optimizm lem2_1} and inequality \eqref{eq: optimizm lem2_2} imply the statement.
\end{proof}

Next we formulate key inequality for the further proof of regret bound.

\begin{corollary}\label{cor:quantile_bounds}
     Let \( n_0 = \lceil c_{n_0} +  \log_{17/16}(T)\rceil\) and $\ur = 2$. Under conditions of Lemma~\ref{lem:proba_master_event}, it holds  on the event $\cE^\star(\delta)$ for any $t \in \N, h \in [H], (s,a) \in \cS \times \cA$,
    \begin{align*}
        p_h \Vstar_{h+1}(s,a) &\leq \Q_{p \sim \rho^{\,t}_h(s,a)}(p \uV^t_{h+1}(s,a), \kappa^{\,t}_h(s,a))\\
        &\leq \up^{\,t}_h \uV^t_{h+1}(s,a) +  2 \sqrt{ \frac{ \Var_{\up^{\,t}_h}[\uV^t_{h+1}](s,a) \log\left( \frac{1}{1 - \kappa^{\,t}_h(s,a)}\right)}{\upn^{\,t}_h(s,a)}} + \frac{2\ur H\sqrt 2 \log\left( \frac{1}{1 - \kappa^{\,t}_h(s,a)}\right)}{\upn^{\,t}_h(s,a)},
    \end{align*}
    where \(\kappa^{\,t}_h(s,a) = 1 - \frac{C_\kappa \delta}{SAH [2 n^{\,t}_h(s,a) + 1]^{3} [\upn^{\,t}_h(s,a)]^{3/2}}\) with an absolute constant $C_\kappa = 1/(5 \cdot (\rme \pi)^3)$ and $c_{n_0}$  defined in \eqref{eq:constant_c_n0}.
\end{corollary}
\begin{proof}
    The first inequality immediately follows from Proposition~\ref{prop:optimism}. The second inequality follows from Lemma~\ref{lem:bernstein_dirichlet}, where we take $\delta = 1-\kappa^{\,t}_h(s,a)$, $f = \uV^t_{h+1}$, and Lemma~\ref{lm:quantile_bound_tail}.
\end{proof}

\subsection{Proof of Theorem~\ref{th:regret_bound_bayesUCBVI}}
\label{app:proof-BayesUCBVI}

Denote $\delta^t_h = \uV^{\,t-1}_h(s^t_h) - V^{\pi_t}_h(s^t_h)$ and surrogate regret $\uregret^{\,t}_h = \sum_{t=1}^T \delta^t_h$. To simplify notations denote $N^{\,t}_h = n^{\,t-1}_h(s^t_h, a^t_h)$, $N^{\,t}_h(s) =n^{\,t-1}_h(s | s^t_h, a^t_h), \upN^{\,t}_h = \upn^{\,t-1}_h(s^t_h, a^t_h), \upN^{\,t}_h(s) = \upn^{\,t-1}_h(s | s^t_h, a^t_h)$, and $\hat\kappa^t_h = \kappa^{t-1}_h(s^t_h, a^t_h)$. Let
\begin{equation}
\label{eq: L delta def}
    L = \max\left\{ n_0, \log(TH), \max_{t\in [T],h \in [H]}\log\left(\frac{1}{1- \hat\kappa^t_h}\right), \betastar(\delta, T),  \beta^{\KL}(\delta, T), \beta^{\conc}(\delta, T), \beta(\delta), \beta^{\Var}(\delta, T), 1\right\}.
\end{equation}
Under conditions of Proposition~\ref{prop:optimism} and Lemma \ref{lem:proba_master_event}, $L = \cO(\log(SATH/\delta)) = \tcO(1)$. In what follows we will follow ideas of \UCBVI with the Bernstein bonuses, see \citet{azar2017minimax}.

\begin{lemma}\label{lem:surrogate_regret_bound}
    Assume conditions of Theorem \ref{th:regret_bound_bayesUCBVI}. Then it holds on the event $\cG(\delta)$, for any $h \in [H],$
    \[
        \uregret^T_{h} \leq U^T_{h} \triangleq A^T_{h} + B^T_{h} + C^T_{h} + 4\rme H\sqrt{2 H T L} + 2\rme SAH^2,
    \]
    where 
    \begin{align*}
        A^T_{h} &= 2 \rme \sqrt{L} \sum_{t=1}^T \sum_{h'=h}^H \sqrt{  \Var_{\up^{t-1}_{h'}}[\uV^{t-1}_{h'+1}](s^t_{h'},a^t_{h'}) \frac{\ind\{N^t_{h'} > 0\}}{N^t_{h'}}}, \\
        B^T_{h} &= \rme \sqrt{2L} \sum_{t=1}^T \sum_{h'=h}^H \sqrt{\Var_{p_{h'}}[\Vstar_{h'+1}](s^t_{h'},a^t_{h'})\frac{\ind\{N^t_{h'} > 0\}}{N_{h'}^t}}, \\
        C^T_{h} &= 21 \rme H^2 S \cdot L \cdot \sum_{t=1}^T \sum_{h=h'}^H \frac{\ind\{N^t_{h'} > 0\}}{N^t_{h'}},
    \end{align*}
    and $L$ is defined in \eqref{eq: L delta def}.
\end{lemma}
\begin{proof}
    By the greedy choice of action, formula \eqref{eq:optimistic_planning} for $\uQ$ and Bellman's equations
    \begin{align*}
        \delta^t_h &= r_h(s^t_h, a^t_h) +  \Q_{p \sim \rho^{\,t-1}_h(s^t_h, a^t_h)}(p \uV^{t-1}_{h+1}(s^t_h,a^t_h), \kappa^{t-1}_h(s^t_h, a^t_h)) - r_h(s^t_h, a^t_h) - p_h V^{\pi_t}_{h+1}(s^t_h, a^t_h) \\
        &= \Q_{p \sim \rho^{\,t-1}_h(s^t_h, a^t_h)}(p \uV^{t-1}_{h+1}(s,a), \kappa^{t-1}_h(s^t_h, a^t_h)) - (\up^{\,t-1}_h -  \up^{\,t-1}_h) \uV^{t-1}_{h+1}(s^t_h, a^t_h) \\
        &- (\hp^{\,t-1}_h - \hp^{\,t-1}_h) \uV^{t-1}_{h+1}(s^t_h, a^t_h) - p_h \uV^{t-1}_{h+1}(s^t_h, a^t_h) + p_h [\uV^{t-1}_{h+1} - V^{\pi_t}_{h+1}](s^t_h, a^t_h) \\
        &= \underbrace{\Q_{p \sim \rho^{t-1}_h(s^t_h, a^t_h)}( p \uV^{t-1}_{h+1}(s^t_h,a^t_h), \kappa^{t-1}_h(s^t_h,a^t_h)) - \up^{\,t-1}_h \uV^{t-1}_{h+1}(s^t_h,a^t_h)}_{\termA} + \underbrace{[\up^{\,t-1}_h - \hp^{\,t-1}_h] \uV^{t-1}_{h+1}(s^t_h, a^t_h)}_{\termB} \\
        &+  \underbrace{(\hp^{\,t-1}_h - p_h)  [\uV^{t-1}_{h+1} - \Vstar_{h+1}](s^t_h, a^t_h)}_{\termC} + \underbrace{(\hp^{\,t-1}_h - p_h) \Vstar_{h+1}(s^t_h, a^t_h)}_{\termD} \\
        &+ \underbrace{p_h[\uV^{t-1}_{h+1} - V^{\pi_t}_{h+1}](s^t_h, a^t_h) - [\uV^{t-1}_{h+1} - V^{\pi_t}_{h+1}](s^t_{h+1})}_{\xi^{\,t}_h} + \delta^t_{h+1}.
    \end{align*}
    % It is easy to check that $(\xi^{\,t}_h, )_{t \in \N}$ is a martingale-difference sequence for ... 
    % \todoAlex{Add $\sigma$-algebra}
    It is easy to see that $\xi^{\,t}_h$ appears in the definition of event $\cE(\delta) \subseteq \cG(\delta)$.
    
    We analyse each term in this representation under assumption $N^t_h > 0$.
    
    \textbf{Term $\termA$.} Then to estimate this term we apply the second inequality in Corollary~\ref{cor:quantile_bounds}. We obtain
    \begin{align*}
        \Q_{p \sim \rho^{t-1}_h(s^t_h, a^t_h)}( p \uV^{t-1}_{h+1}(s^t_h,a^t_h), \hat\kappa^t_h) - \up^{\,t-1}_h \uV^{t-1}_{h+1}(s^t_h,a^t_h) &\leq 2 \sqrt{ \frac{ \Var_{\up^{\,t-1}_h}[\uV^{t-1}_{h+1}](s^t_h,a^t_h) \log\left( \frac{1}{1 - \hat\kappa^t_h}\right)}{\upN^{\,t}_h}} \\
        &+ \frac{2 \ur \sqrt 2 H \log\left( \frac{1}{1 - \hat\kappa^t_h}\right)}{\upN^{\,t}_h}.
    \end{align*}
    Note that this term acts very similar to Bernstein-type bonuses in \UCBVI  algorithm.

    \textbf{Term $\termB$.} The bound follows directly from the definition of $\up^{\,t}_h$ and $\hp^{\,t}_h$. Indeed, 
    \[
        [\up^{\,t-1}_h - \hp^{\,t-1}_h] \uV^{t-1}_{h+1}(s^t_h, a^t_h) = \frac{n_0}{\upN^{\,t}_h} \cdot (\ur H)  + \sum_{s' \in \cS} \left( \frac{N^{\,t}_h(s')}{\upN^{\,t}_h} - \frac{N^{\,t}_h(s')}{N^{\,t}_h} \right) \cdot \uV^t_{h+1}(s') \leq \frac{\ur H L}{\upN^{\,t}_h}.
    \]
 
    \textbf{Term $\termC$.} To estimate this term we first note that by Proposition~\ref{prop:optimism}, $\uV^{t-1}_{h+1}(s) - \Vstar_{h+1}(s) \geq 0$ for any $s \in \cS$. Hence, we may use Lemma~\ref{lem:f_and_l1_concentration} with $f = \uV^t_{h+1} - \Vstar_{h+1}$. We obtain
    \begin{align*}
        (\hp^{\,t-1}_h - p_h)  [\uV^{t-1}_{h+1} - \Vstar_{h+1}](s^t_h, a^t_h) &\leq \frac{1}{H} p_h [\uV^{t-1}_{h+1} - \Vstar_{h+1}](s^t_h, a^t_h) +  \frac{5 \ur H^2 S \cdot \beta^{\KL}(\delta, N^{\,t}_h)}{N^{\,t}_h} \\
        &\leq \frac{1}{H} (\xi^{\,t}_h + \delta^t_h) + \frac{5\ur H^2S \cdot L}{N^{\,t}_h}.
    \end{align*}
    
    \textbf{Term $\termD$.}
    By the definition of event $\cE^{\conc}(\delta) \subseteq \cG(\delta)$
    \begin{align*}
        (\hp^{\,t-1}_h - p_h) \Vstar_{h+1}(s^t_h, a^t_h) \leq \sqrt{2 \Var_{p_h}[\Vstar_{h+1}](s^t_h,a^t_h)\frac{\beta^{\conc}(\delta,N_h^t)}{N_h^t}} + 3 H \frac{\beta^{\conc}(\delta,N_h^t)}{N_h^t}.
    \end{align*}
    Collecting bounds for the terms $\termA$--$\termD$ we get
    \begin{align*}
        \delta^t_h &\leq \left(1 + \frac{1}{H} \right) \delta^t_{h+1} + \left(1 + \frac{1}{H} \right) \xi^{\,t}_h + 2\sqrt{  \Var_{\up^{\,t}_h}[\uV^t_{h+1}](s^t_h,a^t_h) \frac{L}{\upN^{\,t}_h}} \\
        &+ \sqrt{2 \Var_{p_h}[\Vstar_{h+1}](s^t_h,a^t_h)\frac{L}{N_h^t}} +\frac{(2\ur \sqrt{2}H + \ur H+ 5\ur H^2S + 3H) L}{N^{\,t}_h}.
    \end{align*}
    Notice that in the case $N^t_h = 0$ we have a trivial bound $\delta^t_h \leq \ur H$. However, this case might appear at most $SAH$ times in the summation and thus we can handle this case by additive $\ur SAH^2$ error term.
    
    Define $\gamma_h = (1 + 1/H)^{H-h+1}$. Notice that $\gamma_h < \rme$, $1/\upN^{\,t}_h < 1/N^{\,t}_h$, $\ur = 2$, $H \leq H^2S$.
    After summation, we have
    \begin{align*}
        \uregret^T_{h} &\leq \sum_{t=1}^T \sum_{h'=h}^H \gamma_{h'} \xi^t_{h'} + \ur H^2 SA \\
        &+ 2 \rme \sqrt{L} \sum_{t=1}^T \sum_{h'=h}^H \sqrt{  \Var_{\up^{\,t-1}_{h'}}[\uV^{t-1}_{h'+1}](s^t_{h'},a^t_{h'}) \frac{\ind\{N^t_{h'} > 0\}}{N^t_{h'}}} & \triangleq A^T_{h} \\
        &+ \rme \sqrt{2L} \sum_{t=1}^T \sum_{h'=h}^H \sqrt{\Var_{p_{h'}}[\Vstar_{h'+1}](s^t_{h'},a^t_{h'})\frac{\ind\{N^t_{h'} > 0\}}{N_{h'}^t}}& \triangleq  B^T_{h} \\
        &+ 21\rme H^2 S \cdot L \cdot \sum_{t=1}^T \sum_{h=h'}^H \frac{\ind\{N^t_{h'} > 0\}}{N^t_{h'}}. & \triangleq C^T_{h}
    \end{align*}
    
    Finally, by definition of the event $\cE(\delta)$ we get
    \[
        \sum_{t=1}^T \sum_{h'=h}^H \gamma_{h'} \xi^t_{h'} \leq  4\rme \cdot H \sqrt{2 H T L}.
    \]
\end{proof}

\begin{lemma}\label{lem:regret_sums}
    For any $H,T \geq 1$,
    \begin{align*}
        \sum_{t=1}^T \sum_{h=1}^H &\frac{\ind\{n^{\,t-1}_h(s^t_h, a^t_h) > 0\}}{n^{\,t-1}_h(s^t_h, a^t_h)} \leq 2 HSAL, \\
        \sum_{t=1}^T \sum_{h=1}^H &\frac{\ind\{n^{\,t-1}_h(s^t_h, a^t_h) > 0\}}{\sqrt{n^{\,t-1}_h(s^t_h, a^t_h))}} \leq 3H \sqrt{TSA}.
    \end{align*}
\end{lemma}
\begin{proof}
    The main observation for both inequalities follows from pigeon-hole principle: term corresponding to each state-action pair $(s,a)$ appears in the sum exactly $n^{\,t-1}_h(s,a)$ times with a value $1/n$ for $n$ increasing from $1$ to $n^{\,t-1}_h(s,a)$.
    
    For the first sum we use a bound on harmonic numbers
    \[
        \sum_{t=1}^T  \frac{\ind\{n^{\,t-1}_h(s^t_h, a^t_h) > 0\}}{n^{\,t-1}_h(s^t_h, a^t_h)} = \sum_{(s,a) \in \cS \times \cA} \sum_{n=1}^{n^{\,t-1}_h(s,a)} \frac{1}{n} \leq SA (\log(T) + 1) \leq 2 SAL.
    \]
    To finish the proof of the first inequality it remains to take a sum w.r.t $h$. For the second sum we use the following integral bound
    \begin{equation}\label{eq:regret_sum_1/sqrt(n)}
        \sum_{t=1}^T  \frac{\ind\{n^{\,t-1}_h(s^t_h, a^t_h) > 0\}}{\sqrt{n^{\,t-1}_h(s^t_h, a^t_h))}}  = \sum_{(s,a) \in \cS} \sum_{n=1}^{n^{\,t-1}_h(s,a)} \frac{1}{\sqrt{n}} \leq \sum_{(s,a) \in \cS} 2\sqrt{n^{\,t-1}_h(s,a)+1}.
    \end{equation}
    Since $\sum_{s,a} n^{\,t-1}_h(s,a) = t-1$, the last sum is maximized if $n^{\,t-1}_h(s,a) = (t-1)/(SA)$. This implies  the second statement.
\end{proof}

\begin{lemma}\label{lem:sum_variance}
    Assume that conditions of Theorem~\ref{th:regret_bound_bayesUCBVI} are fulfilled. Then it holds on the event $\cG(\delta)$,
    \begin{align*}
        \sum_{t=1}^T \sum_{h=1}^H  \Var_{\up^{\,t-1}_h}[\uV^{t-1}_{h+1}](s^t_h,a^t_h)  \ind\{N^t_h > 0\} &\leq 2 H^2 T + 2 H^2 U_1^T  + 22H^3S^2A L^2 + 32 H^3 S\sqrt{2AT L},\\
        \sum_{t=1}^T \sum_{h=1}^H  \Var_{p_h}[\Vstar_{h+1}](s^t_h,a^t_h) &\leq 2H^2T + 2H^2U_1^T + 6 H^3 L +  8\sqrt{2H^5 T L}.
    \end{align*}
    where $U_h^T$ is defined in Lemma~\ref{lem:surrogate_regret_bound}.
\end{lemma}
\begin{proof}
    We apply the second inequality in Lemma~\ref{lem:switch_variance},
    \begin{align*}
        \sum_{t=1}^T \sum_{h=1}^H  \Var_{\up^{\,t-1}_h}[\uV^{t-1}_{h+1}](s^t_h,a^t_h)  \ind\{N^t_h > 0\}  &\leq  \underbrace{\sum_{t=1}^T \sum_{h=1}^H \Var_{p_h}[\uV^{t-1}_{h+1}](s^t_h,a^t_h)  \ind\{N^t_h > 0\} }_{\termOne} \\
        &+ \underbrace{2 \ur^2 H^2 \sum_{t=1}^T \sum_{h=1}^H \norm{ \up^{\,t-1}_h(s^t_h, a^t_h) -  p_h(s^t_h, a^t_h) }_1  \ind\{N^t_h > 0\} }_{\termTwo}.
    \end{align*}

    To bound the term $\termTwo$ one may use Lemma~\ref{lem:f_and_l1_concentration}. We obtain for $N^t_h > 0$
    \begin{align*}
        \norm{ \up^{\,t-1}_h(s^t_h, a^t_h) -  p_h(s^t_h, a^t_h) }_1 &\leq \norm{\up^{\,t-1}_h(s^t_h, a^t_h) - \hp^{\,t-1}_h(s^t_h, a^t_h)}_1 + \norm{p_h(s^t_h, a^t_h) - \hp^{\,t-1}_h(s^t_h, a^t_h)}_1  \\
        &\leq \frac{n_0}{\upN^{\,t}_h} + \sum_{s \in \cS} N^{\,t}_h(s) \left(\frac{1}{N^{\,t}_h} - \frac{1}{\upN^{\,t}_h}\right) + \sqrt{\frac{2S L}{N^{\,t}_h}} \leq \frac{S L}{N^{\,t}_h} + \sqrt{\frac{2SL}{N^{\,t}_h}}.
    \end{align*}
    Since $\ur = 2$, Lemma~\ref{lem:regret_sums} implies
    \[
        \termTwo \leq 2 \ur^2 H^2 \sum_{t=1}^T \sum_{h=1}^H \norm{ \up^{\,t}_h(s^t_h, a^t_h) -  p_h(s^t_h, a^t_h) }_1 \leq 16H^3S^2A L^2 + 24 H^3 S\sqrt{2 AT L}.
    \]
    Next, we bound $\termOne$ using the first inequality in Lemma~\ref{lem:switch_variance}. We get
    \[
        \termOne \leq 2 \underbrace{\sum_{t=1}^T \sum_{h=1}^H \Var_{p_ h}[V^{\pi_t}_{h+1}](s^t_h,a^t_h) }_{\termThree} + 2\underbrace{\sum_{t=1}^T \sum_{h=1}^H  \ur H p_h\left| \uV^{t-1}_{h+1} - V^{\pi_t}_{h+1} \right|(s^t_h, a^t_h) }_{\termFour}.
    \]
    The term $\termThree$ could be bounded using definition of the event $\cE^{\Var}$. It follows that
    \[
        \termThree  \leq H^2 T + \sqrt{2H^5 T L} + 3H^3 L.
    \]
    By Proposition~\ref{prop:optimism} we have $\uV^t_{h+1}(s) \geq V^{\pi_t}_{h+1}(s)$ for any $s \in \cS$. By the definition of $\xi^{\,t}_h, \delta^t_h$ and definition of event $\cE$ term $\termFour$ could be bounded as follows
    \begin{align*}
        \termFour &\leq \sum_{t=1}^T \sum_{h=1}^H 2H(\xi^{\,t}_h + \delta^t_{h+1}) \\
        &\leq 2\ur H^2\sqrt{2T L} + 2H \sum_{h=1}^H \uregret^T_{h+1} \leq 4H^2 \sqrt{2T L} + 2H^2 U^T_1.
    \end{align*}
   Here the last inequality follows from Lemma~\ref{lem:surrogate_regret_bound}. Therefore, we have
    \begin{align*}
        \sum_{t=1}^T \sum_{h=1}^H  \Var_{\up^{\,t-1}_h}[\uV^{t-1}_{h+1}](s^t_h,a^t_h) \ind\{N^t_h > 0\} &\leq \termTwo + 2 \cdot \termThree + 2 \cdot \termFour \\
        &\leq 2 H^2 T + 2 H^2 U_1^T  + 22H^3S^2A L^2 + (24+8) H^3 S\sqrt{2AT L}  \\
        &\leq 2 H^2 T + 2 H^2 U_1^T  + 22H^3S^2A L^2 + 32 H^3 S\sqrt{2 AT L}.
    \end{align*}
    To bound the second inequality one may apply the first inequality in Lemma~\ref{lem:switch_variance}. We get
    \[
        \sum_{t=1}^T \sum_{h=1}^H  \Var_{p_h}[\Vstar_{h+1}](s^t_h,a^t_h) \leq 2 \underbrace{\sum_{t=1}^T \sum_{h=1}^H \Var_{p_ h}[V^{\pi_t}_{h+1}](s^t_h,a^t_h)}_{\termThree} + 2\sum_{t=1}^T \sum_{h=1}^H  \ur H p_h\left| \Vstar_{h+1} - V^{\pi_t}_{h+1} \right|(s^t_h, a^t_h).
    \]
    
    Note that by Proposition~\ref{prop:optimism} the second term is bounded by  $\termFour$. Thus
    \begin{align*}
        \sum_{t=1}^T \sum_{h=1}^H  \Var_{p_h}[\Vstar_{h+1}](s^t_h,a^t_h)  &\leq 2\termThree + 2 \termFour  \leq 2H^2T + 2H^2U_1^T + 8\sqrt{2H^5 T L} + 6 H^3 L.
    \end{align*}
\end{proof}

\begin{lemma}\label{lem:surrogate_regret_terms_bound}
    Under conditions of Lemma~\ref{lem:surrogate_regret_bound}, it holds  on the event $\cG(\delta)$,
    \begin{align*}
        A^T_1 &\leq 4\rme\sqrt{H^3 SAT} \cdot  L + 4\rme \sqrt{H^3 SA U^T_1} \cdot L + 14\rme  H^2 S^{3/2} A L^{2} + 20 \rme H^2 S A^{3/4} T^{1/4} L^{5/4}, \\
        B^T_1 &\leq 4\rme \sqrt{H^3 SAT} \cdot  L + 4 \rme\sqrt{H^3 SA U^T_1} \cdot  L + 8 \rme H^2 S^{1/2} A^{1/2} L^2  + 10 \rme H^{7/4} S^{1/2} A^{1/2} T^{1/4} L^{5/4},\\
        C^T_1 &\leq 42 \rme H^3 S^2 A L^2 = \tcO(H^3 S^2 A).
    \end{align*}
\end{lemma}
\begin{proof}
    To bound $A^T_1$ we apply the Cauchy–-Schwartz inequality, Lemma~\ref{lem:sum_variance}, Lemma~\ref{lem:regret_sums} and inequality $\sqrt{a+b} \leq \sqrt{a} + \sqrt{b}, a, b \geq 0$, 
    \begin{align*}
        \sum_{t=1}^T \sum_{h=1}^H &\sqrt{  \Var_{\up^{t-1}_{h}}[\uV^{t-1}_{h+1}](s^t_{h},a^t_{h}) \frac{ \ind\{N^t_h > 0\}}{N^t_{h}}} \leq \sqrt{\sum_{t=1}^T \sum_{h=1}^H \Var_{\up^{t-1}_{h}}[\uV^{t-1}_{h+1}](s^t_{h},a^t_{h})  \ind\{N^t_h > 0\}} \cdot \sqrt{\sum_{t=1}^T \sum_{h=1}^H \frac{\ind\{N^t_h > 0\}}{N^t_{h}}} \\
        &\leq \sqrt{ 2 H^2 T + 2 H^2 U_1^T  + 22H^3S^2A L^2 + 32 H^3 S\sqrt{2 AT L}} \cdot \sqrt{2SAH L} \\
        &\leq 2\sqrt{H^3 SAT  L} + 2\sqrt{H^3 SA U^T_1 L} + 7  H^2 S^{3/2} A L^{3/2} + 10 H^2 S A^{3/4} T^{1/4} L^{3/4}.
    \end{align*}
    Similarly, the term $B^T_1$ may be estimated as follows
    \begin{align*}
        \sum_{t=1}^T \sum_{h=1}^H& \sqrt{\Var_{p_{h}}[\Vstar_{h+1}](s^t_{h},a^t_{h})\frac{ \ind\{N^t_h > 0\}}{N_{h}^t}} \leq \sqrt{\sum_{t=1}^T \sum_{h=1}^H \Var_{p_{h}}[\Vstar_{h+1}](s^t_{h},a^t_{h})} \cdot \sqrt{\sum_{t=1}^T \sum_{h=1}^H \frac{ \ind\{N^t_h > 0\}}{N^t_{h}}} \\
        &\leq \sqrt{2H^2T + 2H^2U_1^T + 8\sqrt{2H^5 T L} + 6 H^3 L} \cdot \sqrt{2SAH \cdot L} \\
        &\leq 2\sqrt{H^3 SAT  L} + 2\sqrt{H^3 SA U^T_1 L} + 4 H^2 L \sqrt{SA} + 5 H^{7/4}T^{1/4} L^{3/4}\sqrt{SA}.
    \end{align*}
    Finally, to estimate $C^T_1$ we apply Lemma~\ref{lem:regret_sums}. We obtain
    \[
        C^T_1 \leq 21\rme H^2 S \cdot L \cdot 2SAH L \leq 42 \rme H^3 S^2 A L^2.
    \]
\end{proof}

% \begin{theorem}[Theorem~\ref{th:regret_bound_bayesUCBVI}] Assume conditions of Proposition~\ref{prop:optimism}. Fix $\delta \in (0,1)$ and set $\uH = H+1,$ $n_0 = \lceil c_{n_0} + \log(5\pi T)/\log(5/4) \rceil$,  where $c_{n_0}$ is defined in \eqref{eq:constant_c_n0}. Then for any $T \geq H^3 S^2 A$, it holds with probability at least $1-\delta,$ 
%     \[
%         \regret^T = \tcO\left( \sqrt{H^3 SAT} + H^3 S^2 A \right).
%     \]
% \end{theorem}
\begin{proof}[Proof of Theorem~\ref{th:regret_bound_bayesUCBVI}]
    Note that by Lemma~\ref{lem:proba_master_event} event $\cG(\delta)$ holds with probability at least $1-\delta$. Next we assume that this event holds. Then we have two cases: $T < H^2 S^2 A L^2$ and $T \geq H^2 S^2 A L^2$. In the first case the regret is trivially bounded by $\regret^T \leq H^3 S^2 A L^2$. Thus it is sufficient to analyze only the second case.
    
    By Proposition~\ref{prop:optimism} and Lemma~\ref{lem:surrogate_regret_bound}
    \begin{equation}
    \label{eq: regret estimation0}
        \regret^T = \sum_{t=1}^T \Vstar_h(s^t_1) - V^{\pi_t}_h(s^t_1) \leq \sum_{t=1}^T \uV^{\,t-1}_h(s^t_1) - V^{\pi_t}_h(s^t_1) = \uregret^T_1 \leq U^T_1 = A^T_1 + B^T_1 + C^T_1 + 4\rme \sqrt{2 H^{3} T L} + 2\rme SAH^2.
    \end{equation}
    Next, under our condition on $T$ we can simplify expressions for the bounds of $A^T_1$ and $B^T_1$. Indeed, $T  \geq H^2 S^2 A L^2$ implies that
    \[
     H^{7/4} S^{1/2} A^{1/2} L^{5/4} \cdot T^{1/4}\leq  H^2 S A^{3/4} L^{5/4}\cdot T^{1/4} \leq \sqrt{H^3 S A  T} L.
    \]
    Furthermore,
    \[
        H^2 S^{3/2} A L^{2} \leq H^3 S^2 A L^{2},\qquad  H^2 S^{1/2} A^{1/2} L^2 \leq H^3 S^2 A L^{2}, \qquad \sqrt{2H^{3} T L} \leq \sqrt{2H^3 SAT} \cdot L.
    \]
    We obtain the following bounds
    \begin{align*}
        A^T_1 &\leq 24\rme \sqrt{H^3 SAT} \cdot  L + 4\rme \sqrt{H^3 SA U^T_1} \cdot L + 14\rme H^3 S^2 A L^{2}, \\
        B^T_1 &\leq 14\rme \sqrt{H^3 SAT} \cdot  L + 4 \rme\sqrt{H^3 SA U^T_1} \cdot  L + 8 \rme H^3 S^2 A L^{2},\\
        C^T_1 &\leq 42 \rme H^3 S^2 A L^2 \leq 42 \rme H^3 S^2 A L^{2}.
    \end{align*}
    Hence, by a bound $SAH^2 \leq H^3 S^2 A L^2$
    \[
        U^T_1 \leq 38 \rme \sqrt{H^3 SAT} \cdot L + 8\rme \sqrt{H^3 SA U^T_1} \cdot L + 66\rme H^3 S^2 A L^{2} + 4\rme\sqrt{2} \cdot \sqrt{H^3 T L}.
    \]
    This is a quadratic inequality in $U^T_1$. Solving this inequality and using inequality $\sqrt{a+b} \leq \sqrt{a} + \sqrt{b}, a, b \geq 0$, we obtain
    \begin{align*}
         U^T_1 &\leq 176 \rme \sqrt{H^3 SAT} \cdot L + 264\rme H^3 S^2 A L^{2} + 256\rme^2 H^3 S A L^2. 
    \end{align*}
     The last inequality and \eqref{eq: regret estimation0} imply that 
    \[
        \regret^T  = \cO\left( \sqrt{H^3 SAT}L + H^3 S^2 A L^{2} \right).
    \]
    
\end{proof}

\newpage
%!TEX root = ../BayesUCBVI.tex
\section{Deviation Inequalities}
\label{app:concentration}

\subsection{Deviation inequality for categorical distributions}

Next, we reproduce the deviation inequality for categorical distributions by \citet[Proposition 1]{jonsson2020planning}.
Let $(X_t)_{t\in\N^\star}$ be i.i.d.\,samples from a distribution supported on $\{1,\ldots,m\}$, of probabilities given by $p\in\simplex_{m-1}$, where $\simplex_{m-1}$ is the probability simplex of dimension $m-1$. We denote by $\hp_n$ the empirical vector of probabilities, i.e., for all $k\in\{1,\ldots,m\},$
 \[
 \hp_{n,k} = \frac{1}{n} \sum_{\ell=1}^n \ind\left\{X_\ell = k\right\}.
 \]
 Note that  an element $p \in \simplex_{m-1}$ can be seen as an element of $\R^{m-1}$ since $p_m = 1- \sum_{k=1}^{m-1} p_k$. This will be clear from the context. 
%  We denote by $H(p)$ the (Shannon) entropy of $p\in\Sigma_m$,
%  \[
%  H(p) = \sum_{k=1}^m p_k \log\left(\frac{1}{p_k}\right)\cdot
%  \]
 \begin{theorem} \label{th:max_ineq_categorical}
 For all $p\in\simplex_{m-1}$ and for all $\delta\in[0,1]$,
 \begin{align*}
     \P\left(\exists n\in \N^\star,\, n\KL(\hp_n, p)> \log(1/\delta) + (m-1)\log\left(e(1+n/(m-1))\right)\right)\leq \delta.
 \end{align*}
\end{theorem}

\subsection{Deviation inequality for categorical weighted sum}
%  Let $(X_t)_{t\in\N^\star}$ be i.i.d.\,samples from a distribution supported on $\{1,\ldots,m\}$, of probabilities given by $p\in\Sigma_m$, where $\Sigma_m$ is the probability simplex of dimension $m-1$. We denote by $\hp_n$ the empirical vector of probabilities, i.e., for all $k\in\{1,\ldots,m\},$
%  \[
%  \hp_{n,k} = \frac{1}{n} \sum_{\ell=1}^n \ind\left\{X_\ell = k\right\}.
%  \]
%  Note that  an element $p \in \Sigma_m$ can be seen as an element of $\R^{m-1}$ since $p_m = 1- \sum_{k=1}^{m-1} p_k$. 
 We fix a function $f: \{1,\ldots,m\} \mapsto [0,b]$ and recall the definition of the minimal Kullback-Leibler divergence for $p\in\simplex_{m-1}$  and $u\in\R$
 \[
\Kinf(p,u,f) = \inf\left\{  \KL(p,q): q\in\simplex_{m-1}, qf \geq u\right\}\,.
 \]
As the Kullback-Leibler divergence this quantity admits a variational formula.
\begin{lemma}[Lemma 18 by \citet{garivier2018kl}]
\label{lem:var_form_Kinf} For all $p \in \simplex_{m-1}$, $u\in [0,b)$,
\[
\Kinf(p,u,f) = \max_{\lambda \in[0,1]} \E_{X\sim p}\left[ \log\left( 1-\lambda \frac{f(X)-u}{b-u}\right)\right]\,,
 \]
 moreover if we denote by $\lambda^\star$ the value at which the above maximum is reached, then
 \[
   \E_{X\sim p} \left[\frac{1}{1-\lambda^\star\frac{f(X)-u}{b-u}}\right] \leq 1\,.
 \]
\end{lemma}
\begin{remark} Contrary to \citet{garivier2018kl} we allow that $u=0$ but in this case Lemma~\ref{lem:var_form_Kinf} is trivially true, indeed
  \[
  \Kinf(p,0,f) =  0  = \max_{\lambda \in[0,1]} \E_{X\sim p}\left[ \log\left( 1-\lambda \frac{f(X)}{b}\right)\right]\,.
   \]
\end{remark}

We are now ready to state the deviation inequality for the $\Kinf$ which is a self-normalized version of Proposition~13 by \citet{garivier2018kl}.
 \begin{theorem} \label{th:max_ineq_kinf}
 For all $p\in\simplex_{m-1}$ and for all $\delta\in[0,1]$,
 \begin{align*}
     \P\big(\exists n\in \N^\star,\, n\Kinf(\hp_n, pf, f)> \log(1/\delta) + 3\log(e\pi(1+2n))\big)\leq \delta.
 \end{align*}
\end{theorem}

 \begin{proof}
First if $pf=b$ then $f(k)=b$ for all $k$ such that $p_k>0$. In this case $\Kinf(\hp_n, pf, f)=0$ for all $n$ and the result is trivially true. We thus assume now that $pf<b$.

The proof is a combination of the one of Proposition~13 by \citet{garivier2018kl} and the method of mixtures. We first define the martingale
\[
M_n^\lambda = \exp\left(\sum_{\ell=1}^n \log\left(1-\lambda \frac{f(X_\ell)-pf}{b-pf}\right)\right)\,,
\]
with the convention $M_0^\lambda= 1$. Indeed if we denote by $\cF_n = \sigma(X_1,\ldots,X_n)$ the information available at time $n$, we have
\begin{align*}
  \E\left[M_n^\lambda|\cF_{n-1}\right] = \E\left[1-\lambda\frac{f(X_n)-pf}{b-pf}\right] M_{n-1}^\lambda = M_{n-1}^\lambda\,.
\end{align*}
We fix a real number $\gamma_j = 1/(2j)$ for $j\in\N^*$and let $S_j$ be the set
\[
S_j= \Bigg\{ \frac{1}{2}-\Bigg\lfloor\frac{1}{2\gamma_j}\Bigg\rfloor\gamma_j, \dots,\frac{1}{2}-\gamma_j,\,\frac{1}{2},\,\frac{1}{2}+\gamma_j,\dots,\frac{1}{2}+\Bigg\lfloor\frac{1}{2\gamma_j}\Bigg\rfloor\gamma_j \Bigg\}\,.
\]
The cardinality of this set $S_j$ is bounded by $1 + 2j$. We choose a prior on $\lambda$ the mixture of uniform distribution over this grid: $6/\pi^2\sum_{j=1}^{\infty} 1/j^2 \cU(S_j)$. Thus we consider the integrated martingale
 \begin{align}
     M_n &= \frac{6}{\pi^2}\sum_{j=1}^{\infty} \frac{1}{j^2} \sum_{\lambda \in S_j} \frac{1}{|S_j|}M_n^{\lambda} \nonumber\\
     &\geq \frac{6}{\pi^2 n^2 |S_n|} \max_{\lambda \in S_n}M_n^{\lambda}\nonumber\\
     &\geq \frac{6}{\pi^2 (1+2n)^3}\max_{\lambda \in S_n}M_n^{\lambda}\,.\label{eq:lb_mixture_kinf}
 \end{align}
Lemma~\ref{lem:regularity_ln_lambda} below
indicates that for all $\lambda\in[0,1]$, there exists a $\lambda'\in S_n$ such that for all $x\in[0,b]$,
\begin{equation}
\label{eq:2gamma}
\log\!\Bigg(1-\lambda \, \frac{x-p f}{b-p f } \Bigg)\leq 2\gamma_n
+ \log\!\Bigg(1- \lambda' \frac{x-p f}{b-p f} \Bigg)\,.
\end{equation}
Now, a combination of
the variational formula of Lemma~\ref{lem:var_form_Kinf}
and of the inequality~\eqref{eq:2gamma} yields a finite maximum as an upper bound on $\Kinf(\hp_n,pf,f)$
\begin{align*}
\Kinf(\hp_n,pf,f)
& = \max_{0\leq \lambda\leq 1} \frac{1}{n}\sum_{\ell=1}^{n} \log\!\Bigg(1-\lambda \frac{X_\ell-pf}{b-pf} \Bigg) \\
& \leq 2 \gamma_n + \max_{\lambda' \in S_n} \frac{1}{n}\sum_{k=1}^{n}\log\!\Bigg(1-\lambda' \frac{X_\ell-pf}{b-pf} \Bigg)\,.
\end{align*}
Thanks to the definition of the martingale $M_n^\lambda$ we obtain
\[
\max_{\lambda \in S_n}M_n^{\lambda} \geq e^{-2n\gamma_n} e^{n\Kinf(\hp_n,pf,f)}= e^{-1} e^{n\Kinf(\hp_n,pf,f)}\,.
\]
Combining this inequality with \eqref{eq:lb_mixture_kinf} yields
\[
 M_n  \geq  \frac{6}{e\pi^2 (1+2n)^3}e^{n\Kinf(\hp_n,pf,f)}\,.
\]
Since for any supermartingale we have that
 \begin{equation}\P\left(\exists n \in \N : M_n > 1/\delta\right) \leq \delta \cdot \E[M_0],\label{eq:supermartingale}\end{equation}
 which is a well-known property of the method of mixtures \citep{de2004self}, we conclude that
 \[
 \P\left(\exists n\in \N^\star,\, n\Kinf(\hp_n,pf,f)> \log(1/\delta)+ 3\log(e\pi(1+2n))\right) \leq \delta\,.
 \]
\end{proof}

\begin{lemma}[Lemma 19 by \citealp{garivier2018kl} and comment below]
For all $\lambda, \lambda'\in[0,1]$ such that either $\lambda \leq \lambda'\leq 1/2$ or $1/2\leq \lambda'\leq \lambda$, for all real numbers $c\leq 1$,
\begin{equation*}
\log(1-\lambda c)-\log(1-\lambda' c)\leq 2|\lambda-\lambda'|\,.
\end{equation*}
\label{lem:regularity_ln_lambda}
\end{lemma}

\subsection{Deviation inequality for bounded distributions}
Below, we reproduce the self-normalized Bernstein-type inequality by \citet{domingues2020regret}. Let $(Y_t)_{t\in\N^\star}$, $(w_t)_{t\in\N^\star}$ be two sequences of random variables adapted to a filtration $(\cF_t)_{t\in\N}$. We assume that the weights are in the unit interval $w_t\in[0,1]$ and predictable, i.e. $\cF_{t-1}$ measurable. We also assume that the random variables $Y_t$  are bounded $|Y_t|\leq b$ and centered $\EEc{Y_t}{\cF_{t-1}} = 0$.
Consider the following quantities
\begin{align*}
		S_t \triangleq \sum_{s=1}^t w_s Y_s, \quad V_t \triangleq \sum_{s=1}^t w_s^2\cdot\EEc{Y_s^2}{\cF_{s-1}}, \quad \mbox{and} \quad W_t \triangleq \sum_{s=1}^t w_s
\end{align*}
and let $h(x) \triangleq (x+1) \log(x+1)-x$ be the Cramér transform of a Poisson distribution of parameter~1.

\begin{theorem}[Bernstein-type concentration inequality]
  \label{th:bernstein}
	For all $\delta >0$,
	\begin{align*}
		\PP{\exists t\geq 1,   (V_t/b^2+1)h\left(\!\frac{b |S_t|}{V_t+b^2}\right) \geq \log(1/\delta) + \log\left(4e(2t+1)\!\right)}\leq \delta.
	\end{align*}
  The previous inequality can be weakened to obtain a more explicit bound: if $b\geq 1$ with probability at least $1-\delta$, for all $t\geq 1$,
 \[
 |S_t|\leq \sqrt{2V_t \log\left(4e(2t+1)/\delta\right)}+ 3b\log\left(4e(2t+1)/\delta\right)\,.
 \]
\end{theorem}

\subsection{Deviation inequality for Dirichlet distribution}
Below we provide the Bernstein-type inequality for weighted sum of Dirichlet distribution, using a result on upper bound on tails of Dirichlet boundary crossing (see Lemma~\ref{lem:upper_bound_dbc}).

\begin{lemma}\label{lem:kinf_attains}
     For any $p \in \simplex_m$, $f \colon \{0,\ldots,m\} \to [0,b]$ such that $f(0) = b$, $p_0 > 0$, and $\mu \in (p f, b)$ there exists a measure  $q \in \simplex_m$ such that $p \ll q $, $q f = \mu$ and \(\Kinf(p, \mu, f) = \KL(p, q)\).
\end{lemma}
\begin{proof}
    By the variational form of $\Kinf$ (Lemma~\ref{lem:var_form_Kinf})
    \[ 
        \Kinf(p, \mu, f) = \max_{\lambda \in [0,1]} \E_{X \sim p}\left[ \log\left(1 - \lambda \frac{f(X) - \mu}{b - \mu} \right) \right] = \E_{X \sim p}\left[ \log\left(1 - \lambda^\star \frac{f(X) - \mu}{b - \mu} \right) \right].
    \]
    Note that $\PP{f(X) = b} > 0$ implies $\lambda^\star < 1$. Jensen's inequality and $\mu > pf$ imply $\lambda^\star > 0$. It is easy to check that $\lambda^\star$ satisfies
    \[
        \E\left[ \frac{1}{1 - \lambda^\star (f(X) - \mu)/(b-\mu)} \right] = \sum_{j=0}^m \frac{p(j)}{1 - \lambda^\star (f(j) - \mu)/(b-\mu)} = 1,
    \]
    and 
    \begin{equation}\label{eq: kinf_attains1}
        \E\left[ \frac{f(X) - \mu}{1 - \lambda^\star (f(X) - \mu)/(b-\mu)} \right] = \sum_{j=0}^m \frac{p(j) (f(j) - \mu)}{1 - \lambda^\star (f(j) - \mu)/(b-\mu)} = 0.
    \end{equation}
    Define $q(j) = \frac{p(j)}{1 - \lambda^\star (f(j) - \mu)/(b-\mu)}, j = 0, \ldots, m$, and let $q = (q_0, \ldots, q_m)$. Clearly, $q \in \simplex_m$, $qf = \mu$ by \eqref{eq: kinf_attains1} and $p \ll q$. Moreover,
    \[
        \Kinf(p, \mu, f) = \E_{X \sim p}\left[ \log\left(1 - \lambda^\star \frac{f(X) - \mu}{b - \mu} \right) \right] = \E_{p}\left[ \log \frac{\rmd p}{\rmd q}\right] = \KL(p,q).
    \]
\end{proof}

\begin{lemma}\label{lem:bernstein_dirichlet}
     For any $\alpha = (\alpha_0, \alpha_1, \ldots, \alpha_m) \in \N^{m+1}$ define  $\up \in \simplex_{m}$ such that $\up(\ell) = \alpha_\ell/\ualpha, \ell = 0, \ldots, m$, where $\ualpha = \sum_{j=0}^m \alpha_j$. Then for any $f \colon \{0,\ldots,m\} \to [0,b]$ such that $f(0) = b$ and $\delta \in (0,1)$
     \[
        \P_{w \sim \Dir(\alpha)}\left[wf \geq \up f +  2 \sqrt{ \frac{ \Var_{\up}(f) \log(1/\delta)}{\ualpha}} + \frac{2b\sqrt 2  \cdot \log(1/\delta)}{\ualpha} \right] \leq \delta.
     \]
\end{lemma}
\begin{proof}
Fix $\delta \in (0,1)$ and let $\mu \in (\up f,b)$ be such that
\[
    \Kinf(\up, \mu, f) = \ualpha^{-1} \log (1/\delta). 
\]
Note that such $\mu$ exists. Indeed, it follows from the continuity of $\Kinf$ w.r.t. the second argument, see \citet[Theorem 7]{honda2010asymptotically}. By Lemma~\ref{lem:kinf_attains} there exists $q$ such that $\up \ll q $, $q f = \mu$ and $\KL(\up, q) = \ualpha^{-1} \log (1/\delta)$. By Lemma \ref{lem:upper_bound_dbc}   
\begin{equation}
\label{eq: bernstein 1}
 \P_{w \sim \Dir(\alpha)}[wf \geq qf] = \P_{w \sim \Dir(\alpha)}[wf \geq \mu] \leq \exp\left( -\ualpha \Kinf(\up, \mu, f) \right) = \delta.   
\end{equation}
By Lemma \ref{lem:Bernstein_via_kl}
\[
q f  - \up f \le \sqrt{2\Var_{q}(f)\KL(\up ,q)}.
\]
By Lemma \ref{lem:switch_variance_bis},  $\Var_q(f) \leq 2\Var_{\up}(f) +4b^2 \KL(\up ,q)$.
The last two inequalities and \eqref{eq: bernstein 1} imply that
\begin{equation*}
  \P_{w \sim \Dir(\alpha)}\left[wf - \up f \geq \sqrt{4 \Var_{\up}(f) \KL(\up ,q)} + 2 b \sqrt 2 \cdot      \KL(\up,q) \right] \le \delta. 
\end{equation*}
  
\end{proof}

\newpage
\section{Dirichlet Boundary Crossing}
\label{sec:dbc}

In this section we will provide upper and lower bounds on the Dirichlet boundary crossing. The proof of the upper bound follows \citet{baudry2021optimality}; see also \citet{riou20a}. 
\begin{lemma}[Upper bound]
\label{lem:upper_bound_dbc}
    For any $\alpha = (\alpha_0, \alpha_1, \ldots, \alpha_m) \in \N^{m+1}$ define  $\up \in \simplex_{m}$ such that $\up(\ell) = \alpha_\ell/\ualpha, \ell = 0, \ldots, m$, where $\ualpha = \sum_{j=0}^m \alpha_j$. Then for any $f \colon \{0,\ldots,m\} \to [0,b]$ and $0 < \mu < b$ and
  \[
        \P_{w\sim \Dir(\alpha)} [wf \geq \mu] \leq \exp(-\ualpha \Kinf(\up,\mu, f)).
  \]
\end{lemma}

\begin{proof}
  First if $\mu\leq \up f$ then the upper bound is trivial since in this case $\Kinf(\up,\mu,f)=0$. Assume that $\mu > \up f$. It is well know fact that $w \sim \Dir(\alpha)$ may be represented as follows
  \[
        w \triangleq \left( \frac{Y_0}{V_m}, \frac{Y_1}{V_m}, \ldots, \frac{Y_m}{V_m} \right),
  \]
  where $Y_\ell \mysim \Gamma(\alpha_\ell, 1), \ell = 0, \ldots, m$ and $V_m = \sum_{\ell=0}^m Y_\ell$. Furthermore, denoting $v_\ell \mysimiid \Exponential(1), \ell = 1, \ldots, \ualpha$, we get
  \[
        wf = \sum_{\ell=0}^m w_\ell f(\ell) = \frac{\sum_{j=1}^{\ualpha} v_j x_j}{\sum_{j=1}^{\ualpha} v_j}, 
  \]
  where $x_j = f(\ell)$ iff $\sum_{k=0}^\ell \alpha_k < j \leq \sum_{k=0}^{\ell+1} \alpha_k$. Changing measure and using variational formula for the minimal Kullback-Leibler divergence we get for $\lambda\in[0,1/(b-\mu))$
  \begin{align*}
    \P_{w\sim \Dir(\alpha)} [wf \geq \mu] &= \E_{v_\ell \mysimiid \Exponential(1)}\left[\ind\left\{\sum_{\ell=1}^{\ualpha} v_\ell(x_\ell-\mu)\geq 0\right\}\right]\\
    &= \E_{\hat v_\ell\mysim\Exponential\big(1-\lambda(x_\ell-\mu)\big)}\left[\ind\left\{\sum_{\ell=1}^{\ualpha} \hat v_\ell(x_\ell-\mu)\geq 0\right\} \cdot \prod_{\ell=1}^{\ualpha} \frac{\rme^{(1-\lambda(x_\ell-\mu))\hat v_\ell-\hat v_\ell}}{1-\lambda(x_\ell-\mu)}\right]\\
    &= \rme^{-\sum_{\ell=1}^{\ualpha}\log(1-\lambda(x_\ell-\mu)) } \E_{\hat v_\ell\mysim\Exponential\big(1-\lambda(x_\ell-\mu)\big)}\left[\ind\left\{\sum_{\ell=1}^{\ualpha} \hat v_\ell(x_\ell-\mu)\geq 0\right\}\rme^{-\lambda\sum_{\ell=1}^{\ualpha} \hat v_\ell(x_\ell-\mu))}\right] \\
    &\leq \exp\left(-\sum_{\ell=1}^{\ualpha} \log(1-\lambda(x_\ell-\mu))\right) = \exp\left(-\sum_{\ell=0}^{m} \alpha_\ell \log(1-\lambda(f(\ell)-\mu))\right)\,,
  \end{align*}
  where the last equality follows from regrouping all $x_j$ back to $f(\ell)$. Since the previous inequality is true for all $\lambda \in[0,1/(b-\mu))$, then the variational formula (Lemma~\ref{lem:var_form_Kinf}) allows to conclude
\begin{align*}  
    \P_{w\sim \Dir(\alpha)} [w f\geq \mu] &\leq  \exp\left(-\sup_{\lambda \in[0,1/(b-\mu))}\sum_{\ell=0}^{m} \alpha_\ell \log(1-\lambda(f(\ell)-\mu))\right) = \exp(-\ualpha \Kinf(\up,\mu, f)).
\end{align*}
\end{proof}

\begin{theorem}[Lower bound]\label{thm:lower_bound_dbc} For any $\alpha = (\alpha_0, \alpha_1, \ldots, \alpha_m) \in \N^{m+1}$ define  $\up \in \simplex_{m}$ such that $\up(\ell) = \alpha_\ell/\ualpha, \ell = 0, \ldots, m$, where $\ualpha = \sum_{j=0}^m \alpha_j$. Assume that
     \[
        \alpha_0 \geq \max\left\{\frac{1}{(\sqrt{2\pi} - 1)^2} \cdot \left(\frac{2 \sqrt{2}}{\sqrt{\log(17/16)}} + \frac{98 \sqrt{6}}{9} \right)^2,  \frac{\log(10\pi \cdot \ualpha)}{\log(17/16) } \right\}
    \]
    and $\ualpha \geq 2\alpha_0$. Then for any $f \colon \{0,\ldots,m\} \to [0,\ub]$ such that $f(0) = \ub$, $f(j) \leq b < \ub/2, j \in \{1,\ldots,m\}$ and $\mu \in (\up f,  \ub)$ 
    \[
        \P_{w \sim \Dir(\alpha)}[wf \geq \mu] \geq  \exp\left(-\ualpha \Kinf(\up, \mu, f) - 3/2 \log \ualpha\right).
    \]
\end{theorem}
In the further subsections we are going to prove this theorem.

\subsection{Proof of Theorem \ref{thm:lower_bound_dbc}}

Throughout this section we will often use the following notations.  Let $\Fclass_m(b) = \{f \colon \{0,\ldots,m\} \to [0,b] \}$ and for $b < \ub$, $\Fclass_m(\ub, b) = \{f \colon \{0,\ldots,m\} \to [0,\ub], f(0) = \ub, f(j) \leq b, j  = 1,\ldots,m \}$.  For any $\alpha = (\alpha_0, \alpha_1, \ldots, \alpha_m) \in \N^{m+1}$ define  $\up = \up(\alpha) \in \simplex_{m}$ such that $\up(\ell) = \alpha_\ell/\ualpha, \ell = 0, \ldots, m$, where $\ualpha = \sum_{j=0}^m \alpha_j$

\paragraph{Density of weighted sum of the Dirichlet distribution}

In this section we compute the density of a random variable $Z = w f$, where $w \sim \Dir(\alpha)$ and $f \in \Fclass_m(b)$. 

\begin{proposition}\label{prop:density_linear_dirichlet}
    Let $f \in \Fclass_m(b)$ and $\alpha = (\alpha_0, \alpha_1, \ldots, \alpha_m) \in \R_+^{m+1}$ such that $\ualpha = \sum_{j=0}^m \alpha_j > 1$. Assume that $Z$ is not degenerate. Then for any $0 \leq u < \ub$
    \[
        p_{Z}(u) = \frac{\ualpha - 1}{2\pi} \int_\R \prod_{j=0}^m \left( 1 + \rmi(f(j) - u)s \right)^{-\alpha_j} \rmd s.
    \]
\end{proposition}
Proof of Proposition~\ref{prop:density_linear_dirichlet} will be given at the end of this paragraph.

A function $g \colon \R^{m+1} \to \R$ is called a positive homogeneous on a cone $A \subseteq \R^{m+1}$ of degree $t$ if for any $\gamma > 0$ and $x \in A$ we have $g(\gamma x) = \gamma^t g(x)$. Define $\widetilde{\simplex}_m = \conv(0, \simplex_m) = \{ w \in \R_+^{m+1} : \sum_{j=0}^m w_j \leq 1\}$ as a pyramid with a base $\simplex_m$ and apex at $0$. Denote $\overline{\simplex}_m = \{w \in \R^{m+1}: \sum_{\ell=0}^m w_\ell = 1\}$. For $r > 0$ we write $\simplex_m(r) = \{w \in \R^{m+1}_{+}: \sum_{\ell=0}^m w_\ell = r\}$. Then, clearly $\simplex_m = \simplex_m(1)$. For any $a \in \R^{m+1}$ define $\Hset_a = \{ w \in \R^{m+1} : \langle a, w \rangle = 0\}$. 

For any measurable set $A \subseteq \R^{m+1}$ of dimension $n < m+1$ and any function $g \colon \R^{m+1} \to \R$ define 
\[
    I_n(g, A) = \int_{A} g(w) \cH^n(\rmd w),
\]
where $\cH^n$ is an $n$-dimensional Hausdorff measure (see \citealp{evans2018measure}, for definition). If $A = \cL(Y)$ for a linear map $\cL \colon \R^n \to \R^{m+1}$ and $Y \subseteq \R^n$, then we can write
\[
    I_n(g, \cL(Y)) = [\cL] \cdot \int_{Y} g(\cL(y)) \lambda_n(\rmd y),
\]
where $\lambda_n$ is an $n$-dimensional Lebesgue measure on $Y$ and $[\cL]$ is a Jacobian of the map $\cL$ that could be computed as $[\cL] = \sqrt{\det(\cL \cL^\top)}$. Let us define an affine map $\cL^t_a \colon \R^m \to \R^{m+1}$ that transforms $\R^m$ to $\Hset_a^t$ by mapping $x_1,\ldots,x_m$  to $w_1,\ldots,w_m$ and $w_0 = \frac{t - \sum_{j=1}^m a_j x_j}{a_0}$ for $a_0 > 0$ (without loss of generality). The linear part of this map has a Jacobian that is equal to $[\cL^t_a] = \frac{\norm{a}_2}{a_0}$ (see Lemma~\ref{lem:jacob}). Additionally, define $\cL_a = \cL_a^0$. 
% For a function $g \colon \R^{m+1} \to \R$ we define
% $I_{m+1}(\widetilde{\simplex}_m \cap \Hset_a) = \int_{\tilde \simplex_m \cap \Hset_a} g(w) \rmd w$ and $I_{m}(\simplex_m(r) \cap \Hset_a) = \int_{\simplex_m(r) \cap \Hset_a} g(w) \rmd w$.

\begin{lemma}\label{lem:pyramid_volume}
    Let $g$ be a positively homogeneous function of degree $t > -m$ on $\R_{++}^{m+1}$. Then we have
    \[
        I_{m}(g, \widetilde{\simplex}_{m} \cap \Hset_a) = \frac{\mathrm{dist}(\overline{\simplex}_m \cap \Hset_a, 0)}{m+t}  I_{m-1}(g, \simplex_m \cap \Hset_a).
    \]
\end{lemma}
\begin{proof}
    Denote $h = \mathrm{dist}(\overline{\simplex}_m \cap \Hset_a, 0)$. The change of variable formula (\citealp{evans2018measure}, 3.4.3) implies that
    \[
        I_{m}(g, \widetilde{\simplex}_{m} \cap \Hset_a) = \int_0^h I_{m-1}\left(g, \simplex_{m}(s/h) \cap \Hset_a\right) \rmd s.
    \]
    Using definition of a positive homogeneous function and properties of the Haudorff measure $\cH^{m-1},$ we derive
    \begin{align*}
        I_{m-1}\left(g, \simplex_m(s/h) \cap \Hset_a\right) &= \int_{\simplex_m \cap \Hset_a} g(w \cdot (s/h)) \cH^{m-1}(\rmd(w \cdot s / h))  = \\
        &=  \left(\frac{s}{h}\right)^{m+t-1} \int_{\simplex_m \cap \Hset_a} g(w) \cH^{m-1}(\rmd w) = \left(\frac{s}{h}\right)^{m+t-1}  \cdot I_{m-1}(g, \simplex_m \cap \Hset_a).
    \end{align*}
   Hence,
    \begin{align*}
        I_{m}(g, \widetilde{\simplex}_m \cap \Hset_a)  =  \frac{I_{m-1}\left(g, \simplex_m \cap \Hset_a\right)}{h^{m+t-1}} \int_0^h s^{m+t-1} \rmd s =  \frac{h I_{m-1}\left(g, \simplex_m \cap \Hset_a\right) }{m+t}\cdot
    \end{align*}
\end{proof}

Now we  see that in order to  find $I_{m-1 }\left(g, \simplex_m \cap \Hset_a\right)$ it is sufficient to compute the integral $I_{m}(g, \widetilde{\simplex}_m\cap \Hset_a)$ and the distance $\mathrm{dist}(\overline{\simplex}_m \cap \Hset_a, 0)$. This distance was computed in \citet{dirksen2015sections}.
\begin{lemma}[ Lemma~3.2 in \citealp{dirksen2015sections}]\label{lem:simplex_distance}
    Let $a \in \R^{m+1}$ be such that $\Vert a \Vert_2 = 1$. Then
    \[
        \mathrm{dist}(\overline{\simplex}_m \cap \Hset_a, 0) = \frac{1}{\sqrt{m+1 - (\sum_{i=0}^m a_i)^2}}.
    \]
\end{lemma}
Without normalization of the vector $a$ we get as a corollary the following representation.
\begin{corollary}\label{cor:simplex_distance}
   Let $a \in \R^{m+1}$. Then
    \[
        \mathrm{dist}(\overline{\simplex}_m \cap \Hset_a, 0) = \frac{\norm{a}_2}{\sqrt{(m+1) \left(\sum_{j=0}^m a_j^2\right) - (\sum_{j=0}^m a_j)^2}}.
    \]
\end{corollary}

Next we provide another representation of the integral $I_{m}(g, \widetilde{\simplex}_m\cap \Hset_a)$. We follow \citet{lasserre2020simple} and use the same technique based on the Laplace transform.
\begin{lemma}\label{lem:simplex_laplace}
    Let $g$ be a positively homogeneous function of degree $t$ on $\R_{++}^{m+1}$ such that $t > -(1+m)$ and $\int_{\widetilde{\simplex}_m} \vert g(w) \vert \rmd w < \infty$. Then
    \[
        I_{m}(g, \widetilde{\simplex}_m \cap \Hset_a) = \frac{1}{\Gamma(1 + m + t)} \int_{\R_+^{m+1} \cap \Hset_a} g(w) \exp\left(-\sum_{\ell=0}^m w_\ell\right) \cH^m(\rmd w).
    \]
\end{lemma}
\begin{proof}
    Consider $h(y) = \int_{w \geq 0, \langle \bOne, w \rangle \leq y, \langle a, w \rangle =0} g(w) \cH^m(\rmd w)$. Clearly, $h(1) = I_{m}(g, \widetilde{\simplex}_m \cap \Hset_a)$. Since $g$ is positively homogeneous function we get $h(y) = y^{m+t} h(1)$. This implies that the Laplace transform of $h$ is equal to $L(\lambda) = \int_0^\infty h(y) \rme^{-\lambda y} \rmd y = h(1) \frac{\Gamma(m+t+1)}{\lambda^{m + t + 1}}$. On the other hide, the Laplace transform $L(\lambda)$ may be calculated via a linear parametrization of the subspace $\langle a,w \rangle = 0$ using the map $\cL_a$ and the Fubini theorem
    \begin{align*}
        L(\lambda) &= \int_0^\infty \rme^{-\lambda y} \left[ \int_{w \in \R^{m+1}_+, \langle \bOne,w \rangle \leq y, \langle a,w \rangle =0 }  g(w) \cH^m(\rmd w) \right] \rmd y\\
        &= [\cL_a]\int_{\cL_a(x) \geq 0} \rmd x \cdot g(\cL_a(x)) \cdot \left[ \int_{\langle \bOne, \cL_a(x) \rangle \le y}  \rme^{-\lambda y} \rmd y \right] \\
        &= \frac{[\cL_a]}{\lambda} \int_{\cL_a(x) \geq 0} g(\cL_a(x)) \exp\left( -\lambda \langle \bOne, \cL_a(x) \rangle \right) \rmd x \\
        &= \frac{[\cL_a]}{\lambda^{m+t+1}} \int_{\cL_a(x) \geq 0} g(\cL_a(x)) \exp\left( -\langle \bOne, \cL_a(x) \rangle \right) \rmd x \\
        &= \frac{1}{\lambda^{m+t+1}} \int_{\R^{m+1}_+ \cap \Hset_a} g(w) \exp\left( -\sum_{\ell=0}^m w_\ell  \right) \cH^m(\rmd w).
    \end{align*}
    Identifying two ways of computation of the Laplace transform, we finish the proof.
\end{proof}

We now compute the integral in the r.h.s.\,of Lemma~\ref{lem:simplex_laplace}. We shall use the Fourier transform method and follow the approach of \citet{dirksen2015sections}. 
\begin{lemma}\label{lm:simplex_fourier_integral}
    Let \(g(w) = w_0^{\alpha_0 -1} \cdot \ldots \cdot w_{m}^{\alpha_m - 1}\). Then we have
    \[
        \int_{\R_+^{m+1} \cap \Hset_a} g(w) \exp\left(-\sum_{i=0}^m w_i\right) \cH^m(\rmd w) = \frac{\norm{a}_2  \cdot \prod_{j=0}^m \Gamma(\alpha_j)}{2\pi} \int_{\R} \prod_{j=0}^m (1 + \rmi a_j \tau)^{-\alpha_j} \rmd \tau.
    \]
\end{lemma}
\begin{proof}
    % We rewrite the integral 
    % \[
    %     \int_{\R_+^{m+1} \cap \Hset_a} g(w) \exp\left(-\sum_{i=0}^m w_i\right) \rmd w = \int_{\forall j: \langle a, w\rangle = 0, w_j \geq 0} g(w) \exp\left(-\sum_{i=0}^m w_i\right) \rmd w.
    % \]
    Denote for any $t \in \R$
    \[
        G(t) = \int_{\forall j: \langle a, w\rangle = t, w_j \geq 0} g(w) \exp\left(-\sum_{\ell=0}^m w_\ell\right) \cH^m (\rmd w).
    \]
    Next, we apply affine parametrization induced by map $\cL^t_a$
    \[
        G(t) = [\cL^t_a] \int_{\substack{\forall j\  x_j \geq 0 \\  t - \sum_{j=1}^m a_j x_j \geq 0}} \prod_{j=1}^m x_j^{\alpha_j-1} \cdot \left(\frac{t - \sum_{j=1}^m a_j x_j}{a_0}\right)^{\alpha_0 - 1} \exp\left( -\sum_{j=1}^m x_j - \left( \frac{t - \sum_{j=1}^m a_j x_j}{a_0} \right) \right) \rmd x.
    \]
    By Lemma~\ref{lem:jacob} $[\cL^t_a] = \frac{\norm{a}_2}{a_0}$. The Fourier transform of $G$ may be calculated using the Fubini's theorem
    \begin{align*}
        \mathcal{F}[G](\tau) &= \frac{\norm{a}_2 }{\sqrt{2\pi} \cdot a_0}\iint\limits_{\substack{\forall j\  x_j \geq 0 \\  t - \sum_{j=1}^m a_j x_j \geq 0 \\ t \in \R}} \prod_{j=1}^m x_j^{\alpha_j-1} \biggl(\frac{t - \sum_{j=1}^m a_j x_j}{a_0}\biggr)^{\alpha_0 - 1} \exp\biggl( -\sum_{j=1}^m x_j - \biggr( \frac{t - \sum_{j=1}^m a_j x_j}{a_0} \biggr) \biggr) \cdot \rme^{-\rmi t\tau} \rmd t \rmd x.
    \end{align*}
    By the change of variables $s = t - \sum_{j=1}^m a_j x_j,$ the above integral can be represented as a product of  $m+1$  one-dimensional integrals
    \[
        \mathcal{F}[G](\tau) = \frac{\norm{a}_2 }{\sqrt{2\pi} \cdot a_0}\int_0^\infty \rmd x_1 x_1^{\alpha_1 - 1} \rme^{-x_1 - \rmi \tau a_1 x_1}\ldots \int_0^\infty \rmd x_m x_m^{\alpha_m - 1} \rme^{-x_m - \rmi \tau a_m x_m} \int_0^\infty \rmd s \left( \frac{s}{a_0} \right)^{\alpha_0-1} \rme^{-s/a_0 - \rmi\tau s}.
    \]
    Performing the change of variables $x_0 = s/a_0$, we arrive at the product of $m$ characteristic functions of independent $\Gamma(\alpha_\ell, 1)$-distributed random variables
    \[
         \mathcal{F}[G](\tau) = \frac{\norm{a}_2 }{\sqrt{2\pi}} \cdot \prod_{j=0}^m \Gamma(\alpha_j) \cdot \frac{1}{(1 + \rmi a_j \tau)^{\alpha_j}}\cdot 
    \]
    Finally, by the inverse Fourier transform
    \[
        G(0) = \frac{\norm{a}_2 }{\sqrt{2\pi}} \int_{\R}   \mathcal{F}[G](\tau)  \rmd\tau = \frac{\norm{a}_2 \cdot \prod_{j=0}^m \Gamma(\alpha_j)}{2\pi} \int_{\R} \prod_{j=0}^m (1 + \rmi a_j \tau)^{-\alpha_j} \rmd\tau.
    \]
\end{proof}
\begin{remark}
    Notice that the value of the integral is the right-hand side is real because the function under integral has even real part and odd imaginary one.
\end{remark}

\begin{corollary}\label{cor:simplex_integral_subspace_zero}
   Let $g(w) = \Gamma(\ualpha) \prod_{j=0}^m w_j^{\alpha_j - 1}/\Gamma(\alpha_j)$, where $\ualpha > 1$ and $a \in \R^{m+1}$. Then
    \[
        I_{m-1}(g, \simplex_m \cap \Hset_a) = (\ualpha - 1)\sqrt{(m+1)\left(\sum_{i=0}^m a_i^2 \right) - \left(\sum_{i=0}^m a_i\right)^2} \cdot \frac{1}{2\pi}\int_{\R} \prod_{j=0}^m (1 + \rmi a_j s)^{-\alpha_j} \rmd s.
    \]
\end{corollary}
\begin{proof}
    Notice that $g$ is positively homogeneous function of degree $\ualpha - (m+1) > -m$ on $\R^{m+1}_{++}$. Hence,  we may apply Lemma \ref{lem:pyramid_volume} and Lemma \ref{lem:simplex_laplace}. We obtain
    \begin{align*}
         I_{m-1}(g, \simplex_m \cap \Hset_a) &= \frac{\ualpha-1}{\mathrm{dist}(\overline{\Delta} \cap \Hset_a, 0)} \cdot I_{m}(g, \widetilde{\simplex}_m \cap \Hset_a) \\
         &= \frac{\ualpha-1}{\mathrm{dist}(\overline{\Delta} \cap \Hset_a, 0) \Gamma(\ualpha)}  \int_{\R_+^{m+1} \cap \Hset_a} g(w) \exp\left(-\sum_{\ell=0}^m w_\ell\right) \cH^m(\rmd w).
    \end{align*}
    The last integral could be computed by Lemma~\ref{lm:simplex_fourier_integral}. We have
    $$
        I_{m-1}(g, \simplex_m \cap \Hset_a) = \frac{\ualpha - 1}{\mathrm{dist}(\overline{\simplex}_m \cap \Hset_a, 0)} \cdot \frac{\norm{a}_2}{2\pi}\int_{\R} \prod_{j=0}^m (1 + \rmi a_j s)^{-\alpha_j} \rmd s.
    $$
    Finally, we apply Corollary~\ref{cor:simplex_distance} to conclude the statement. 
\end{proof}

Now we are ready to prove Proposition~\ref{prop:density_linear_dirichlet}.
\begin{proof}[Proof of Proposition~\ref{prop:density_linear_dirichlet}.]
    % First, we give a formula for $p_Z$ in terms of $I_{m-1}$. Define a linear map $\phi(w) = wf$, then by the change of variables formula (see (3.4.3),  \citealp{evans2018measure}), we have
    % \[
    %     \P_{w \sim \Dir(\alpha)}[wf \leq \mu] = \int_{\simplex_m, wf \leq \mu} g(w) \cH^{m}(\rmd w) = \frac{1}{[\phi]} \int_{0}^\mu  \left[ \int_{\simplex_m, wf=u} g(w) \cH^{m-1}(\rmd w) \right] \rmd u,
    % \]
    % where $g(w) = \Gamma(\ualpha) \prod_{j=0}^m w_j^{\alpha_j - 1}/\Gamma(\alpha_j)$ is the density of the Dirichlet distribution. By the formula for the Jacobian of a linear map we have that $[\phi] = \norm{f}_2$. Hence $p_Z$ can be represented as  the following integral
    % \[
    %     p_Z(u) = \frac{1}{\norm{f}_2} \int_{\simplex_m \cap \Hset^u_{f}} g(w) \cH^{m-1}(\rmd w) = \frac{1}{\norm{f}_2}I_{m-1}(g, \simplex_m \cap \Hset^u_{f}),
    % \]
    First, we give a formula for $p_Z$ in terms of $I_{m-1}$. We start from rewriting the probability in terms of a usual integral
    \[
        \P_{w \sim \Dir(\alpha)}[wf \leq \mu] = \int_{w \geq 0, \sum_{i=1}^m w_i \leq 1, wf \leq \mu} g\left(1-\sum_{i=1}^m w_i, w_1,\ldots,w_m\right) \rmd w_1,\ldots, \rmd w_m
     \]
     where $g(w) = \Gamma(\ualpha) \prod_{j=0}^m w_j^{\alpha_j - 1}/\Gamma(\alpha_j)$ is the density of the Dirichlet distribution. We note that this transform exactly defines a map $\cL_\bOne^1$. Then we apply changing of variables formula  \citep[3.4.3]{evans2018measure} using map $\phi(x) = f^\top \cL_{\bOne}^1(x)$
    \begin{align*}
        \P_{w \sim \Dir(\alpha)}[wf \leq \mu] = \frac{1}{[\phi]} \int_0^\mu \left[ \int_{\cL_{\bOne}^1(x) \geq 0, f^\top \cL_{\bOne}^1(w) = u} g(\cL_{\bOne}^1(x)) \cH^m(\rmd x) \right] \rmd u
    \end{align*}
    Define a vector $c = \cL_{\bOne}^\top f$. Then we apply changing of variables formula \citep[3.3.3]{evans2018measure} to the inner integral using parametrization through map $\cL_c^u$
    \begin{align*}
        \P_{w \sim \Dir(\alpha)}[wf \leq \mu]  &= \frac{1}{[\phi]} \int_0^\mu \left[ [\cL_{c}^u] \int_{\cL_{\bOne}^1(\cL_{c}^u w) \geq 0} g(\cL_{\bOne}^1(\cL_{c}^u(z))) \rmd z\right] \rmd u
    \end{align*}
    We note that a Jacobian of $\cL_c^u$ does not depend on the shift parameter $u$, therefore could me moved from the integral sign. Next we apply changing of variables formula for a map $\cL_{\bOne}^1 \circ \cL_c^u$
    
    \begin{align*}
        \P_{w \sim \Dir(\alpha)}[wf \leq \mu]   = \frac{[\cL_{c}^u]}{[\phi][\cL_{\bOne}^1 \circ \cL_c^u]} \int_0^\mu \left[ \int_{\simplex_m, wf = u} g(w) \cH^{m-1}(\rmd w)\right] \rmd u.
    \end{align*}
    To compute all Jacobians we shall use Lemma~\ref{lem:jacob} and Lemma~\ref{lem:jacob_composition}, and notice that $[\phi] = \norm{c}_2$. As a result, $p_Z$ can be represented as  the following integral
    \begin{align*}
        p_Z(u) &= \frac{1}{\sqrt{(m+1) \sum_{j=0}^m f^2(j) - \left(\sum_{j=0}^m f(j) \right)^2}} \int_{\simplex_m \cap \Hset^u_{f}} g(w) \cH^{m-1}(\rmd w) \\
        &= \frac{1}{\sqrt{(m+1) \sum_{j=0}^m f^2(j) - \left(\sum_{j=0}^m f(j) \right)^2}} I_{m-1}(g, \simplex_m \cap \Hset^u_{f}),
    \end{align*}
    where  $\Hset^u_f = \{ w \in \R^{m+1} : wf  = u \}$. Unfortunately, we cannot apply the previous result directly because the hyperplane $\Hset^u_f$ does not intersect $0$ in general. To overcome this issue, define the following vector $a(u)_j = f(j) - u$.
    Note that $\langle w, a(u) \rangle = 0$ iff $\langle w, f - u \bOne \rangle = wf - u = 0$, where we used $\langle w ,\bOne\rangle = 1$. Hence $\Hset^u_f \cap \simplex_m = \Hset_{a(u)} \cap \simplex_m$. We can apply Corollary~\ref{cor:simplex_integral_subspace_zero} to the subspace $\Hset_{a(u)}$
    \begin{align*}
        I_{m-1}(g, \simplex_m \cap \Hset^u_{f}) &= \frac{\ualpha - 1}{2\pi} \sqrt{(m+1)\left(\sum_{j=0}^m a(u)_j^2 \right) - \left(\sum_{j=0}^m a(u)_j\right)^2} \cdot \int_{\R} \prod_{j=0}^m (1 + \rmi a(u)_j s)^{-\alpha_j} \rmd s \\
        &=\frac{\ualpha - 1}{2\pi} \sqrt{ (m+1) \sum_{j=0}^m (f(j) - u)^2 - \left(\sum_{j=0}^m (f(j) - u)\right)^2} \cdot \int_{\R} \prod_{j=0}^m \left(1 + \rmi (f(j) - u)t \right)^{-\alpha_j} \rmd t.
    \end{align*}
    Finally, we will rewrite the expression under square root as follows
    \begin{align*}
        (m+1) \sum_{j=0}^m (f(j) - u)^2 - \left(\sum_{j=0}^m (f(j) - u) \right)^2 &= (m+1)\left( \sum_{j=0}^m f(j)^2 - 2 u \sum_{j=0}^m f(j) + (m+1) u^2 \right) \\
        &- \left( \sum_{j=0}^m f(j) \right)^2 + 2u(m+1) \sum_{j=0}^m f(j) - (m+1)^2 u^2 \\
        &= (m+1) \sum_{j=0}^m f(j)^2 - \left(\sum_{j=0}^m f(j) \right)^2.
    \end{align*}
    We conclude the proof of proposition.
\end{proof}

\paragraph{Saddle point method}
In this paragraph, we analyze the asymptotic behavior of the density of a linear statistic of the Dirichlet distribution using the method of saddle point \cite{olver1997asymptotics,fedoryuk1977metod} and obtain sharp bounds on remainder terms.
\begin{proposition}\label{prop:asymptotics_integral}
     Let $f \in \Fclass_m(\ub,b)$ and let  $\alpha = (\alpha_0, \alpha_1, \ldots, \alpha_m) \in \R_+^{m+1}$ be a fixed vector with $\alpha_0 \geq 2$. Then for any $u \in (\up f, \ub),$ 
    \[
        \int_\R \prod_{\ell=0}^m \left(1 + \rmi(f(\ell) - u)s \right)^{-\alpha_\ell} \rmd s = \left(\sqrt{\frac{2\pi}{\ualpha \, \sigma^2 }} - R_1(\alpha) + R_2(\alpha) \right) \exp(-\ualpha \,\Kinf(\up, u, f)) + R_3(\alpha),
    \]
    where   
    \begin{align*}
        \sigma^2 &= \E_{X \sim \up}\left [\left(\frac{f(X) - u}{1 - \lambda^\star(f(X) - u)}\right)^2\right], 
                \\
        \vert R_1(\alpha) \vert &\leq \frac{c_1}{\sqrt{\sigma^2 c_\kappa \alpha_0}} \cdot \frac{\exp(-c_\kappa \alpha_0) }{\sqrt{\ualpha}}, \\
        \vert R_2(\alpha) \vert &\leq   \frac{c_2}{\sqrt{\sigma^2\,\ualpha \,\alpha_0}} \cdot \frac{\ub}{\ub - \up f}, \\
        \vert R_3(\alpha) \vert &\leq  c_3 \cdot \exp(-\ualpha \Kinf(\up, u, f)) \cdot \frac{1 - \lambda^\star(\ub-u)}{\ub-u} \exp(-c_\kappa \alpha_0)
    \end{align*}
with $c_1 = 2\sqrt{2}, c_2 = \frac{49\sqrt{6}}{9}, c_3 = \frac{\sqrt{5\pi}}{2}, c_\kappa = 1/2 \cdot \log\left(1 + \frac{1}{4} \left(\frac{\ub - \up f}{\ub}\right)^2 \right)$ and $\lambda^\star$ being a solution to the optimization problem
    $$
        \lambda^\star(\up, u, f) = \argmax_{\lambda \in [0, 1/(\ub-u)]} \E_{X \sim \up}\left[\log(1 - \lambda (f(X) - u)) \right].
    $$
\end{proposition}
\begin{remark}
    From these bounds on remainder terms we see that $\alpha_0$ should increase at least as $\log \ualpha$ in order to make $R_3$ small enough.
\end{remark}

\begin{proof}
    Let us first rewrite our integral in the form
    \begin{align}
        I &= \int_\R \prod_{j=0}^n \left(1 + \rmi (f(j) - u)s \right)^{-\alpha_j}\, \rmd s = \int_\R \exp\left( - \ualpha \sum_{j=0}^m \up_j \log(1 + \rmi (f(j) - u)s) \right)\, \rmd s  \notag \\
        &= \int_\R \exp\left( - \ualpha \, \E_{X \sim \up}[\log(1 + \rmi (f(X) - u)s)] \right) \, \rmd s, \label{eq:integral_representation}
    \end{align}
    where we choose the principle branch of the complex logarithmic function. Denote $S(z) = \E_{X \sim \up}[\log(1 + \rmi (f(X) - u)z)]$. In the sequel we shall write for simplicity $\E$ instead of $\E_{X \sim \up}$.
    
    Since $f(X) \leq \ub$, this function is holomorphic for $\vert \im z \vert < 1 / (\ub - u)$ and $\re z \in \R$. The last integral representation \eqref{eq:integral_representation} allows to use the method of saddle point \cite{olver1997asymptotics}.
  Next, we are going to compute the saddle points of the function $S$. To do it, compute the derivative of the function $S$ at complex point $z = x + \rmi y$
    \[
        S'(z) = \rmi \EE{\frac{f(X)-u}{1 + \rmi (f(X) - u)z} } = \EE{\frac{x (f(X) - u)^2 + \rmi (f(X)-u)(1 - y (f(X) - u))}{(1 - y(f(X) - u))^2 + x^2 (f(X) - u)^2} } = 0.
    \]
    Notice that the real part of the expression above is zero if and only if $x = 0$. Therefore the saddle points could be only on the imaginary line $\rmi\R$. They can be found from the equation
    \[
        S'(\rmi y) = \rmi\EE{\frac{f(X) - u}{1 - y(f(X) - u)}} = 0.
    \]
Note that for $y\geq 0,$ this equation  coincides with the optimality condition for $\lambda^\star$ in the definition of $\Kinf(\up, u, f).$     Since $\up f < u < \ub$, the function $y\mapsto S(\rmi y)=\E\left[\log(1 - y (f(X) - u)) \right]$ is strictly concave in $y$ and, therefore, equation $S'(\rmi y) = 0$ has a unique solution $y=\lambda^\star$. Thus the unique saddle point of $S$ is equal to $z_0 = \rmi\lambda^\star$.
Next, let us change the integration contour to $\gamma^\star = \R + \rmi\lambda^\star$. To prove that this contour is suitable, let us show that the real part of $S$ achieves a minimum at $z_0$ over all $z \in \gamma^\star$
    \[
        \re S(x + \rmi\lambda^\star) = \frac{1}{2} \E\left[ \log\left( (1 - \lambda^\star (f(X) - u))^2 + x^2 (f(X) - u)^2 \right) \right].
    \]
    The minimum of \(\re S(x + \rmi\lambda^\star)\) is achieved for $x=0$, therefore the contour  \(\gamma^\star\) is suitable. Hence, we can apply the Laplace method after a simple change of coordinates
    \[
        I = \int_\R \exp\left( -\ualpha \,\E\left[\log( 1 - \lambda^\star (f(X) - u) + \rmi s (f(X) - u))\right] \right) \, \rmd s
    \]
    Denote 
    \[
    T(s) = \E\left[\log( 1 - \lambda^\star (f(X) - u) + \rmi s (f(X) - u))\right].
    \]
     Fix  a cut-off parameter $K > 0$  and  define $\kappa_1 = T(-K) - T(0)$, $\kappa_2 = T(K) - T(0)$. 
 Next,  similarly to Section~6 by \citet{olver1997asymptotics}, we define the change of  variables $v_1 = T(-s) - T(0), v_2 = T(s) - T(0)$ and the implicit functions $q_1(v_1) = \frac{1}{T'(-s)}, \, q_2(v_2) = \frac{1}{T'(s)}.$ Using the first order  Taylor expansion, we can write $q_1(v_1) = \frac{1}{\sqrt{2 T''(0) \cdot v_1}} + r_1(v_1), \, q_2(v_2) = \frac{1}{\sqrt{2 T''(0) \cdot v_2}} + r_2(v_2)$.
    Then we have the following decomposition
    \[
        I = \int_{-K}^{K} \exp(-\ualpha\,  T(s))\, \rmd s + R_3(\alpha) = \left( \sqrt{\frac{2\pi}{\ualpha \, T''(0)}} - R_1(\alpha) + R_2(\alpha) \right) \exp(-\ualpha \, T(0)) + R_3(\alpha),
    \]
    where
    \begin{align*}
        R_1(\alpha) &=  \left(\Gamma\left( \frac{1}{2}, \kappa_1 \, \ualpha \right) + \Gamma\left( \frac{1}{2}, \kappa_2 \, \ualpha \right) \right) \frac{1}{\sqrt{2 T''(0) \, \ualpha}}, \\
        R_2(\alpha) &= \int_0^{\kappa_1} e^{-\ualpha v_1} r_1(v_1)\, \rmd v_1 + \int_0^{\kappa_2} e^{-\ualpha v_2} r_2(v_2)\, \rmd v_2, \\
        R_3(\alpha) &= \int_{\R \setminus [-K, K]} \exp(-\ualpha \, T(s))\, \rmd s,
    \end{align*}
    where $\Gamma(\alpha, x)$ is an upper incomplete gamma function and integration w.r.t. \(v_1,v_2\) is performed over the straight lines connecting the points \(0\) and \(\kappa_1,\kappa_2,\) respectively. Our next goal is to analyze these remainder terms and compute upper bounds on their absolute values.
    
    \paragraph{Term $R_2$.} First, we derive  the uniform bounds for $r_2$ and $r_1$. Using the expansions     
    \begin{align*}
        T(s) &= T(0) + T'(0) \cdot s + \frac{T''(0)}{2} s^2 + \frac{T'''(\xi_1)}{6} s^3,  &\xi_1 \in (0, s) \\
        T'(s) &= T'(0) + T''(0) s + \frac{T'''(\xi_2)}{2} s^2, &\xi_2 \in (0,s).
    \end{align*}
    and noting  that $T'(0) = 0,$   we get
    \begin{align*}
        \vert r_2(v) \vert &= \left\vert \frac{1}{T'(s)} - \frac{1}{\sqrt{2 T''(0) (T(s) - T(0))}}\right\vert \\
        &= \left\vert \frac{\sqrt{T''(0)^2 s^2 +  T''(0) \frac{T'''(\xi_1)}{3} s^3} - T''(0) s - \frac{T'''(\xi_2)}{2} s^2}{[T''(0)s + \frac{T'''(\xi_2)}{2} s^2] \cdot \sqrt{T''(0)^2 s^2 + T''(0) \frac{T'''(\xi_1)}{2} s^3}}\right\vert \\
        &= \left\vert T''(0) \frac{\sqrt{1  + \frac{T'''(\xi_1)}{3T''(0)} s} - 1 - \frac{T'''(\xi_2)}{2T''(0)} s}{s \cdot [T''(0) + \frac{T'''(\xi_2)}{2} s] \cdot \sqrt{T''(0)^2 + T''(0) \frac{T'''(\xi_1)}{3} s}}\right\vert.
    \end{align*}
    Next by applying the inequality $\sqrt{1+x} -1 = \frac{x}{2} - \frac{x^2}{8(1 + \xi_3)^{3/2}}$ for $ \vert \xi_3 \vert < x,$
    \begin{align*}
        \vert r_2(v) \vert &= \left\vert T''(0) \frac{\frac{T'''(\xi_1)}{6 T''(0)} - \frac{T'''(\xi_2)}{2 T''(0)} - s \cdot \frac{(T'''(\xi_1)/ T''(0))^2}{8 \cdot 9 \cdot (1 + \xi_3)^{3/2}}}{ [T''(0) + \frac{T'''(\xi_2)}{2} s] \cdot \sqrt{T''(0)^2 + T''(0) \frac{T'''(\xi_1)}{3} s}}\right\vert \\
        &= \left\vert \frac{\frac{T'''(\xi_1)}{6} - \frac{T'''(\xi_2)}{2} - s \cdot \frac{(T'''(\xi_1))^2}{8 \cdot 9 \cdot (1 + \xi_3)^{3/2}  \cdot T''(0)} }{ [T''(0) + \frac{T'''(\xi_2)}{2} s] \cdot \sqrt{T''(0)^2 + T''(0) \frac{T'''(\xi_1)}{3} s}}\right\vert \leq \frac{ \frac{2}{3} \sup_{u \in (0,s)} \vert T'''(u) \vert + s \cdot \frac{\sup_{u \in (0,s)} \vert T'''(u)\vert^2}{72 \cdot \vert T''(0) \vert } }{ \vert T''(0) + \frac{T'''(\xi_2)}{2} s \vert \cdot \sqrt{\vert T''(0)^2 + T''(0) \frac{T'''(\xi_1)}{3} s \vert}}.
    \end{align*}
    Let us analyze the second and third derivative of $T$
    \begin{align*}
        T''(s) = \E\left[ \left(\frac{f(X) - u}{1 - \lambda^\star(f(X) - u) + is (f(X) - u)} \right)^2 \right], \quad T'''(s) = -2i \E\left[ \left(\frac{f(X) - u}{1 - \lambda^\star(f(X) - u) + is (f(X) - u)} \right)^3 \right].
    \end{align*}
    Define a random variable $Y_s = \frac{f(X) - u}{1 - \lambda^\star(f(X) - u) + is (f(X) - u)},$ then  $T''(s) = \E[Y_s^2]$, $T'''(s) = -2i \E[Y_s^3]$. Let us compute an upper bound on the absolute value of $T'''(s)$
    \[
        \vert T'''(s) \vert \leq 2 \E\left[ \frac{\vert f(X) - u \vert^3}{( (1 - \lambda^\star(f(X) - u))^2 + s^2 (f(X) - u)^2)^{3/2}}  \right] \leq 2 \E[\vert Y_0 \vert^3].
    \]
  By choosing \(1/(2K) = \max\left\{\frac{\ub-u}{1 - \lambda^\star(\ub-u)}, \frac{u}{1+\lambda^\star u} \right\},\) we ensure that
  \( \E[Y_0^2] - s \E[\vert Y_0\vert^3] \geq \frac{1}{2}\E[Y_0^2]\)  for all $0 \leq s < K,$ since 
  \[
        \E[\vert Y_0\vert^3] \leq \max_{j \in \{0,\ldots,m\}} \frac{\vert f(j) - u \vert}{1 - \lambda^\star(f(j) - u)} \E[Y_0^2] \leq \max\left\{\frac{\ub-u}{1 - \lambda^\star(\ub-u)}, \frac{u}{1+\lambda^\star u} \right\} \E[Y_0^2] \leq \frac{1}{2K}\E[Y_0^2].
 \] 
Hence 
 \begin{eqnarray*}
\vert r_2(v) \vert &\leq &\frac{\frac{4}{3} \E[\vert Y_0\vert^3] + s \frac{\E[\vert Y_0\vert^3]^2 }{18 \cdot \E[Y_0^2]}}{(\E[Y_0^2] - \E[\vert Y_0\vert^3] s)\cdot \sqrt{\E[Y_0^2]} \cdot \sqrt{\E[Y_0^2] - \E[\vert Y_0 \vert^3] \frac{2s}{3} }}
\\
&\leq &\frac{4/3 + 1/36}{ 1/2 \cdot \sqrt{2/3}} \cdot \frac{\E[\vert Y_0\vert^3]}{\E[Y_0^2]^{2}} \leq \frac{49\sqrt{6}}{36 \E[Y_0^2]} \cdot \max\left\{\frac{\ub-u}{1 - \lambda^\star(\ub-u)}, \frac{u}{1+\lambda^\star u} \right\}.
\end{eqnarray*}
Next, using the bound
    \[
        \E[Y_0^2] = \frac{\alpha_0}{\ualpha}  \left(\frac{\ub-u}{1 - \lambda^\star(\ub-u)}\right)^2  + \sum_{j=1}^m \frac{\alpha_j}{\ualpha}  \left(\frac{f(j)-u}{1 - \lambda^\star(f(j)-u)}\right)^2 \geq \frac{\alpha_0}{\ualpha}  \left(\frac{\ub-u}{1 - \lambda^\star(\ub-u)}\right)^2,
    \]
   we obtain 
    \[
        \vert r_2(v) \vert \leq \frac{49\sqrt{6}}{36 \sqrt{\E[Y_0^2]}} \cdot \sqrt{\frac{\ualpha}{\alpha_0}} \max\left\{1, \frac{u(1 - \lambda^\star(\ub - u))}{(1+\lambda^\star u)(\ub - u)} \right\}.
    \]
    Next  we use Lemma~12 from \cite{honda2010asymptotically}
    \[
        \lambda^\star \geq \frac{u - \up f}{u(\ub - u)} \iff 1 + \lambda^\star u \geq \frac{\ub - \up f}{\ub - u},
    \]
    thus
    \[
        \vert r_2(v) \vert \leq \frac{49\sqrt{6}}{36 \sqrt{\E[Y_0^2]}} \cdot \sqrt{\frac{\ualpha}{\alpha_0}} \frac{\ub}{\ub - \up f}.
    \]
    
 A similar bound  also holds for $r_1(v)$ by symmetry. Set $c_2' =  \frac{49\sqrt{6}}{18},$ then 
   \begin{align*}
        \vert R_2(\alpha) \vert &\leq \frac{c_2'}{2\sqrt{\E[Y_0^2]}} \cdot \sqrt{\frac{\ualpha}{\alpha_0}}\cdot \left\vert \int_0^{\kappa_2} e^{-\ualpha v} \rmd v + \int_0^{\kappa_1} e^{-\ualpha v} \rmd v \right\vert \leq \frac{c_2'}{\sqrt{\E[Y_0^2]}} \cdot \frac{1 + \exp(-\ualpha  \kappa)}{\sqrt{\ualpha \cdot \alpha_0}}  \frac{\ub}{\ub - \up f},
   \end{align*}
 where \(\kappa=\min\{\re \kappa_1,\re \kappa_2\}.\)  
Using the identity
    \begin{align*}
        \re \kappa_1 = \re \kappa_2 &= \frac{1}{2} \EE{\log\left( \frac{(1 - \lambda^\star (f(X) - u))^2 + K^2 (f(X) - u)^2}{(1 - \lambda^\star (f(X) - u))^2} \right)}= \frac{1}{2} \EE{\log\left( 1 + K^2 Y_0^2 \right)}
    \end{align*}
 and the inequality
 \begin{align*}
        \EE{\log\left( 1 + K^2 Y_0^2 \right)} &= 
         \frac{\alpha_0}{\ualpha} \log\left( 1 + K^2 \cdot \left(\frac{\ub-u}{1 - \lambda^\star(\ub-u)}\right)^2 \right) + \sum_{j=1}^m \frac{\alpha_j}{\ualpha}\log\left( 1 + K^2 \cdot \left(\frac{f(j)-u}{1 - \lambda^\star(f(j)-u)}\right)^2 \right) \\
         &\geq\frac{\alpha_0}{\ualpha} \log\left( 1 + K^2 \cdot \left(\frac{\ub-u}{1 - \lambda^\star(\ub-u)}\right)^2 \right) \geq \frac{\alpha_0}{\ualpha} \log\left(1 + \frac{1}{4}\left( \frac{\ub - \up f}{\ub} \right)^2\right),
    \end{align*}
we have $\kappa = \re \kappa_2 = \re \kappa_1 \geq c_\kappa \cdot \frac{\alpha_0}{\ualpha}$
    with $c_\kappa = 1/2 \cdot \log\left(1 + \frac{1}{4}\left( \frac{\ub - \up f}{\ub}\right)^2\right).$ Since $\alpha_0 \geq 2$, we also have $\exp(-c_\kappa \alpha_0) \leq 1$. 
Finally setting $c_2 = 2\cdot c_2' = \frac{49\sqrt{6}}{9},$ we derive the following bound on $R_2$
    \[
        \vert R_2(\alpha) \vert \leq \frac{c_2}{\sqrt{\E[Y_0^2]}} \cdot \frac{1}{\sqrt{\ualpha \, \alpha_0}} \cdot \frac{\ub}{\ub - \up f}.
    \]
    
    \paragraph{Term $R_1$.} By Theorem 2.4 of \citet{borwein2007uniform} we have the following bound on complex gamma function for any complex $z$ with $\re z > 0$
    \[
        \left\vert \Gamma\left( \frac{1}{2}, z \right) \right\vert \leq \frac{2 e^{-\re z}}{\vert z \vert^{1/2}}. 
    \]
    Therefore,
    \[
        \vert R_1(\alpha) \vert \leq \frac{4 }{\sqrt{2 T''(0) \vert \kappa \vert}} \cdot \frac{\exp(-\ualpha \,  \kappa) }{\ualpha}.
    \]
    Set $c_1 = 2 \sqrt{2}$, then we have under our choice of $K$ and $\kappa,$
    \[
        \vert R_1(\alpha) \vert \leq \frac{c_1}{\sqrt{\E[Y_0^2] c_\kappa \alpha_0}} \cdot \frac{\exp(-c_\kappa \alpha_0) }{\ualpha^{1/2}}.
    \]
    
    \paragraph{Term $R_3$.} 
    We have
\begin{eqnarray}
\label{eq: R3-int}
\left\vert \int_{K}^\infty \exp(-\ualpha T(s))\, \rmd s \right\vert \leq \exp(-\ualpha \cdot \re [T(K) - T(0)] ) \cdot \exp(-\ualpha T(0)) \cdot \int_{K}^\infty \exp(-\ualpha \re [T(s) - T(K)])\, \rmd s.
\end{eqnarray}
    Our goal is to bound the last integral. Let us analyze the function under exponent after a change of variables $s \to t+K$
    \begin{align*}
        q(t) &\triangleq \re [T(t+K) - T(K)] = \frac{1}{2} \EE{\log\left( \frac{(1 - \lambda^\star (f(X) - u))^2 + (t + K)^2 (f(X) - u)^2}{(1 - \lambda^\star (f(X) - u))^2 + K^2 (f(X) - u)^2}\right)} \\
        &=\frac{1}{2} \EE{\log\left( 1 +\frac{(2tK + t^2) (f(X) - u)^2}{(1 - \lambda^\star (f(X) - u))^2 + K^2 (f(X) - u)^2}\right)}.
    \end{align*}
    Define a function $g(j) = \frac{(f(j) - u)^2}{(1 - \lambda^\star (f(j) - u))^2 + K^2 (f(j) - u)^2} \geq 0$. Then 
 \begin{eqnarray}
 \label{eq:ineq-q}
 q(t) = \frac{1}{2} \EE{\log\left(1 + (2tK + t^2)g(X) \right)} \geq \frac{1}{2} \cdot \frac{\alpha_0}{\ualpha} \cdot \log(1 + (2t K + t^2) g(0)),
\end{eqnarray}
by  positivity of $g(j)$. By  choosing $\ub > u,$ we have  $g(0) > 0$, therefore \eqref{eq:ineq-q} is a non-trivial lower bound. By substitution of \eqref{eq:ineq-q} into the integral \eqref{eq: R3-int}, we get
    \begin{align*}
        \int_{K}^\infty \exp(-\ualpha \re [T(s) - T(K)]) \rmd s &= \int_0^\infty \exp(-\ualpha q(t)) \rmd t \leq \int_0^\infty \left(\frac{1}{1 + (2t K + t^2) g(0)}\right)^{\alpha_0/2} \rmd t \\
        &\leq \int_0^\infty \left(\frac{1}{1 + t^2 g(0)}\right)^{\alpha_0/2} \rmd t = \frac{\sqrt{\pi} \cdot \Gamma\left(\frac{\alpha_0 - 1}{2}\right)}{2 \sqrt{g(0)} \cdot \Gamma(\alpha_0/2)}.
    \end{align*}
    Notice that the last integral converges if $\alpha_0 > 1$. By symmetry of our arguments, we have the same bound for another part of the integral in \(R_3\). Thus,
    \begin{align*}
        \vert R_3(\alpha) \vert &\leq \exp(-\ualpha \kappa -\ualpha \Kinf(\up, u, f)) \frac{\sqrt{\pi} \cdot \Gamma\left(\frac{\alpha_0 - 1}{2}\right)}{\Gamma(\alpha_0/2)} \left( \sqrt{\frac{(1 - \lambda^\star (\ub - u))^2 + K^2 (\ub - u)^2}{(\ub - u)^2}} \right) \\
        &\leq \frac{\sqrt{5\pi}}{2} \cdot \exp(-c_\kappa \alpha_0 -\ualpha \Kinf(\up, u, f)) \frac{ \Gamma\left(\frac{\alpha_0 - 1}{2}\right)}{\Gamma(\alpha_0/2)} \cdot \frac{1 - \lambda^\star(\ub-u)
        }{\ub-u},
    \end{align*}
    where the last inequality follows from the lower bound  $\re \kappa \geq c_\kappa \alpha_0 / \ualpha$ and the choice of $K$. Set $c_3 = \sqrt{5\pi}/2$ and note that
    \[
        \frac{ \Gamma\left(\frac{\alpha_0 - 1}{2}\right)}{\Gamma(\alpha_0/2)} \leq 1.
    \]
    Finally, we have
    \[
        \vert R_3(\alpha) \vert \leq c_3\cdot \exp(-\ualpha \Kinf(\up, u, f)) \cdot \frac{1 - \lambda^\star(\ub-u)}{\ub-u}  \exp(-c_\kappa \alpha_0) .
    \]
\end{proof}

Our next goal is to provide a lower bound on the density \(p_Z\) using the above representation.
\begin{lemma}\label{lem:lb_dirichlet_density}
     Consider a function $f \in \Fclass_m(\ub,b)$ and a vector $\alpha = (\alpha_0, \alpha_1, \ldots, \alpha_m) \in \R_+^{m+1}$ with $\ualpha  \geq 2\alpha_0$, $\ub \geq 2b$ and
     \[
        \alpha_0 \geq \max\left\{\frac{1}{(\sqrt{2\pi} - 1)^2} \cdot \left(\frac{2 \sqrt{2}}{\sqrt{\log(17/16)}} + \frac{98 \sqrt{6}}{9} \right)^2,  \frac{\log(10\pi \cdot \ualpha)}{\log(17/16) } \right\}.
    \]
    Then for any  $u \in (\up f, \ub)$, 
    \[
        p_Z(u) \geq \frac{\sqrt{\ualpha - 1/\ualpha}}{8\pi}  \cdot \left( \frac{1 - \lambda^\star(\ub-u)}{\ub-u}  \right) \cdot  \exp(-\ualpha \Kinf(\up, u, f)).
    \]
\end{lemma}
\begin{proof}
    First, we derive from Proposition~\ref{prop:asymptotics_integral} a lower bound for the integral \eqref{eq:integral_representation}, using inequality $\ub/(\ub - \up f) \leq 4$ under our conditions on $\ub$ and $\ualpha$
    \begin{align*}
        \int_\R \prod_{j=0}^n \left(1 + \rmi(f(j) - u)s \right)^{-\alpha_j} \rmd s \geq \frac{\left(\sqrt{2\pi} - \frac{c_1 \exp(-c_\kappa \alpha_0)}{\sqrt{c_\kappa \alpha_0}} - \frac{4c_2}{\sqrt{\alpha_0}} \right)}{\sqrt{\ualpha \, \E\left[Y_0^2 \right] }} \exp(-\ualpha \Kinf(\up, u, f)) + R_3(\alpha),
    \end{align*}
    First, we want to ensure that 
    \begin{equation}\label{eq:alpha_first_inequality}
        \sqrt{2\pi} - \frac{c_1 \exp(-c_\kappa \alpha_0)}{\sqrt{c_\kappa \alpha_0}} -  \frac{c_2}{\sqrt{\alpha_0}} \geq 1.
    \end{equation}
    To do it, notice that since $\alpha_0 \geq 2$, then $\frac{c_1 \exp(-c_\kappa \alpha_0)}{\sqrt{c_\kappa \alpha_0}} \leq \frac{c_1 }{\sqrt{c_\kappa \alpha_0}}$. Therefore, to ensure that \eqref{eq:alpha_first_inequality} holds, we can choose 
    \[
        \alpha_0 \geq \frac{1}{(\sqrt{2\pi} - 1)^2} \cdot \left(\frac{c_1}{\sqrt{c_\kappa}} + 4c_2 \right)^2 = \frac{1}{(\sqrt{2\pi} - 1)^2} \cdot \left(\frac{2 \sqrt{2}}{\sqrt{\log(17/16)}} + \frac{98 \sqrt{6}}{9} \right)^2.
    \]
    Under these conditions, we derive
    \begin{align*}
        \int_\R \prod_{j=0}^n \left(1 + \rmi(f(j) - u)s \right)^{-\alpha_j} \rmd s &\geq \frac{1}{\sqrt{\ualpha \, \E[Y_0^2]}}\exp(-\ualpha \Kinf(\up, u, f)) -  \vert R_3(\alpha)\vert \\
        &\geq \exp(-\ualpha \Kinf(\up, u, f)) \left( \frac{1}{\sqrt{\ualpha \, \E[Y_0^2]}} - c_3 \cdot \frac{1 - \lambda^\star(\ub-u)}{\ub-u} \exp(-c_\kappa \alpha_0)\right).
    \end{align*}
    Now using an inequality 
    \[
        \E[Y_0^2] \leq \left(\frac{\ub-u}{2(1 - \lambda^\star(\ub-u))} + \frac{u}{2(1 + \lambda^\star u)}\right)^2 = \frac{\ub^2}{4(1-\lambda^\star(\ub - u))^2(1 + \lambda^\star u)^2} \leq \frac{4(\ub - u)^2}{(1 - \lambda^\star(\ub - u))^2},
    \]
        we conclude that
    \[
          \int_\R \prod_{j=0}^n \left(1 + \rmi(f(j) - u)s \right)^{-\alpha_j} \rmd s \geq \frac{1 - \lambda^\star(\ub-u)}{\ub-u} \cdot \frac{\exp(-\ualpha \Kinf(\up, u, f))}{\sqrt{\ualpha}}\left( 1/2 - c_3 \exp(-c_\kappa \alpha_0) \cdot \sqrt{\ualpha}\right).
    \]
    Next, by choosing $\alpha_0 \geq \frac{1}{c_\kappa} (1/2 \cdot \log(\ualpha) + \log(4 c_3)) = \log(10\pi \cdot \ualpha)/\log(17/16)$, we have
    \[
        \int_\R \prod_{j=0}^n \left(1 + \rmi(f(j) - u)s \right)^{-\alpha_j} \rmd s \geq \frac{1 - \lambda^\star(\ub-u)}{\ub-u} \cdot \frac{\exp(-\ualpha \Kinf(\up, u, f))}{4\sqrt{\ualpha}}.
    \]
    Using this bound, we can easily derive a lower bound on the density \(p_Z\)
    \begin{align*}
        p_{Z}(u) &= \frac{\ualpha - 1}{2\pi} \int_\R \prod_{j=0}^m \left( 1 + i(f(j) - u)s \right)^{-\alpha_j} \rmd s \\
        &\geq \frac{\ualpha - 1}{2\pi} \cdot \frac{1 - \lambda^\star(\ub-u)}{\ub-u} \cdot \frac{\exp(-\ualpha \Kinf(\up, u, f))}{4\sqrt{\ualpha}}.
    \end{align*}
    % Finally, we need to derive a lower bound on the term under the square root using our assumptions that $f(0) = \ub$ and $f(i) \leq b < \ub$ for all $i \in \{1,\ldots,m\}$. Set $V = \sum_{j=1}^m f^2(j)$ and $S = \sum_{j=1}^m f(j)$. It holds
    % \begin{align*}
    %     (m+1)\sum_{j=0}^m f^2(j) - \left(\sum_{j=0}^m f(j) \right)^2 &= (m+1)(\ub^2 + V) - (\ub + S)^2 \\
    %     &= \ub^2 + m \ub^2 + mV + V - \ub^2 - 2\ub S - S^2 \\
    %     &\geq (m\ub^2 + V - 2\ub S) + (mV - S^2).
    % \end{align*}
    % Note that by Cauchy–Schwarz inequality we have $mV = m \sum_{j=1}^m f^2(j) \geq \left(\sum_{j=1}^m f(j)\right)^2 = S^2$. Hence
    % \[
    %     (m+1)\sum_{j=0}^m f^2(j) - \left(\sum_{j=0}^m f(j) \right)^2 \geq m\ub^2 + V - 2 \ub S = \sum_{j=1}^m (\ub^2 -2 \ub f(j) + f^2(j)) = \sum_{j=1}^m (\ub - f(j))^2 \geq m(\ub-b)^2.
    % \]
    % For denominator we have a trivial upper bound $\sum_{j=0}^m f(j)^2 \leq (m+1)\ub^2$. Finally, we have
    % \[
    %     p_Z(u) \geq \frac{(\ub - b)}{\ub} \sqrt{\frac{m}{m+1}}\frac{\sqrt{(\ualpha-1/\ualpha)}}{8\pi}  \cdot \left( \frac{1 - \lambda^\star(\ub-u)}{\ub-u}  \right) \cdot  \exp(-\ualpha \Kinf(\up, u, f)).
    % \]
    %    We conclude the statement.
\end{proof}

Before we proceed with the proof of Theorem~\ref{thm:lower_bound_dbc}, we need the following auxiliary result.
\begin{lemma}\label{lem:ub_lambda_star}
       Consider a function $f \in \Fclass_m(\ub,b)$ and a vector $\alpha = (\alpha_0, \alpha_1, \ldots, \alpha_m) \in \R_+^{m+1}$ with $\ualpha  > \alpha_0$. Then for all $\up f < u < \ub,$ 
       \[
            \left(\frac{1}{\ub-u} - \lambda^\star(\up, u, f)\right) \geq \frac{\alpha_0}{\ualpha} \frac{1}{\ub-u}.
       \]
\end{lemma}
\begin{proof}
    Under the condition $\up f < u < \ub,$ the value $\lambda^\star = \lambda^\star(\up, u, f)$ satisfies the following equation
    \begin{equation}\label{eq:expectation|lem:ub_lambda_star}
        \E\left[ \frac{f(X) - u}{1 - \lambda^\star \cdot (f(X) - u)} \right] = \frac{\alpha_0}{\ualpha} \frac{\ub - u}{1 - \lambda^\star \cdot (\ub - u)} + \sum_{j=1}^m \frac{\alpha_j}{\ualpha} \frac{f(j) - u}{1 - \lambda^\star \cdot (f(j) - u)} = 0.
    \end{equation}
    Define a distribution $\hp$ with $\hp(i) = \frac{\alpha_i}{\ualpha - \alpha_0}$ for $i \in \{1,\ldots,m\}$ and $\hp(0) = 0$. Then the expectation in \eqref{eq:expectation|lem:ub_lambda_star} can be written as
    \[
        \frac{\alpha_0}{\ualpha} \frac{\ub - u}{1 - \lambda^\star \cdot (\ub - u)} + \frac{\ualpha - \alpha_0}{\ualpha} \E_{X \sim \hp}\left[\frac{f(X) - u}{1 - \lambda^\star \cdot (f(X) - u)}\right] = 0.
    \]
    Define a function $w(x,u) = \frac{x-u}{1 - \lambda^\star \cdot (x-u)}$, which is convex  in $x.$ By the Jensen inequality, 
    \[
        \E_{X \sim \hp}\left[\frac{f(X) - u}{1 - \lambda^\star \cdot (f(X) - u)}\right] \geq \frac{\hp f - u}{1 - \lambda^\star \cdot (\hp f - u)}.
    \]
 Hence
    \begin{align*}
        \frac{\alpha_0}{\ualpha} \frac{\ub - u}{1 - \lambda^\star \cdot (\ub - u)} &\leq - \frac{\ualpha - \alpha_0}{\ualpha} \frac{\hp f - u}{1 - \lambda^\star \cdot (\hp f - u) } =  \frac{\ualpha - \alpha_0}{\ualpha} \frac{u - \hp f}{1 + \lambda^\star \cdot (u - \hp f) }.
    \end{align*}
    By rearranging terms, we get
    \[
        \frac{1}{\alpha_0} \left( \frac{1}{\ub-u} - \lambda^\star \right) \geq \frac{1}{\ualpha - \alpha_0} \left( \frac{1}{u - \hp f} + \lambda^\star \right) = \frac{1}{\ualpha - \alpha_0} \left( \frac{1}{u - \hp f} + \frac{1}{\ub-u}\right) - \frac{1}{\ualpha - \alpha_0} \left(\frac{1}{\ub-u} - \lambda^\star \right).
    \]
    As a result,
    \[
        \left( \frac{1}{\ub-u} - \lambda^\star \right) \geq \frac{\alpha_0}{\ualpha } \cdot \frac{\ub - \hp f}{(\ub-u)(u - \hp f)} \geq \frac{\alpha_0}{\ualpha} \frac{1}{\ub-u}.
    \]
\end{proof}

Now we are ready to finish the proof of Theorem~\ref{thm:lower_bound_dbc}.
\begin{proof}[Proof of Theorem~\ref{thm:lower_bound_dbc}]
    By Lemma~\ref{lem:lb_dirichlet_density},
   \begin{align*}
        \PP{Z \geq \mu} = \int_{\mu}^{\ub} p_{Z}(u) \rmd u \geq  \frac{\sqrt{\ualpha-1/\ualpha}}{8\pi}  \cdot  \int_{\mu}^{\ub} \left( \frac{1}{\ub-u} - \lambda^\star(\up, u, f) \right) \cdot  \exp(-\ualpha \Kinf(\up, u, f))\, \rmd u.
    \end{align*}
    Our goal is to analyze the last integral. First of all, by Lemma~\ref{lem:ub_lambda_star} we have
    \[
         \int_{\mu}^{\ub} \left( \frac{1}{\ub-u} - \lambda^\star(\up, u, f) \right)   \exp(-\ualpha \Kinf(\up, u, f)) \,\rmd u \geq \frac{\alpha_0}{\ualpha} \cdot \int_{\mu}^{\ub} \frac{1}{\ub-u}   \exp(-\ualpha \Kinf(\up, u, f)) \,\rmd u.
    \]
   Next, by Theorem 6 of \citet{honda2010asymptotically} we have for all $\up f < u < \ub,$
    \[
        \lambda^\star(\up, u, f) = \frac{\partial}{\partial u} \Kinf(\up, u, f).
    \]
    Therefore, since $\lambda^\star \leq 1/(\ub-u)$, we get
    \[
        \Kinf(\up, u, f) - \Kinf(\up, \mu, f) = \int_{\mu}^u \lambda^\star(\up, s, f) \rmd s \leq \int_{\mu}^u \frac{1}{\ub-s} \rmd s = \log\left( \frac{\ub-\mu}{\ub-u} \right).
    \]
    Hence,
    \begin{align*}
        \int_{\mu}^{\ub} \frac{1}{\ub-u} \exp(-\ualpha \Kinf(\up, u, f)) \,\rmd u \geq \exp(-\ualpha \Kinf(\up, \mu, f)) \cdot \int_{\mu}^{\ub} \frac{1}{\ub-u} \exp\left(-\ualpha \log\left( \frac{\ub-\mu}{\ub-u} \right)\right) \,\rmd u ,
    \end{align*}
    where the last integral could by easily computed
    \[
        \int_{\mu}^{\ub} \frac{1}{\ub-u} \exp\left(-\ualpha \log\left( \frac{\ub-\mu}{\ub-u} \right)\right) \rmd u = \int_{\mu}^{\ub} \frac{1}{\ub-u}\left( \frac{\ub-u}{\ub-\mu} \right)^{\ualpha} \rmd u = \frac{1}{(\ub-\mu)^{\ualpha}} \int_{\mu}^{\ub} (\ub-u)^{\ualpha-1} \rmd u = \frac{1}{\ualpha}.
    \]
    Thus,
    \begin{align*}
        \PP{Z \geq \mu}  &\geq \frac{\alpha_0}{\ualpha^2} \cdot \frac{\sqrt{\ualpha - 1/\ualpha}}{8\pi} \exp(-\ualpha \Kinf(\up, \mu, f)) \\
        &\geq \frac{\alpha_0 \cdot \sqrt{1 - \ualpha^2}}{8\pi} \exp\left(-\ualpha \Kinf(\up, \mu, f) - 3/2 \log \ualpha\right) \\
        &\geq \frac{\alpha_0 \cdot \sqrt{3}}{16\pi} \exp\left(-\ualpha \Kinf(\up, \mu, f) - 3/2 \log \ualpha\right),
    \end{align*}
    where we used that $\ualpha \geq 2$. To conclude the statement, note that $\alpha_0 \geq 16\pi$.
\end{proof}
\newpage
%!TEX root = ../BayesUCBVI.tex
\section{Technical Lemmas}
\label{app:technical}

\subsection{A Bellman-type equation for the variance}
\label{app:Bellman_variance}
For a deterministic policy $\pi$ we define Bellman-type equations for the variances as follows
\begin{align*}
  \Qvar_h^\pi(s,a) &\triangleq \Var_{p_h}{V_{h+1}^\pi}(s,a) + p_h \Vvar^\pi_{h+1}(s,a)\\
  \Vvar_h^\pi(s) &\triangleq \Qvar^\pi_h (s, \pi(s))\\
  \Vvar_{H+1}^\pi(s)&\triangleq0,
\end{align*}
where $\Var_{p_h}(f)(s,a) \triangleq \E_{s' \sim p_h(\cdot | s, a)} \big[(f(s')-p_h f(s,a))^2\big]$ denotes the \emph{variance operator.}
 In particular, the function $s \mapsto \Vvar_1^\pi(s)$ represents the average sum of the local variances $\Var_{p_h}{V_{h+1}^\pi}(s,a)$ over a trajectory following the policy $\pi$, starting from $(s, a)$. Indeed, the definition above implies that
 \[\Vvar_1^\pi(s_1) = \sum_{h=1}^H\sum_{s,a} p_h^\pi(s,a) \Var_{p_h}(V_{h+1}^\pi)(s,a).
 \]
 The lemma below shows that we can relate the global variance of the cumulative reward over a trajectory to the average sum of local variances.
\begin{lemma}[Law of total variance]  For any deterministic policy $\pi$ and for all $h\in[H]$,
  \[
  \E_\pi\!\left[  \left(\sum_{h'=h}^H r_{h'}( s_{h'},a_{h'}) - Q_h^\pi(s_h,a_h)\right)^{\!\!2}\middle| (s_h,a_h)=(s,a) \right] = \Qvar_h^\pi(s,a).
  \]
In particular,
\[
\E_\pi\!\left[ \left(\sum_{h=1}^H r_{h}( s_{h},a_{h}) - V_1^\pi(s_1)\right)^{\!\!2} \right] = \Vvar_1^\pi(s_1) = \sum_{h=1}^H\sum_{s,a} p_h^\pi(s,a) \Var_{p_h}(V_{h+1}^\pi)(s,a).
\]
\label{lem:law_of_total_variance}
\end{lemma}
\begin{proof}
	We proceed by induction. The statement in Lemma~\ref{lem:law_of_total_variance} is trivial for $h=H+1$. We now assume that it holds for $h+1$ and prove that it also holds for $h$. For this purpose, we compute
	\begin{align*}
		&
		%\Qvar_h^\pi(s,a) =
		\E_\pi\left[\!
		 \left(\sum_{h'=h}^H r_{h'}( s_{h'},a_{h'}) - Q_h^\pi(s_h,a_h)\right)^{\!\!2}
		 \middle| (s_h,a_h) \right] \\
		& =
		\E_\pi\left[\! \left( Q_{h+1}^\pi(s_{h+1},a_{h+1}) - p_h V_{h+1}^\pi(s_h,a_h) + \sum_{h'=h+1}^H r_{h'}( s_{h'},a_{h'}) - Q_{h+1}^\pi(s_{h+1},a_{h+1})\right)^{\!\!2}
		\middle| (s_h,a_h) \right] \\
		& =
		\E_\pi\left[\!
			\left( Q_{h+1}^\pi(s_{h+1},a_{h+1}) - p_h V_{h+1}^\pi(s_h,a_h) \right)^{\!\!2}
		\middle| (s_h,a_h) \right] \\
		&
		+ \E_\pi\left[\!
		\left( \sum_{h'=h+1}^H r_{h'}( s_{h'},a_{h'}) - Q_{h+1}^\pi(s_{h+1},a_{h+1}) \right)^{\!\!2}
		\middle| (s_h,a_h) \right] \\
		& + 2 \E_\pi\left[\!
			\left( \sum_{h'=h+1}^H r_{h'}( s_{h'},a_{h'}) - Q_{h+1}^\pi(s_{h+1},a_{h+1}) \right)
			\left( Q_{h+1}^\pi(s_{h+1},a_{h+1}) - p_h V_{h+1}^\pi(s_h,a_h) \right)
		\middle| (s_h,a_h) \right].
	\end{align*}
	The definition of  $Q_{h+1}^\pi(s_{h+1},a_{h+1})$ implies that \[\E_\pi\!\left[ \sum_{h'=h+1}^H r_{h'}( s_{h'},a_{h'}) - Q_{h+1}^\pi(s_{h+1},a_{h+1})	\middle| (s_{h+1},a_{h+1}) \right] = 0.\]
	Therefore, the law of total expectation gives us
	\begin{align*}
		&\E_\pi\left[\!
		 \left(\sum_{h'=h}^H r_{h'}( s_{h'},a_{h'}) - Q_h^\pi(s_h,a_h)\right)^{\!\!2}
		 \middle| (s_h,a_h) \right] \\ & = \E_\pi\!\left[
		\left( V_{h+1}^\pi(s_{h+1}) - p_h V_{h+1}^\pi(s_h,a_h) \right)^2
		\middle| (s_h,a_h) \right]
		+ \E_\pi\!\left[
		\left( \sum_{h'=h+1}^H \!\! \!r_{h'}( s_{h'},a_{h'}) - Q_{h+1}^\pi(s_{h+1},a_{h+1}) \right)^2
		\middle| (s_h,a_h) \right] \\
		& = \Var_{p_h}{V_{h+1}^\pi}(s_h,a_h) \\ & \hspace{0.3cm} + \!\!\!\!\!\!\sum_{(s_{h+1},a_{h+1})}\!\!\!\!\!\! p_h(s_{h+1} | s_h,a_h)\ind_{\left(a_{h+1} = \pi(s_{h+1})\right)} \E_\pi\left[\left( \sum_{h'=h+1}^H \!\!\! r_{h'}( s_{h'},a_{h'}) - Q_{h+1}^\pi(s_{h+1},a_{h+1}) \right)^2
		\middle| (s_{h+1},a_{h+1})\right] \\
		&  = \Var_{p_h}{V_{h+1}^\pi}(s_h,a_h) + p_h \Vvar^\pi_{h+1}(s_h,a_h) \\
		& = \sigma Q_h^\pi(s_h,a_h)
	\end{align*}
	where in the third equality we used the inductive hypothesis and the definition of $\sigma V_{h+1}^{\pi}$. % and the fact that since $\pi$ is deterministic, we have that $\sigma Q_{h+1}^\pi(s_{h+1},a_{h+1}) =  \sigma V_{h+1}^\pi(s_{h+1})$.
\end{proof}

\subsection{On the Bernstein inequality}
\label{app:Bernstein}
We reproduce here a Bernstein-type inequality by \citet{talebi2018variance}.
\begin{lemma}[Corollary 11 by \citealp{talebi2018variance}]\label{lem:Bernstein_via_kl}
Let $p,q\in\simplex_{S-1},$ where $\simplex_{S-1}$ denotes the probability simplex of dimension $S-1$. For all functions $f:\ \cS\mapsto[0,b]$ defined on $\cS$,
\begin{align*}
	p f - q f &\leq  \sqrt{2\Var_{q}(f)\KL(p,q)}+\frac{2}{3} b \KL(p,q)\\
  q f- p f &\leq  \sqrt{2\Var_{q}(f)\KL(p,q)}\,.
\end{align*}
where use the expectation operator defined as $pf \triangleq \E_{s\sim p} f(s)$ and the variance operator defined as
$\Var_p(f) \triangleq \E_{s\sim p} \big(f(s)-\E_{s'\sim p}f(s')\big)^2 = p(f-pf)^2.$
\end{lemma}

\begin{lemma}
\label{lem:switch_variance_bis}
Let $p,q\in\simplex_{S-1}$ and a function $f:\ \cS\mapsto[0,b]$, then
\begin{align*}
  \Var_q(f) &\leq 2\Var_p(f) +4b^2 \KL(p,q)\,,\\
  \Var_p(f) &\leq 2\Var_q(f) +4b^2 \KL(p,q).
\end{align*}
\end{lemma}
\begin{proof}
Let $\tp$ be the distribution of the pair of random variables $(X,Y)$ where $X,Y$ are i.i.d.\,according to the distribution $p$. Similarly, let $\tq$ be the distribution of the pair of random variables $(X,Y)$ where $X,Y$ are i.i.d.\,according to distribution $q$. Since Kullback–Leibler divergence is additive for independent distributions, we know that
\[\KL(\tp,\tq) = 2\KL(p,q) \leq 2\KL(p,q).\]
Using Lemma~\ref{lem:Bernstein_via_kl} for the function $g(x,y) = (f(x)-f(y))^2$ defined on $\cS^2$, such that  $0\leq g\leq b^2,$ we get
\begin{align*}
  |\tp g- \tq g| &\leq \sqrt{4\Var_{\tq}(g) \KL(p,q)} + \frac{4}{3}b^2 \KL(p,q)\\
  &\leq  \sqrt{4 b^2 \KL(p,q) \tq g  } + \frac{4}{3}b^2 \KL(p,q)\\
  &\leq \frac{1}{2} \tq g + \frac{10}{3} b^2 \KL(p,q)\,,
\end{align*}
where in the last line we used $2\sqrt{xy}\leq x +y$ for $x,y\geq 0$. In particular we obtain
\begin{align*}
  \tp g &\leq \frac{3}{2} \tq g + \frac{10}{3} b^2\KL(p,q)\\
  \tq g &\leq 2 \tp g + \frac{20}{3}b^2 \KL(p,q) \,.
\end{align*}
To conclude, it remains to note that
\[ \tp g = 2\Var_p(f) \text{ and } \tq g = 2\Var_q(f). \]
\end{proof}

\begin{lemma}
	\label{lem:switch_variance}
	For $p,q\in\simplex_{S-1}$, for $f,g:\cS\mapsto [0,b]$ two functions defined on $\cS$, we have that
	\begin{align*}
 \Var_p(f) &\leq 2 \Var_p(g) +2 b p|f-g|\quad\text{and} \\
 \Var_q(f) &\leq \Var_p(f) +3b^2\|p-q\|_1,
\end{align*}
where we denote the absolute operator by $|f|(s)= |f(s)|$ for all $s\in\cS$.
\end{lemma}
\begin{proof}
First note that
\[
\Var_p(f) = p(f-g+ g-p g + p g- p f)^2 \leq 2 p(f-g - p f + p g)^2 +2 p(g-p g)^2 = 2\Var_p(f-g)+2\Var_p(g).
\]
From the above we can immediately conclude the proof of the first inequality with
\[
\Var_p(f-g) \leq p(f-g)^2 \leq b p|f-g|,
\]
where we used that for all $s\in\cS$, $0\leq |f(s)-g(s)| \leq b$. For the second inequality, using the Hölder inequality,
\begin{align*}
	\Var_q(f) &= pf^2 - (pf)^2 +(q-p)f^2 + (pf)^2 -(qf)^2 \\
	&\leq \Var_p(f) + b^2\|p-q\|_1 +2b^2\|p-q\|_1 \\
	&\leq \Var_p(f) +3 b^2 \|p-q\|_1.
\end{align*}
\end{proof}

\subsection{Inequalities for quantiles}
\begin{lemma}\label{lm:quantile_bound_tail}
    Let $X$ be a random variable from distribution $\rho$, $\kappa \in (0,1)$, and $u \in \R$. Then
    \[
        \P_{X \sim p}(X > u) \leq 1 - \kappa \iff \Q_{X \sim p}(X, \kappa) \leq u, \qquad 
        \P_{X \sim p}(X > u) \geq 1- \kappa \iff \Q_{X \sim p}(X, \kappa) \geq u.
    \]
\end{lemma}
\begin{proof}
    Follows directly from the definition
    $$
        \Q_{X \sim p}(X,  \kappa) = \inf\{ t \in \R : \P_{X \sim p}(X > u) \leq 1 - \kappa \}.
    $$
\end{proof}

\subsection{Jacobian computation}

For any non-zero vector $v \in \R^n$ we define linear parametrization of the set $\Hset_v^t = \{ x \in \R^n | v^\top x = t\}$ by a linear map $\cL^t_v \colon \R^{n-1} \to \R^n$ by a rule $y \to x, y_1 = x_1,\ldots, y_{n-1} = x_{n-1}, x_0 = \frac{t - \sum_{i=1}^{n-1} v_i y_i}{v_0}$. Without loss of generality we may assume that $v_0 \not = 0$ so the parametrization is well-defined. Using matrix language, we can define a matrix $\cL_v \in \R^{n \times n-1}$ with the first row is equal to $[-v_1/v_0,\ldots,v_m/v_0]$, and the last $n-1$ row forms an identity matrix. For a  linear map $\cL \colon \R^n \to \R^m$ we define a Jacobian $[\cL]$ as a generalized determinant of a gradient matrix of this map (see \cite{evans2018measure} for more detail).

\begin{lemma}[Jacobian of a linear parametrization]\label{lem:jacob}
    For any non-zero vector $v \in \R^n$ such that $v_0 \not = 0$  and $t \in \R$ a Jacobian of map $\cL_v^t$ is equal to $\norm{v}_2 / \vert v_0 \vert$.
\end{lemma}
\begin{proof}
     Note that the gradient vector does not depend on constant shifts. Thus the gradient matrix is equal to a linear map $\cL_v$. Define a vector $\tilde v = [v_1,\ldots,v_n]$, then the square Jacobian is equal to
     \[
        [\cL_v]^2 = \det(\cL_v^\top \cL_v) = \det\left(\begin{bmatrix}
            \tilde v / v_0 & I_{n-1}
        \end{bmatrix}
        \begin{bmatrix}
            \tilde{v}^\top / v_0 \\
            I_{n-1}
        \end{bmatrix}\right)
        = \det\left(\frac{\tilde v \tilde{v}^\top}{v_0^2} + I_{n-1}\right).
     \]
     This matrix is a rank-one matrix plus identity. Its eigenvalues are equal to $\norm{\tilde v}^2/v_0^2 + 1$ and $n-2$ ones. Thus $[\cL_v]^2 = \norm{\tilde v}^2/v_0^2 + 1 = \norm{v}^2/v_0^2$.
\end{proof}

\begin{lemma}[Jacobian of a composition of linear parametrizations]\label{lem:jacob_composition}
    For a vector $f \in \R^{m+1}$ define a vector $c = \cL_{\bOne}^\top f$. Assume that $f_0 \not = f_1$. Then
        \[
            [\cL_{\bOne} \cL_{c}]^2 = \frac{(m+1) \sum_{j=0}^m f_j^2 - \left(\sum_{j=0}^m f_j \right)^2}{(f_1 - f_0)^2}.
        \]
    \end{lemma}
    \begin{proof}
         To compute a Jacobian of the composition we have to compute a vector $c$, a matrix $\cL_c$ and a matrix $\cL_{\bOne} \cL_{c}$.
    \begin{align*}
        \cL_{\bOne}&= \begin{bmatrix}
            -1 & -1 & \ldots & -1 \\
            1 & 0 & \ldots & 0 \\
            0 & 1 & \ldots & \\
            \vdots & \vdots & \ddots & \vdots \\
            0 & 0 & \ldots & 1 \\
        \end{bmatrix}, \qquad
        c^\top = \cL_{\bOne}^\top f = \begin{bmatrix}f_1 - f_0 \\ \vdots \\ f_m - f_0 \end{bmatrix}, \qquad
        \cL_c = \begin{bmatrix}
            \frac{f_0 - f_2}{f_1 - f_0} & \frac{f_0 - f_3}{f_1 - f_0} & \ldots & \frac{f_0 - f_m}{f_1 - f_0} \\
            1 & 0 & \ldots & 0 \\
            0 & 1 & \ldots & \\
            \vdots & \vdots & \ddots & \vdots \\
            0 & 0 & \ldots & 1 \\
        \end{bmatrix}.
    \end{align*}
    Finally,
    \[
        A = \cL_{\bOne} \cL_c = \begin{bmatrix}
            \frac{f_2 - f_1}{f_1 - f_0} & \frac{f_3 - f_1}{f_1 - f_0} & \ldots & \frac{f_m - f_1}{f_1 - f_0} \\
            \frac{f_0 - f_2}{f_1 - f_0} & \frac{f_0 - f_3}{f_1 - f_0} & \ldots & \frac{f_0 - f_m}{f_1 - f_0} \\
            1 & 0 & \ldots & 0 \\
            0 & 1 & \ldots & \\
            \vdots & \vdots & \ddots & \vdots \\
            0 & 0 & \ldots & 1 \\
        \end{bmatrix}
    \]
    Call $v = \left(\frac{f_2 - f_1}{f_1 - f_0} , \frac{f_3 - f_1}{f_1 - f_0} , \ldots , \frac{f_m - f_1}{f_1 - f_0} \right)^\top $ and $u = \left(\frac{f_0 - f_2}{f_1 - f_0} , \frac{f_0 - f_3}{f_1 - f_0} , \ldots , \frac{f_0 - f_m}{f_1 - f_0}\right)^\top $ the first and the second rows of the matrix $A$. We have to compute the following determinant
    $$
        \det(A^\top A) = \det\left( \begin{bmatrix} v & u & I \end{bmatrix} \begin{bmatrix} v^\top \\ u^\top \\ I\end{bmatrix} \right) = v v^\top + u u^\top + I.
    $$
    Notice that this matrix is a rank-2 matrix plus the identity matrix. Let $\lambda_1,\lambda_2$ be two non-zero eigenvalues. Then eigenvalues of matrix $A^\top A$ is $\lambda_1 +1, \lambda_2 + 1$ and all other eigenvalues are ones. Thus
    $$
        \det(A^\top A) = (\lambda_1 + 1)(\lambda_2 + 1) = \lambda_1 \lambda_2 + (\lambda_1 + \lambda_2) + 1.
    $$
    To compute it, notice that non-zero eigenvalues of $B^\top B$ and $B B^\top$ coincide. Then we note
    \[
        v v^\top + u u^\top = \begin{bmatrix} v & u\end{bmatrix} \begin{bmatrix} v^\top \\ u^\top\end{bmatrix},
    \]
    therefore $\lambda_1,\lambda_2$ are eigenvalues of the following matrix $2 \times 2$
    \[
        \begin{bmatrix} v^\top \\ u^\top\end{bmatrix}\begin{bmatrix} v & u\end{bmatrix} = \begin{bmatrix}
            \norm{v}^2 & v^\top u \\
            u^\top v & \norm{u}^2
        \end{bmatrix} = C
    \]
    Then we note that $\det(C) = \lambda_1 \lambda_2$ and $\lambda_1 + \lambda_2 = \Tr(C)$. Thus
    \[
        \det(A^\top A) = \norm{v}^2 \norm{u}^2 - (v^\top u)^2 + \norm{v}^2 + \norm{u}^2 + 1.
    \]
    
    Next we start computing required quantities
    \begin{align*}
        \norm{v}^2 = \frac{\sum_{i=2}^m (f_1 - f_i)^2}{(f_1 - f_0)^2}, \qquad
        \norm{u}^2 = \frac{\sum_{i=2}^m (f_0 - f_i)^2}{(f_1 - f_0)^2}, \qquad
        \vert v^\top u \vert &= \frac{\sum_{i=2}^m (f_1 - f_i)(f_0-f_i)}{(f_1 - f_0)^2}, 
    \end{align*}
    and then
    \begin{align*}
        \norm{v}^2 \norm{u}^2 - (v^\top u)^2 &= \frac{1}{(f_1 - f_0)^4}\left( \sum_{i,j=2}^m (f_1 - f_i)^2(f_0 - f_j)^2 - \sum_{i,j=2}^m (f_1 - f_i)(f_1 - f_j)(f_0-f_i)(f_1 - f_j) \right) \\
        &= \frac{1}{(f_1 - f_0)^4}\left( \sum_{i,j=2}^m (f_1 - f_i)(f_0 - f_j)[ (f_1 - f_i)(f_0 - f_j) - (f_1 - f_j)(f_0 - f_i)] \right) \\
        &= \frac{1}{(f_1 - f_0)^3}\left( \sum_{i,j=2}^m (f_1 - f_i)(f_0 - f_j)(f_i - f_j) \right) \\
        &= \frac{1}{(f_1 - f_0)^3}\left( \sum_{i,j=2}^m (f_0 f_1 - f_0 f_i - f_1 f_j + f_i f_j )(f_i - f_j) \right) \\
        &= \frac{1}{(f_1 - f_0)^3} \sum_{i,j=2}^m (f_0 f_1 f_i - f_0 f_i^2 - f_1 f_i f_j + f_i^2 f_j - f_0 f_1 f_j + f_0 f_i f_j + f_1 f_j^2 - f_i f_j^2 ).
    \end{align*}
    Define $S = \sum_{i=2}^m f_i, V = \sum_{i=2}^m f_i^2$. Then after grouping the terms
    \begin{align*}
        \norm{v}^2 \norm{u}^2 - (v^\top u)^2 &= \frac{ f_0 f_1 (m-1) S - f_0 (m-1) V - f_1 S^2 + V S - f_0 f_1 (m-1) S + f_0 S^2 + f_1 (m-1) V - S V}{(f_1 - f_0)^3} \\
        &= \frac{(m-2) V - S^2}{(f_1 - f_0)^2}.
    \end{align*}
    Finally
    \begin{align*}
        \det(A^\top A) = \frac{1}{(f_1 - f_0)^2} \biggl(& (m-1) \sum_{i=2}^m f_i^2 - \left( \sum_{i=2}^m f_i\right)^2 + (m-1) f_0^2 - 2 f_0 \sum_{i=2}^m f_i + \sum_{i=2}^m f_i^2\\
        &+ (m-1) f_1^2 - 2 f_1 \sum_{i=2}^m f_i + \sum_{i=2}^m f_i^2 + 2 f_0^2 + 2f_1^2 - 2f_0f_1 - f_0^2 - f_1^2 \biggl) \\
        &= \frac{(m+1)\sum_{i=0}^m f_i^2 - \left( \sum_{i=0}^m f_i \right)^2}{(f_1 - f_0)^2}.
    \end{align*}

    \end{proof}
    
\newpage

\section{Non-tabular Extension Detailed}
\label{app:non-tabular-extension-detailed}

In this section, we detail the extension of \BayesUCBVI beyond the tabular setting. First recall that in \IncrBayesUCBVI, Q-value functions are given by the quantile $\uQ_h^t(s,a) = \Q_{b\sim\Unif([B])}\left(\uQ_h^{t,b}(s,a),\kappa_h^t(s,a)\right)$ over $B$ (incremental) Bayesian bootstrap samples. Theses samples can be computed by updating sums of random weights distributed according to an exponential distribution of parameter one. Precisely, we define  
\[
Z_h^{t,b}(s,a,s') \triangleq \sum_{\ell=1}^t \ind\{(s_h^\ell,a_h^\ell)=(s,a)\} z_h^{\ell,b}(s,a) +\ind\{s'=s_0\} \sum_{\ell = -n_0+1}^{0} z_h^{\ell,b}(s,a)\,,
\]
where the weights are i.i.d. $z_h^{t,b}(s,a)\sim \Exponential(1)$. Then, the Bayesian bootstrap sample at time $t$ is given by 
\[
 \uQ_h^{t,b}(s,a) = r_h(s,a) +\sum_{s'\in\cS'}\frac{Z_h^{t,b}(s,a,s')}{\sum_{s''\in\cS'} Z_h^{t,b}(s,a,s'')} \uV_{h+1}^{t}(s')\,.
\]
The complete procedure is detailed in Algorithm~\ref{alg:incrBayesUCBVI}. Note that if the rewards are unknown we just need to also re-weight the observed rewards for a given state-action pair, similarly to \citet{riou20a} for multi-armed bandits. There are two approximations made in \IncrBayesUCBVI with respect to \BayesUCBVI. First, \IncrBayesUCBVI approximates the quantile with $B$ Monte-Carlo samples. Second, \IncrBayesUCBVI reuses the same exponential noise from one time step to the next one to improve the time complexity (run-time).
\begin{algorithm}[h!]
\centering
\caption{\IncrBayesUCBVI}
\label{alg:incrBayesUCBVI}
\begin{algorithmic}[1]
  \STATE {\bfseries Input:}  quantile $\kappa$, $B$ number of bootstrap samples, number of prior transitions $n_0$.
  \STATE {\bfseries Initialize: } weights $Z_h^0(s,a,s') \gets \ind\{s'=s_0\} \sum_{\ell = -n_0+1}^{0} z_h^{\ell,b}(s,a), \text{ where } z_h^{\ell,b}\sim\Exponential(1) \text{ i.i.d.}$, for $h,s,a,s'\in[H]\times\cS\times\cA\times{\cS}'$.
      \FOR{$t \in[T]$}
      \STATE Optimistic value iteration 
        \begin{small}
        \begin{align*}
          \uQ_h^{t-1}(s,a) & \gets  \Q_{b\sim\Unif[B]}\left( r_h(s,a) +\sum_{s'\in\cS'}\frac{Z_h^{t-1,b}(s,a,s')}{\sum_{s''\in\cS'} Z_h^{t-1,b}(s,a,s'')} \uV_{h+1}^{t-1}(s'), \kappa_h^{t-1}(s,a)\right)\\
          \uV_h^{t-1}(s) &\gets \max_{a\in\cA} \uQ_h^{t-1}(s,a) \qquad \uV_h^{t-1}(s_0) \gets \Vstar_h(s_0)\\
          \uV_{H+1}^{t-1}(s) & \gets 0\nonumber\,.
        \end{align*}
        \end{small}
      \FOR{$h \in [H]$}
        \STATE Play $a_h^t \in \argmax_a \uQ_h^{t-1}(s,a)$.
        \STATE Observe reward $r_h^t$ and next state $s_{h+1}^t\sim p_h(s_h^t,a_h^t)$.
        \STATE Update weights $Z_h^{t,b}(s,a,s') \gets Z_h^{t-1,b}(s,a,s') + \begin{cases} z_h^{t,b}(s,a,s')\sim \Exponential(1) \text{ i.i.d. } &\text{if } (s,a,s')=(s_h^t,a_h^t,s_{h+1}^t)\\
  0 &\text{else}\end{cases}\,.$
      \ENDFOR
  \ENDFOR
\end{algorithmic}
\end{algorithm}

\begin{algorithm}[h!]
\centering
\caption{\BayesUCBDQN}
\label{alg:BayesUCBDQN}
\begin{algorithmic}[1]
  \STATE {\bfseries Input:} discount factor $\gamma$, quantile $\kappa$, $B$ number of bootstrap samples, $\phi,\psi$ parameter of the Q-value respectively target Q-value network, replay buffer $\replay$, \textcolor{blue}{pseudo target $y^{\texttt{pseudo}}$}, \textcolor{blue}{pseudo transition probability $\epsilon$}.
      \FOR{$t \in[T]$}
      \FOR{$h \in [H]$}
        \STATE Play $a_h^t \in \argmax_a \Q_{b\sim\Unif[B]}\big(Q_{h}^{\phi,b}(s_h^t,a),\kappa\big)$.
        \STATE Observe reward $r_h^t$ and next state $s_{h+1}^t\sim p_h(s_h^t,a_h^t)$.
        \STATE Record $\{h,r_h^t,s_h^t,a_h^t,s_{h+1}^t,z_h^t\}$ in $\replay$ where $z_h^{t,b}\sim\Exponential(1)$ i.i.d for $b\in[B]$.
        \STATE \textcolor{blue}{With probability $\epsilon$ record $\{h,\_,s_h^t,a^0,s_0,z^0\}$ in $\replay$ where $z^{0,b}\sim\Exponential(1)$ i.i.d for $b\in[B]$ and $a^0\sim\Unif(\cA)$}.
      \ENDFOR
      \IF{time to update}
        \STATE{Sample a batch of transitions $\cC =\big\{(h,r,s,a,s',z)\big\}$ from $\replay$}
        \STATE Compute the targets
        \FOR{$(h,r,s,a,s',z)\in\cC$}
        \IF{$s \neq s_0$}
        \STATE \[y(h,s,a,s') \gets r+\gamma\max_{a'\in\cA}\Q_{b\sim\Unif[B]}\big(Q_{h+1}^{\psi,b}(s',a'),\kappa\big)\]
        \ELSE
        \STATE \textcolor{blue}{\[y(h,s,a,s') \gets y^{\texttt{pseudo}}\,.\]}
        \STATE
        \ENDIF
        \ENDFOR
        \STATE{Update the Q-value network by one step of gradient descent with
        \[
          \nabla_\phi \frac{1}{|\cC|} \sum_{(h,s,a,s',z)\in\cC}\sum_{b=1}^B z^b \left(Q_h^{\phi,b}(s,a) - y(h,s,a,s')\right)^2
        \,.\]}
        \IF{time to update target}
        \STATE{Update the target Q-value network $\psi \gets \phi$.}
        \ENDIF
      \ENDIF
   \ENDFOR
\end{algorithmic}
\end{algorithm}
We now present the non-tabular extension of \IncrBayesUCBVI. Fix a state-action pair $(s,a)$ at time $t$. As explained in Section~\ref{sec:non_tabular_extension}, the Bayesian bootstrap samples can be obtained thanks to a weighted linear regression,
\begin{align*}
  \uQ_h^{t,b}(s,a) \triangleq & \argmin_{x} \sum_{n=-n_0+1}^{n_h^t(s,a)} z_h^{n,b}(s,a) \left(x - y_h^n(s,a)\right)^2, \qquad
  \text{ where }z_h^{n,b}(s,a) \sim \Exponential(1)\text{ i.i.d.}
\end{align*}
and the targets are defined by $y_h^n(s,a) \triangleq r_h(s,a) + \uV_{h+1}^t(s_{h+1}^n)$. Note that it is possible to adapt this re-weighting to any loss instead of the mean-squared error. To combine this exploration procedure with the \DQN algorithm, we introduce two Q-value networks. One Q-value network $Q^{\psi}$ parametrized by $\phi$ and one target Q-value network parametrized by $\psi$. We also introduce a replay buffer $\replay$. For each observed transition, we record in $\replay$ the tuple $\{h,r_h^t,s_h^t,a_h^t,s_{h+1}^t,z_h^t\},$ where we also add an exponential mask $z_h^t$. Each element of the mask $z_h^{t,b}\sim \Exponential(1)$ i.i.d.\,, will be used to weight the loss. Precisely, given a batch of transitions $\cC =\big\{(h,r,s,a,s',z)\big\};$ from $\replay$ we first compute the targets with the target Q-value network,
\[
y(h,s,a,s') \gets r+\gamma\max_{a'\in\cA}\Q_{b\sim\Unif[B]}\big(Q_{h+1}^{\psi,b}(s',a'),\kappa\big)\,.
\]
Then, we update each (incremental) Bayesian bootstrap of the Q-value network with one gradient step
\[
  \nabla_\phi \frac{1}{|\cC|} \sum_{(h,s,a,s',z)\in\cC}\sum_{b=1}^B z^b \left(Q_h^{\phi,b}(s,a) - y(h,s,a,s')\right)^2
\,.\]
As in \DQN, the target Q-value network is regularly updated with the weights of the Q-value network. To stick to the \IncrBayesUCBVI, it remains to introduce a mechanism that emulates the prior transitions. Since obviously, we cannot add a prior transition for each state-action pair, we propose to simply use $\epsilon$-greedy \emph{to add prior transitions} to the replay buffer $\replay$. In particular, when interacting with the environment, say that the agent is in state $s_h^t,$ then with probability $\epsilon,$ we add a pseudo transition $\{h,\_,s_h^t,a_0,s_0,z_h^0\}$  to the replay buffer $\replay,$ where the action $a_0\sim\Unif(\cA)$ is sampled uniformly at random and $z_h^0$ is an exponential mask. When a pseudo transition is sampled from the replay buffer, we assign to it an arbitrary fixed target $y^{\texttt{pseudo}}$. Typically, we want the constant $y^{\texttt{pseudo}}$ to be large, e.g., of the order of the value of the pseudo state $s_0$. We name the resulting algorithm \BayesUCBDQN and detail it in Algorithm~\ref{alg:BayesUCBDQN}. We highlight in \textcolor{blue}{blue} the procedure used to add the pseudo transitions.

\newpage
\newpage
\section{Experimental Details}
\label{app:experiment_details}

In this section we detail the experiments we conduct for tabular and large-scale environments.

\subsection{Tabular}

\paragraph{Environment} For the tabular experiments\footnote{Our code is based on the library of \citet{rlberry}.} we consider a simple grid-world with $5$ connected rooms of size $5\times5$ totalling $S=129$ states. The agents starts in the middle where there is a very small deterministic reward of $0.01$. Additionally there is one small deterministic reward of $0.1$ in the leftmost room, one large deterministic reward of $1$ in the  rightmost room and zero reward elsewhere. The agent can take $A=4$ actions: moving up, down, left, right. When taking an action, the agent moves in the corresponding direction with probability $0.9$ and moves to a neighbor state at random with probability $0.1$. The horizon is fixed to $H=30$. Thus in this environment, the agent must explore efficiently all the rooms avoiding being lured by the small reward in the leftmost room.

\paragraph{Baselines} We compare \BayesUCBVI and \IncrBayesUCBVI to the following baselines:
\begin{itemize}
    \item \PSRL \citep{obsband2013more},
    \item \RLSVI \citep{osband16generalization}, 
    \item \UCBVI \citep{azar2017minimax}.
\end{itemize}

Using different parameters (e.g., by changing the multiplicative constants in the bonus of \UCBVI or the scale of the noise in \RLSVI) can result in drastically different empirical regrets. Thus for a fair comparison, we make the following choices. For \UCBVI, the bonus at state-action pair is 
\begin{align*}
	\beta_h^t(s,a) \triangleq
	\min\left(
	\sqrt{\frac{1}{n_h^t(s,a)}} + \frac{H-h+1}{n_h^t(s,a)}, H-h+1
	\right)\,.
\end{align*}
As explained by \citet{menard2021ucb}, this bonus does not necessarily result in a true UCB on the optimal Q-value. However, it is a valid UCB for $n_h^t(s,a)= 0$ which is important in order to discover new state-action pairs. For \RLSVI at state-action pair $(s,a),$ we use noise from a centered Gaussian probability distribution with standard deviation $\beta_h^t(s,a)$. For \PSRL, we use a Dirichlet prior on the transition probability distribution with parameter $(1/S,\ldots,1/S)$ and for the rewards a Beta prior with parameter $(1,1)$. Note that since the reward $r$ is not necessarily in $\{0,1\}$ we just sample a new reward $r'\sim\Ber(r)$ accordingly to a Bernoulli distribution of parameter $r$, to update the posterior, see \citet{garivier2011the}. We choose the same parameter for \BayesUCBVI and \IncrBayesUCBVI. The quantile is fixed to $\kappa_h^t(s,a) \triangleq 0.85$. This choice is informally justified as follows: If we assume $\beta_h^t(s,a) \approx \sqrt{\big(H^2\log(1/\delta)\big)/\big(2n_h^t(s,a)\big)}$ then it holds $\delta \approx e^{-2/H^2} \geq e^{-2} \approx 0.15$. Thus, we get $\kappa \approx 1-\delta \approx 0.85$. The number of pseudo transitions is set to $n_0 =1$ as well as the reward of the pseudo transition is $r_0 = 1$. We use $B=64$ Monte-Carlo samples to approximate the quantile. 
 
\paragraph{Results} In Figure~\ref{fig:regret_baselines}, we plot the regret of the various baselines, \BayesUCBVI and \IncrBayesUCBVI in the aforementioned environment. In this experiment, we observe that both \BayesUCBVI and \IncrBayesUCBVI achieves competitive results with respect to baselines relying on noise-injection for exploration (\PSRL, \RLSVI). This is remarkable, since the common belief is that optimistic algorithm perform poorly in practice \citep{osband2017why}. Indeed \IncrBayesUCBVI exhibits a similar regret than \PSRL. It is not completely surprising since they share the same model on the transitions\footnote{However, not exactly the same on the rewards: Beta/Bernoulli versus non-parametric. Nonetheless, this should have a relatively small impact since the rewards are deterministic.} (up to the prior). Whereas \BayesUCBVI performs slightly worse than \IncrBayesUCBVI but better than \RLSVI. One possible reason to explain this gap between \BayesUCBVI and \IncrBayesUCBVI is that the incremental implementation of Bayesian bootstrap forgets faster the prior than the non-incremental version resulting in a more aggressive algorithm. We also note that \RLSVI performs slightly worse than \PSRL/\IncrBayesUCBVI but much better than \UCBVI. A possible explanation for this ranking is that \RLSVI is much more aggressive than \UCBVI when they have comparable noise/bonuses. Whereas \PSRL,\IncrBayesUCBVI,\BayesUCBVI take better advantage of the small variance of this particular environment than the two last baselines.

\subsection{Deep RL}
We introduce important details of a few baseline algorithms.

\paragraph{\DQN.} 
\DQN adopts the standard Q-learning algorithm. \DQN
maintains a Q-function $Q_\theta$. When interacting with the environment, \DQN adopts $\epsilon$-greedy with respect to $Q_\theta(x,a)$. The value of $\epsilon$ decays linearly from $\epsilon_{\text{max}}=1.0$ to $\epsilon_{\text{min}}=0.01$. The algorithm puts transition tuples $(s,a,r,s')$ into a replay buffer $\mathcal{R}$. At training time, \DQN samples $C=32$ transitions tuples $(s_i,a_i,r_i,s_i')_{i=1}^C\sim \mathcal{R}$ uniformly from the replay. It updates the parameter by minimizing the following loss function with respect to $\theta$,
\begin{align*}
    \frac{1}{C}\sum_{i=1}^C \left(Q_\theta(s_i,a_i) - r_i - \gamma \max_{a'} Q_{\theta^-}(s_i',a')\right)^2.
\end{align*}
Here $\theta^-$ is the parameter of the target network, a delayed copy of $\theta$ used for stabilizing training \citep{mnih2013playing}. See \citep{mnih2013playing} for other missing hyper-parameters.

\paragraph{Double \DQN.} Double \DQN \citep{van2016deep}, built on top of \DQN \citep{mnih2013playing}, further applies the double Q-learning algorithm to stabilize learning \citep{hasselt2010double}. Concretely, in addition to the online network $\theta$, Double \DQN maintains two copies of online networks $\theta_1,\theta_2$; and two copies of target networks $\theta_1',\theta_2'$. Double \DQN constructs the following learning target:
\begin{align*}
    Q_1^{(i)} &= r_i + \gamma Q_{\theta_1^{-}}(s_i',a_i'), a_i'=\arg\max_{a'} Q_{\theta_2^{-}}(s_i,a') \\
    Q_2^{(i)} &= r_i + \gamma Q_{\theta_2^{-}}(s_i',a_i'), a_i'=\arg\max_{a'} Q_{\theta_1^{-}}(s_i,a').
\end{align*}
Then both networks $\theta_1,\theta_2$ are updated  by minimizing the following loss function
\begin{align*}
    \frac{1}{C}\sum_{i=1}^C \left(Q_{\theta_1}(s_i,a_i) - Q_1^{(i)}\right)^2 + \frac{1}{C}\sum_{i=1}^C \left(Q_{\theta_2}(s_i,a_i) - Q_2^{(i)}\right)^2.
\end{align*}
In summary, Double \DQN decouples the maximizing operation and the maximizing value. Empirically, this helps avoid the over-estimation issue that plagues vanilla \DQN. Similar techniques have been implemented in \BootDQN and we adopt the same technique in \BayesUCBDQN.

\paragraph{\BootDQN.} \BootDQN \citep{osband2015bootstrap} maintains $B$ copies of the Q-function $Q_{\theta_b},b\in[B]$. In practice, maintaining $B$ full networks might be too expensive; instead, \citet{osband2015bootstrap} suggested to share the torso network and only maintain $B=10$ different heads to the Q-networks. During interaction with the environment, \BootDQN uniformly random samples a network $Q_{\theta_b},b\sim \text{Uniform}[B]$. Then the algorithm selects an action by being greedy with respect to $Q_{\theta_b}$. In practice, it has been observed that some local greedy exploration might also be helpful \citep{osband2015bootstrap,osband18randomized}. The transition $(s,a,r,s')$ is put into the buffer $\mathcal{R}$, along with a mask $m\in\mathbb{R}^B$ with each component independently generated from a Bernoulli distribution of parameter $p\in[0,1]$. At training time, only the Q-networks with mask $m_b=1$ is trained with the sampled transition using the \DQN loss function.

Though the mask is meant to enforce diversity across different copies of the Q-networks. In practice, \citet{osband2015bootstrap} reported that when $p=1$ the method works well too. They speculated it is because the stochastic gradient based training of \DQN networks already led to sufficient amount of diversity.

\paragraph{\BayesUCBDQN.} Similar to \BootDQN, \BayesUCBDQN maintains $B=10$ copies of the Q-functions by creating $B$ separate Q-network heads with a shared torso network. At acting time, \BayesUCBDQN acts greedily with respect to the bootstrap quantile, computed across $B$ heads. The transition $(s,a,r,s')$ is put into the buffer $\mathcal{R}$, along with a mask $z\sim\mathcal{E}(1)$. At training time, this mask is used for weighing different transitions (see Algorithm~\ref{alg:BayesUCBDQN}). All Q-networks are updated by a common target, computed as the $\kappa$-quantile bootstrapped values. The quantile is set at $\kappa=0.85$ in our experiments. 

We provide a detailed discussion on the effect of different hyper-parameters and implementations on the performance of the agent, such as the number of heads $B$ and quantile parameter $\kappa$. See Appendix~\ref{app:experiment_details_discussion}.

\paragraph{Network architecture.} The network architecture follows from \citet{mnih2013playing}. The network consists of a torso network with convolutional layers that process the input state images $s$. The torso layer outputs an embedding $\text{embed} = \text{torso}(s)$. The downstream network is a MLP that takes the embedding as input, and output a $A$-dimensional vector as the Q-function approximation $\text{head}(\text{embed})\in\mathbb{R}^A$;  see \citet{mnih2013playing} for detailed definitions of layer sizes and non-linear activation functions.

For \BootDQN, in order to avoid the high computational complexity of having $B$ copies of the full network, they argued to instead maintain $H$ copies of the head networks $\text{head}_b,b\in[B]$. The Q-function for the $b$-th head is defined as $Q_{\theta_b}(s,a) = \text{head}_b(\text{embed})$.

\paragraph{Other hyper-parameters.} All algorithms are trained with $200$M frames for each environment. The networks are trained with RMSProp optimizer \citep{hinton2012neural} with learning rate $2.5\cdot 10^{-4}$. See \citet{mnih2013playing} and \citet{osband2015bootstrap} for other missing hyper-parameters.

\paragraph{Environment and evaluation.} The testing environments are Atari-57 games, consisting of $57$ selected Atari games \citep{bellemare2013arcade}. For each game, the state $s_t$ is an image and the action $a_t$ corresponds to controls in the game. For each game, at iteration $t$, let $z_t^{(i)}$ be the performance of the algorithm in a particular game $i$ for $1\leq i\leq 57$. Then the human-normalized performance is normalized per game as 
\begin{align*}
    \text{norm}(z_t^{(i)}) = \frac{z_t^{(i)} - z_u^{(i)}}{z_h^{(i)} - z_u^{(i)}}.
\end{align*}
Here, $z_u^{(i)}$ is the performance of the random policy in game $i$; $z_h^{(i)}$ is the human performance in game $i$. The median performance at iteration $t$ across all game is computed as 
\begin{align*}
    \text{median-performance}(t) = \text{median}\left(\{z_t^{(t)}\}_{i=1}^{57}\right).
\end{align*}
The operation $\text{median}\left(\cdot\right)$ takes a set of human normalized scores per game and computes the median value over all scores.

\subsection{Discussion on the effect of hyper-parameters on \texorpdfstring{\BayesUCBDQN}{Bayes-UCBDQN}} \label{app:experiment_details_discussion}
As shown in Algorithm 3, \BayesUCBDQN adds a considerable number of hyper-parameters compared to \DQN and \BootDQN. In our experiments, we have empirically studied the effect of such hyper-parameters and provide a detailed discussion here.

\paragraph{Effect of $\kappa$.} Throughout experiments, we use $\kappa=0.85$. We have tested the algorithm's performance under other values of $\kappa\in\{0.5,0.6,0.7,0.8,0.9\}$. Overall, we find that the algorithm's performance is not very sensitive to $\kappa$.

\paragraph{Pseudo transition probability $\epsilon$ and pseudo target $y^\text{pseudo}$.} Throughout experiments, we use $\epsilon=0$. This is mainly because, we find that the training tend to be unstable when using $\epsilon$ significantly larger than $0$. The best choice of $y^\text{pseudo}$ seems to be also game-dependent, and as a result, is more challenging to tune in practice.

We speculate that the instability is due to the fact that \DQN updates are based on the minimization of squared Bellman errors. Concretely, when with probability $\epsilon$, the algorithm encounters $y^\text{pseudo}$, which is an optimistic value estimate and is hence likely to be an outlier in the data distribution over target values, the update becomes unstable. To address such issues might require further modification to the \DQN updates, such as based on categorical representation of values \citep{schrittwieser2020mastering} or alternative update rules \citep{bas2020logistic}.

\paragraph{Exponential mask $z^b$.} Throughout experiments, we sample $z\sim \mathcal{E}(1)$ for each transition and apply the mask during learning, based on Algorithm 3. Empirically, we find that the mask plays a similar role as the Bernoulli masks adopted in \citep{osband2015bootstrap}. 

\paragraph{Effect of number of heads $B$.} Throughout experiments, we use $B=10$. We have also tried $B\in\{10,20,100\}$. Overall, we find that increasing the number of heads does not improve the performance. Quite on the contrary, larger number of heads slightly degrades the performance. Such empirical ablations are consistent with the observation of \citet{osband2015bootstrap} that $B\approx10$ works the best.
\end{document}